\documentclass[USenglish,onecolumn]{article}

\usepackage[utf8]{inputenc}
\usepackage[big]{dgruyter}

\usepackage{enumitem}

\usepackage{tabularx}
 
\newcommand{\indep}{\mathrel{\perp\mspace{-10mu}\perp}}
\usepackage{centernot}
\newcommand{\dep}{\centernot{\indep}}
\newcommand{\given}{\,|\,}
\newcommand{\dsep}[1]{\overset{d}{\underset{{#1}}{\perp}}}

\newcommand{\tuple}[1]{\langle {#1} \rangle}

\newcommand{\BG}{\mathcal{B}}
\newcommand{\G}{\mathcal{G}}
\newcommand{\V}{\mathcal{V}}
\newcommand{\E}{\mathcal{E}}
\newcommand{\R}{\mathbb{R}}

\usepackage{subcaption}
\usepackage{float}
\usepackage{tikz}
\usepackage{tkz-graph}
\usetikzlibrary{arrows,arrows.meta,fit,automata,matrix,calc,patterns,decorations,decorations.markings,shapes,positioning}

\usepackage{thmtools}
\theoremstyle{plain}
\newtheorem{theorem}{Theorem}
\newtheorem*{theorem*}{Theorem}

\newtheorem*{proposition*}{Proposition}
\newtheorem{lemma}{Lemma}

\theoremstyle{definition}

\newtheorem{assumption}{Assumption}

\newtheorem*{remarkxx}{Remark}
\newtheorem{examplex}{Example}
\newtheorem*{examplexx}{Example}

\newenvironment{remark*}
{\pushQED{\qed}\remarkxx}
{\popQED\endremarkxx\vspace{2ex}}

\newenvironment{example}
{\pushQED{\qed}\examplex}
{\popQED\endexamplex\vspace{2ex}}
\newenvironment{example*}
{\pushQED{\qed}\examplexx}
{\popQED\endexamplexx\vspace{2ex}}

\graphicspath{{fig/}}

\begin{document}

  \articletype{Research Article{\hfill}Open Access}

\author[1]{Tineke Blom}

\author*[2]{Joris M. Mooij}

\affil[1]{Informatics Institute, University of Amsterdam, the Netherlands; E-mail: t.blom2@uva.nl}

\affil[2]{Korteweg-de Vries Institute, University of Amsterdam, the Netherlands; E-mail: j.m.mooij@uva.nl}

\title{\huge Causality and independence in perfectly adapted dynamical systems}

\runningtitle{Causality and independence under perfect adaptation}

\abstract{Perfect adaptation in a dynamical system is the phenomenon that one or more variables have an initial transient response to a persistent change in an external stimulus but revert to their original value as the system converges to equilibrium. With the help of the causal ordering algorithm, one can construct graphical representations of dynamical systems that represent the causal relations between the variables and the conditional independences in the equilibrium distribution. We apply these tools to formulate sufficient graphical conditions for identifying perfect adaptation from a set of first-order differential equations. Furthermore, we give sufficient conditions to test for the presence of perfect adaptation in experimental equilibrium data. We apply this method to a simple model for a protein signalling pathway and test its predictions both in simulations and using real-world protein expression data. We demonstrate that perfect adaptation can lead to misleading orientation of edges in the output of causal discovery algorithms.}



\journalname{Journal of Causal Inference}
\DOI{DOI}
\startpage{1}
\received{..}
\revised{..}
\accepted{..}

\journalyear{2023}
\journalvolume{..}

\maketitle

\section{Introduction}
\label{sec:introduction}

Understanding causal relations is an objective that is central to many scientific endeavors. It is often said that the gold standard for causal discovery is a randomized controlled trial, but practical experiments can be too expensive, unethical, or otherwise infeasible. The promise of causal discovery is that we can, under certain assumptions, learn about causal relations by using a combination of data and background knowledge \citep{Cooper1997, Richardson1999, Pearl2000, Spirtes2000, Zhang2008, Hyttinen2012, Colombo2012, Forre2018, Mooij2020, Mooij2020b}. Roughly speaking, causal discovery algorithms construct a graphical representation that encodes certain aspects of the data, such as conditional independences, given some constraints that are imposed by background knowledge. Under additional assumptions on the underlying causal mechanisms these graphical representations have a causal interpretation as well. In this work, we specifically consider the equilibrium distribution of \emph{perfectly adapted dynamical systems}. We will show that such systems may have the property that the corresponding graphs that encode the conditional independences in the distribution do not have a straightforward causal interpretation in terms of the changes in distribution induced by interventions. 

\emph{Perfect adaptation} in a dynamical system is the phenomenon that one or more variables in the system temporarily respond to a persistent change of an external stimulus, but ultimately revert to their original values as the system reaches equilibrium again. We study the differences between the causal structure implied by the dynamic equations and the conditional dependence structure of the equilibrium distribution. To do so, we make use of the technique of \emph{causal ordering}, introduced by Simon \cite{Simon1953}, which can be used to construct a \emph{causal ordering graph} that represents causal relations and a \emph{Markov ordering graph} that encodes conditional independences between variables \citep{Blom2020}. We introduce the notion of a \emph{dynamic causal ordering graph} to represent transient causal effects in a dynamical model. We use these graphs to provide a sufficient graphical condition for a dynamical system to achieve perfect adaptation, which does not require simulations or cumbersome calculations. Furthermore, we provide sufficient conditions to test for the presence of perfect adaptation in real-world data with the help of the equilibrium Markov ordering graph and we explain why the usual interpretation of causal discovery algorithms, when applied to perfectly adapted dynamical systems at equilibrium, can be misleading.

We illustrate our ideas on three simple dynamical systems with feedback: a filling bathtub model \citep{Iwasaki1994,Dash2005}, a viral infection model \citep{Boer2012,BlomMooij_UAI_22}, and a chemical reaction network \citep{Ma2009}. We point out how perfect adaptation may also manifest itself in protein signalling networks, and take a closer look at the consequences for causal discovery from a popular protein expression data set \citep{Sachs2005}.  We adapt a model for the RAS-RAF-MEK-ERK signalling pathway from \citep{Shin2009} and study its properties  under certain saturation conditions. We test its predictions both in simulations and on real-world data. We propose that the phenomenon of perfect adaptation might explain why the presence and orientation of edges in the output of causal discovery algorithms does not always agree with the direction of edges in biological consensus networks that are based on a partial representation of the underlying dynamical mechanisms.

\section{Background}
\label{sec:background}

In this section, we provide an overview of the relevant background material on which this work is based. We first consider the assumptions underpinning popular constraint-based causal discovery algorithms and give a brief description of a simple local causal discovery algorithm in Section~\ref{sec:background:local causal discovery}. In Section~\ref{sec:background:causal ordering}, we proceed with a concise introduction to the causal ordering algorithm of Simon \cite{Simon1953} and demonstrate how it can be applied to a set of equations together with a pre-specified set of exogenous variables to deduce the implied causal relations and conditional independences. Finally, in Section~\ref{sec:background:related work}, we discuss the relationship of the present work with existing work.

\subsection{Causal discovery}
\label{sec:background:local causal discovery}

The main objective of causal discovery is to infer causal relations from experimental and observational data. The most common causal discovery algorithms can be roughly divided into score-based and constraint-based approaches. In this work, we will focus on the latter approach (examples include PC, FCI and variants thereof, see \citep{Cooper1997,Richardson1999,Spirtes2000,Zhang2008,Colombo2012,Mooij2020,Mooij2020b}), which exploits conditional independences in data to infer causal relations. We will first discuss assumptions for constraint-based causal discovery. We then consider the application of causal discovery algorithms to models with feedback. We proceed with a brief but concrete introduction to a simple local causal discovery algorithm. Finally, we discuss how the present work relates to existing literature. 

Learning a graphical structure from conditional independence constraints typically relies on Markov and faithfulness assumptions relating conditional independences to properties of a graph. In particular, a \emph{$d$-separation} is a relation between three disjoint sets of vertices in a graph that indicates whether all paths between two sets of vertices are $d$-blocked by the vertices in a third \citep{Pearl2000,Spirtes2000}. If every $d$-separation in a graph implies a conditional independence in the probability distribution then we say that the distribution satisfies the \emph{directed global Markov property} (or the \emph{$d$-separation} criterion) w.r.t.\ that graph \citep{Lauritzen1990}. Conversely, if every conditional independence in the probability distribution corresponds to a $d$-separation in a graph then we say that the distribution is \emph{$d$-faithful} to that graph.\footnote{While most authors refer to the notion of `$d$-faithful' simply as `faithful', we distinguish between two different notions, `$d$-faithful' and `$\sigma$-faithful'.} When a probability distribution satisfies the $d$-separation Markov property w.r.t.\ a graph and is also $d$-faithful to the graph, then this graph is a compact representation of the conditional independences in the probability distribution and we say that it \emph{encodes} its independence relations. A lot of work has been done to understand the (combinations of) various assumptions (e.g.\ linearity, Gaussianity, discreteness, causal sufficiency, acyclicity, no selection bias) under which a graph that encodes all conditional independences and dependences in a probability distribution has a certain causal interpretation \citep[see, e.g., ][]{Richardson1999, Spirtes2000, Lacerda2008, Zhang2008, Colombo2012, Hyttinen2012, Forre2018, Strobl2018, Mooij2020, Mooij2020b}. Perhaps the simplest assumption is that the data was generated by a causal DAG \citep{Spirtes2000,Pearl2000}.\footnote{A \emph{directed acyclic graph} (DAG) is a pair $\tuple{V,E}$ where $V$ is a set of vertices and $E \subseteq \{ i \to j : i, j \in V \}$ a set of directed edges between vertices such that there are no directed cycles.} This graphical object represents both the conditional independence structure (the observational probability distribution is assumed to satisfy the directed global Markov property with respect to the graph) and the causal structure (every directed edge that is absent in the DAG corresponds with the absence of a direct causal relation between the two variables, relative to the set of variables in the DAG).
For this acyclic setting, sophisticated constraint-based causal discovery algorithms (such as the PC and FCI algorithms \citep{Spirtes2000,Zhang2008}) have been developed.
The key assumption that these algorithms make is that the same DAG expresses both the conditional independences in the observational distribution and the causal relations between the variables.

However, many systems of interest in various scientific disciplines (e.g.\ biology, econometrics, physics) include feedback mechanisms. Cyclic Structural Causal Models (SCMs) can be used to model causal features and conditional independence relations of systems that contain cyclic causal relationships \citep{Bongers2020}. For linear SCMs with causal cycles, several causal discovery algorithms have been developed that are based on $d$-separations \citep{Richardson1999, Lacerda2008, Hyttinen2012, Strobl2018}. The $d$-separation criterion is applicable to acyclic settings and to cyclic SCMs with either discrete variables or linear relations between continuous variables, but it is too strong in general \citep{Spirtes1995}. A weaker Markov property, based on the notion of \emph{$\sigma$-separation}, has been derived for graphs that may contain cycles \citep{Spirtes1995,Forre2017,Bongers2020}. If every $\sigma$-separation in a graph implies a conditional independence in the probability distribution then we say that it satisfies the \emph{generalized directed global Markov property} (or the \emph{$\sigma$-separation} criterion) w.r.t.\ that graph \citep{Forre2017}. Conversely, if every conditional independence in the probability distribution corresponds to a $\sigma$-separation in a graph then we say that it is \emph{$\sigma$-faithful} to that graph. 
Richardson \cite{Richardson1996} already proposed a causal discovery algorithm that is sound under the generalized directed Markov property and the $d$-faithfulness assumption, assuming causal sufficiency.
Recently, a causal discovery algorithm was proposed based on the $\sigma$-separation criterion and the assumption of $\sigma$-faithfulness that is sound and complete \citep{Forre2018}. 
It was also shown that, perhaps surprisingly, the PC and FCI algorithms are still sound and complete in that setting, although their output has to be interpreted differently \citep{Mooij2020b}. 

For the sake of simplicity, we will limit our attention in this paper to one of the simplest causal discovery algorithms, Local Causal Discovery (LCD) \citep{Cooper1997}. This algorithm is a straightforward and efficient search method to detect a specific causal structure from background knowledge and observational or experimental data. The algorithm looks for triples of variables $(C, X, Y)$ for which (a) $C$ is a variable that is not caused by $X$, and (b) the following (in)dependences hold: $C\dep X$, $X\dep Y$, and $C\indep Y\mid X$. Figure~\ref{fig:lcd triple} shows all acyclic directed mixed graphs\footnote{An \emph{acyclic directed mixed graph} (ADMG) is a triple $\tuple{V,E,F}$ where $V$ is a set of vertices, $E \subseteq \{ i \to j : i, j \in V \}$ a set of directed edges between vertices such that there are no directed cycles, and $F \subseteq \{ i \leftrightarrow j : i,j \in V : i \ne j \}$ is a set of bidirected edges between vertices.} that correspond to the LCD triple $(C, X, Y)$, assuming $d$-faithfulness and no selection bias. They all have in common that there is no bidirected edge between $X$ and $Y$, while there is a directed edge from $X$ to $Y$. Hence, we can conclude that $X$ causes $Y$ and the two variables are not confounded. The algorithm was proven to be sound in both the $\sigma$- and $d$-separation settings even when cycles may be present \citep{Mooij2020}.
Even though this algorithm is not as sophisticated as many other causal discovery algorithms, it suffices for our exposition of the pitfalls of attempting causal discovery on data generated by a perfectly adaptive dynamical system.

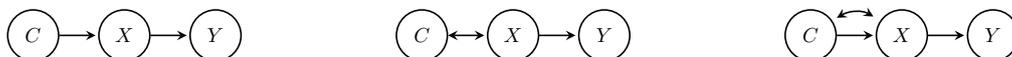
\begin{figure}[H]\centering
\begin{subfigure}[b]{0.33\textwidth}
	\centering
	\begin{tikzpicture}[->,>=stealth,shorten >=1pt,auto,node distance=1.5cm,
	semithick,square/.style={regular polygon,regular polygon sides=4}, scale=0.8,every node/.style={transform shape}]
	\node[state] (C) [] {$C$};
	\node[state] (X) [right of=C] {$X$};
	\node[state] (Y) [right of=X] {$Y$};	
	
	\path (C) edge node {} (X);
	\begin{scope}[]
	\path (X) edge node { } (Y);
	\end{scope}
	\end{tikzpicture}
\end{subfigure}%
\begin{subfigure}[b]{0.33\textwidth}
	\centering
	\begin{tikzpicture}[->,>=stealth,shorten >=1pt,auto,node distance=1.5cm,
	semithick,square/.style={regular polygon,regular polygon sides=4}, scale=0.8,every node/.style={transform shape}]
	\node[state] (C) [] {$C$};
	\node[state] (X) [right of=C] {$X$};
	\node[state] (Y) [right of=X] {$Y$};	
	
	\path (X) edge node {} (Y);
	\begin{scope}[style={<->}]
	\path ([yshift=0ex]C.east) edge[] node { } ([yshift=0ex]X.west);
	\end{scope}
	\end{tikzpicture}
\end{subfigure}%
\begin{subfigure}[b]{0.33\textwidth}
	\centering
	\begin{tikzpicture}[->,>=stealth,shorten >=1pt,auto,node distance=1.5cm,
	semithick,square/.style={regular polygon,regular polygon sides=4}, scale=0.8,every node/.style={transform shape}]
	\node[state] (C) [] {$C$};
	\node[state] (X) [right of=C] {$X$};
	\node[state] (Y) [right of=X] {$Y$};	
	
	\path (X) edge node {} (Y);
	\begin{scope}[]
	\path (C) edge node { } (X);
	\end{scope}
	\begin{scope}[style={<->}]
	\path ([yshift=2ex]C.east) edge[bend left] node { } ([yshift=2ex]X.west);
	\end{scope}
	\end{tikzpicture}
\end{subfigure}
  \caption{All acyclic directed mixed graphs that form an LCD triple $(C,X,Y)$, assuming $d$-faithfulness. In the absence of latent confounders and selection bias, the structure can only be that of the directed acyclic graph on the left.}
\label{fig:lcd triple}
\end{figure}

In this paper, we consider equilibrium distributions that are generated by dynamical models. The causal relations in an equilibrium model are defined through the effects of persistent interventions (i.e.\ interventions that are constant over time) on the equilibrium distribution, assuming that the system again converges to equilibrium. It has been shown that directed graphs encoding the conditional independences between endogenous variables in the equilibrium distribution of dynamical systems with feedback do not always have a straightforward and intuitive causal interpretation \citep{Iwasaki1994, Dash2005, Blom2020} (see also Appendix~\ref{app:causal interpretation mog}).
As a consequence, the output of causal discovery algorithms applied to equilibrium data of dynamical systems with feedback cannot always be interpreted causally in a na\"{i}ve way. In this work, we will show that this not only happens in isolated cases, but that this is actually a common phenomenon in a large class of models with perfectly adapting feedback mechanisms. 
In our opinion, a better understanding of how perfectly adapting feedback loops affect the causal interpretation of the conditional independence structure is a prerequisite for successful applications of causal discovery in fields like systems biology, where one often encounters perfect adaptivity. One way to arrive at such understanding is through the application of the causal ordering algorithm, the topic of the next section.

\subsection{Causal ordering}
\label{sec:background:causal ordering}

The causal ordering algorithm of Simon \cite{Simon1953} returns an ordering of endogenous variables that occur in a set of equations, given a specification of which variables are exogenous. The algorithm was recently reformulated so that the output is a causal ordering graph that encodes the generic effects of certain interventions and a Markov ordering graph that represents conditional independences (both under certain assumptions regarding the solvability of the equations) \citep{Blom2020}. We refer readers that are not yet familiar with the causal ordering algorithm to \citep{Blom2020} for a more extensive introduction to this approach. Here, we will provide only a brief description of the algorithm and discuss how its output should be interpreted.

First note that the structure of a set of equations and the variables that appear in them can be represented by a bipartite graph $\BG=\tuple{V,F,E}$, where vertices $F$ correspond to the equations and vertices $V$ correspond to the endogenous variables that appear in these equations. For each endogenous variable $v\in V$ that appears in an equation $f\in F$ there is an undirected edge $(v-f)\in E$. The output of the causal ordering algorithm is a directed cluster graph $\tuple{\V,\E}$, consisting of a partition $\V$ of the vertices $V\cup F$ into clusters, and edges $(v\to S)\in\E$ that go from vertices $v\in V$ to clusters $S\in \V$. Application of the causal ordering algorithm to a bipartite graph $\BG=\tuple{V,F,E}$ results in the directed cluster graph $\mathrm{CO}(\BG)=\tuple{\V,\E}$, which we will call the \emph{causal ordering graph} \citep{Blom2020}. For simplicity, we assume here that the bipartite graph has a \emph{perfect matching} (i.e.\ there exists a subset $M\subseteq E$ of the edges in the bipartite graph $\BG=\tuple{V,F,E}$ so that every vertex in $V \cup F$ is adjacent to exactly one edge in $M$).\footnote{Although the causal ordering algorithm has been extended to general bipartite graphs that may not have a perfect matching \citep{Blom2020}, we will always assume in this work that a perfect matching exists to keep the exposition simple.}
 The causal ordering graph is constructed by the following steps:\footnote{Here, we actually give the reformulation of the causal ordering algorithm by Nayak \cite{Nayak1995} based on the block triangular form of matrices in \citep{Pothen1990}. It was shown that the output of this algorithm is equivalent to that of the original causal ordering algorithm \citep{Blom2020}, but computationally more efficient than the original causal ordering algorithm \citep{Goncalves2016}. } 
\begin{enumerate}
	\item Find a perfect matching $M\subseteq E$ and let $M(S)$ denote the vertices in $V\cup F$ that are joined to vertices in $S\subseteq V\cup F$ by an edge in $M$.\label{alg:step 1}
	\item For each $(v-f)\in E$ with $v\in V$ and $f\in F$: if $(v-f)\in M$ orient the edge as $(v\leftarrow f)$ and if $(v-f)\notin M$ orient the edge as $(v\rightarrow f)$. Let $\G(\BG,M)$ denote the resulting directed bipartite graph. \label{alg:step 2}
  \item Partition vertices $V\cup F$ into strongly connected components $\V'$ of $\G(\BG,M)$. Create the cluster set $\V$ consisting of clusters $S\cup M(S)$ for each $S\in \V'$. For each edge $(v\to f)$ in $\G(\BG,M)$ add an edge $(v\to \mathrm{cl}(f))$ to $\E$ when $v\notin \mathrm{cl}(f)$, where $\mathrm{cl}(f)$ denotes the cluster in $\V$ that contains $f$. \label{alg:step 3}
  \item Exogenous variables appearing in the equations are added as singleton clusters to $\V$, with directed edges towards the clusters of the equations in which they appear in $\E$. \label{alg:step 4}
  \item Output the directed cluster graph $\mathrm{CO}(\BG) = \tuple{\V,\E}$.
\end{enumerate}

\begin{example}
	\label{ex:causal ordering}
	Consider the following set of equations with index set $F=\{f_1,f_2\}$ that contain endogenous variables with index set $V=\{v_1,v_2\}$:
	\begin{align}
	\label{eq:co ex1}
	f_1:\qquad & X_{v_1} - U_{w_1} =0, \\
	\label{eq:co ex2}
	f_2:\qquad & X_{v_2} + X_{v_1} - U_{w_2} =0,
	\end{align}
	where $U_{w_1}$ and $U_{w_2}$ are exogenous (random) variables indexed by $W=\{w_1,w_2\}$. Figure~\ref{fig:co example bipartite graph} shows the associated bipartite graph $\BG=\tuple{V,F,E}$. This graph has exactly one perfect matching $M=\{(v_1-f_1),(v_2-f_2)\}$, which is used in step \ref{alg:step 2} of the causal ordering algorithm to construct the directed graph $\G(\BG,M)$ in Figure~\ref{fig:co example oriented graph}. The causal ordering graph that is obtained after applying steps \ref{alg:step 3} and \ref{alg:step 4} of the causal ordering algorithm is given in Figure~\ref{fig:co example causal ordering graph}.
\end{example}

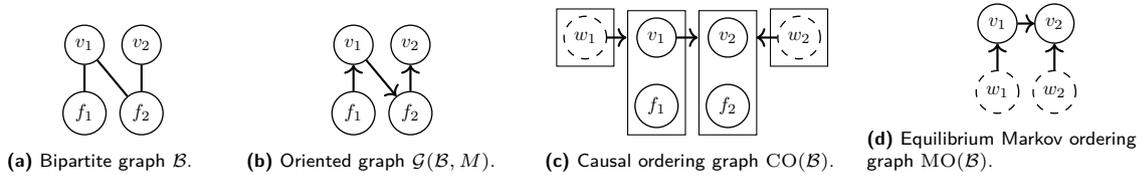
\begin{figure}[ht]
  \centering
	\begin{subfigure}[b]{0.18\textwidth}
		\centering
		\begin{tikzpicture}[scale=0.75,every node/.style={transform shape}] 
		\GraphInit[vstyle=Normal]
		\SetGraphUnit{1}
		\Vertex[L=$v_1$,x=0,y=0] {v1}
		\Vertex[L=$v_2$, x=1.0, y=0] {v2}
		\Vertex[L=$f_1$, x=0.0, y=-1.2] {f1}
		\Vertex[L=$f_2$, x=1.0, y=-1.2] {f2}
		
		\draw[EdgeStyle, style={-}](v1) to (f1);
		\draw[EdgeStyle, style={-}](v1) to (f2);
		\draw[EdgeStyle, style={-}](v2) to (f2);
		\end{tikzpicture}
		\caption{Bipartite graph $\BG$.}
		\label{fig:co example bipartite graph}
	\end{subfigure}\hfill
	\begin{subfigure}[b]{0.23\textwidth}
		\centering
		\begin{tikzpicture}[scale=0.75,every node/.style={transform shape}] 
		\GraphInit[vstyle=Normal]
		\SetGraphUnit{1}
		\Vertex[L=$v_1$,x=0,y=0] {v1}
		\Vertex[L=$v_2$, x=1.0, y=0] {v2}
		\Vertex[L=$f_1$, x=0.0, y=-1.2] {f1}
		\Vertex[L=$f_2$, x=1.0, y=-1.2] {f2}
		
		\draw[EdgeStyle, style={<-}](v1) to (f1);
		\draw[EdgeStyle, style={->}](v1) to (f2);
		\draw[EdgeStyle, style={<-}](v2) to (f2);
		\end{tikzpicture}
    \caption{Oriented graph $\G(\BG,M)$.}
		\label{fig:co example oriented graph}
	\end{subfigure}\hfill
	\begin{subfigure}[b]{0.25\textwidth}
		\centering
		\begin{tikzpicture}[scale=0.75,every node/.style={transform shape}] 
		\GraphInit[vstyle=Normal]
		\SetGraphUnit{1}
		\Vertex[L=$w_1$,style={dashed},x=-1.25,y=0] {w1}
		\Vertex[L=$v_1$,x=0,y=0] {v1}
		\Vertex[L=$v_2$, x=1.25, y=0] {v2}
		\Vertex[L=$f_1$, x=0.0, y=-1.2] {f1}
		\Vertex[L=$f_2$, x=1.25, y=-1.2] {f2}
		\Vertex[L=$w_2$,style={dashed},x=2.5,y=0] {w2}
		
		\node[draw=black, fit=(w1), inner sep=0.09cm ]{};
		\node[draw=black, fit=(v1) (f1), inner sep=0.09cm ]{};
		\node[draw=black, fit=(v2) (f2), inner sep=0.09cm ]{};
		\node[draw=black, fit=(w2), inner sep=0.09cm ]{};
		
		\draw[EdgeStyle, style={->}] (w1) to (-0.525,0.0);
		\draw[EdgeStyle, style={->}] (v1) to (0.725,0.0);
		\draw[EdgeStyle, style={->}] (w2) to (1.75,0.0);
		\end{tikzpicture}
    \caption{Causal ordering graph $\mathrm{CO}(\BG)$.}
		\label{fig:co example causal ordering graph}
	\end{subfigure}\hfill
	\begin{subfigure}[b]{0.27\textwidth}
		\centering
		\begin{tikzpicture}[scale=0.75,every node/.style={transform shape}] 
		\GraphInit[vstyle=Normal]
		\SetGraphUnit{1}
		\Vertex[L=$v_1$,x=0,y=0] {v1}
		\Vertex[L=$v_2$, x=1.0, y=0] {v2}
		\Vertex[L=$w_1$, x=0.0, y=-1.2, style={dashed}] {f1}
		\Vertex[L=$w_2$, x=1.0, y=-1.2, style={dashed}] {f2}
		
		\draw[EdgeStyle, style={<-}](v1) to (f1);
		\draw[EdgeStyle, style={->}](v1) to (v2);
		\draw[EdgeStyle, style={<-}](v2) to (f2);
		\end{tikzpicture}
    \caption{Equilibrium Markov ordering graph $\mathrm{MO}(\BG)$.}
		\label{fig:co example markov ordering graph}
	\end{subfigure}%
  \caption{Several graphs occurring in Example~\ref{ex:causal ordering}. The bipartite graph $\BG$ associated with equations \eqref{eq:co ex1} and \eqref{eq:co ex2} is given in Figure~\ref{fig:co example bipartite graph}. The oriented graph $\G(\BG,M)$ obtained in step \ref{alg:step 2} of the causal ordering algorithm (with perfect matching $M$) is shown in Figure~\ref{fig:co example oriented graph}. The causal ordering graph $\mathrm{CO}(\BG)$ is given in Figure~\ref{fig:co example causal ordering graph}. The corresponding equilibrium Markov ordering graph $\mathrm{MO}(\BG)$ is displayed in Figure~\ref{fig:co example markov ordering graph}.}
	\label{fig:co example}
\end{figure}

Throughout this work, we will assume that sets of equations are \emph{uniquely solvable with respect to the causal ordering graph}, which means that for each cluster, the equations in that cluster can be solved uniquely for the endogenous variables in that cluster (see Definition 14 in \citep{Blom2020} for details).
This implies amongst others that the endogenous variables in the model can be solved uniquely along a topological ordering of the causal ordering graph. Under this assumption, the causal ordering graph represents the effects of soft and certain perfect interventions \citep{Blom2020}. \emph{Soft interventions} target equations; they do not change which variables appear in the targeted equation and may only alter the parameters or functional form of the equation. \emph{Perfect interventions} target clusters in the causal ordering graph and replace the equations in the targeted cluster with equations that set the variables in the cluster equal to constant values.  A soft intervention on an equation or a perfect intervention on a cluster has no effect on a variable in the causal ordering graph whenever there is no directed path to that variable from the intervention target (i.e.\ the targeted equation or an arbitrary vertex in the targeted cluster, respectively), see Theorems 20 and 23 in \citep{Blom2020}.\footnote{We say that there is a \emph{directed path} from vertex $x$ to vertex $y$ in a directed cluster graph $\tuple{\V,\E}$ if either $\mathrm{cl}(x)=\mathrm{cl}(y)$ or there is a sequence of clusters $V_1=\mathrm{cl}(x),V_2,\ldots, V_{k-1}, V_k=\mathrm{cl}(y)$ so that for all $i\in\{1,\ldots,k-1\}$ there is a vertex $z_i\in V_i$ such that $(z_i\to V_{i+1})\in\mathcal{E}$.}
Since the equations in Example~\ref{ex:causal ordering} are uniquely solvable w.r.t.\ the causal ordering graph in Figure~\ref{fig:co example causal ordering graph},\footnote{Indeed, we can solve equation $f_1$ uniquely for $X_{v_1}$ in terms of $U_{w_1}$ resulting in $X_{v_1} = U_{w_1}$, and we can solve equation $f_2$ uniquely for $X_{v_2}$ in terms of $X_{v_1}$ and $U_{w_2}$ resulting in $X_{v_2} = U_{w_2} - X_{v_1}$.} we can for example read off from the causal ordering graph that a soft intervention targeting $f_1$ may have an effect on $X_{v_2}$ (since there is a directed path from $f_1$ to $v_2$), and that a perfect intervention targeting the cluster $\{v_2,f_2\}$ has no effect on $X_{v_1}$ (as there is no directed path from the cluster $\{v_2,f_2\}$ to $v_1$).

Given the probability distribution of the exogenous random variables, one obtains a unique probability distribution on all the variables under the assumption of unique solvability w.r.t.\ the causal ordering graph. The \emph{Markov ordering graph} is a directed acyclic graph $\mathrm{MO}(\BG)$ such that $d$-separations in this graph imply corresponding conditional independences between variables in this joint distribution, provided that all exogenous random variables are independent \citep{Blom2020}. The Markov ordering graph is obtained from a causal ordering graph $\mathrm{CO}(\BG)=\tuple{\V,\E}$ by constructing a graph $\tuple{V', E'}$ with vertices $V'=V$ and edges $(v\to w)\in E'$ if and only if $(v\to \mathrm{cl}(w))\in \mathcal{E}$. 
The Markov ordering graph for the set of equations in Example~\ref{ex:causal ordering} is given in Figure~\ref{fig:co example markov ordering graph}. The $d$-separations in this graph imply conditional independences between the corresponding variables under the assumption that $U_{w_1}$ and $U_{w_2}$ are independent. For instance, since $v_1$ and $w_2$ are $d$-separated we can read off from the Markov ordering graph that $X_{v_1}$ and $U_{w_2}$ must be independent.

Assuming that the probability distribution is $d$-faithful to the Markov ordering graph and that we have a conditional independence oracle, we know that the output of the PC algorithm represents the Markov equivalence class of the Markov ordering graph \citep{Mooij2020b}.\footnote{More generally, if only a subset of the variables is observed, the output of the FCI algorithm represents the  Markov equivalence class of the marginalization (latent projection) of the Markov ordering graph.} However, while for acyclic systems, the directed edges in the Markov ordering graph can also be interpreted as direct causal relations, this is not the case for all systems \citep{Blom2020,BlomMooij_UAI_22}. Several examples of this phenomenon are provided in Appendix~\ref{app:causal interpretation mog}. In this work we will show that such discrepancy between the causal and the Markov structure is actually a common feature of perfectly adapted dynamical systems at equilibrium. 

\subsection{Related work}
\label{sec:background:related work}

The causal ordering algorithm can be applied, for instance, to the equilibrium equations of a dynamical system to uncover its causal properties and conditional independence relations at equilibrium. The relationship between dynamical models and causal models has received much attention over the years. For instance, the works of \citep{Fisher1970, Iwasaki1994, Voortman2010,Sokol2014,Rubenstein2018, Bongers2018, Mogensen2018} considered causal relations in dynamical systems that are not at equilibrium, while \citep{Iwasaki1994,Dash2005,Mooij2013, Hyttinen2012, Lauritzen2002, Lacerda2008,Mooij2012,Bongers2018, Blom2019,BlomMooij_UAI_22} considered graphical and causal models that arise from studying the stationary behavior of dynamical models. In particular, extensions of the causal ordering algorithm for dynamical systems were proposed and discussed in \citep{Iwasaki1994}. Also, the causal ordering algorithm was applied in \citep{BlomMooij_UAI_22} to study the robustness of model predictions when combining two systems. The relationship between the causal semantics of a dynamical system before it reaches equilibrium and at equilibrium has also been studied \citep{Dash2005,Mooij2013,Bongers2018}. 

In previous work, researchers have noted various problems when attempting to use a single graphical model to represent both conditional independence properties and causal properties of certain dynamical systems at equilibrium \citep{Dash2005, Lacerda2008, Lauritzen2002, Dawid2010, Blom2019}. Often, restrictive assumptions on the underlying dynamical models are made to avoid these subtleties---the most common one being to exclude the possibility of cycles altogether. In this work, we will not make such restrictive assumptions and instead show that such problems are pervasive in the important class of perfectly adapted dynamical systems. We follow \citep{Blom2020,BlomMooij_UAI_22} in addressing the issues by using the causal ordering algorithm to construct separate graphical representations for the causal properties and conditional independence relations implied by these systems. 

It has been shown that the popular SCM framework \citep{Bongers2020, Pearl2000} is not flexible enough to fully capture the causal semantics (in terms of perfect interventions targeting variables) of the dynamics of a basic enzyme reaction at equilibrium, and for that purpose \citep{Blom2019} proposed to use Causal Constraints Models (CCMs) instead. However, CCMs lack a graphical representation (and consequently, all the benefits that usually come with it, like a Markov property and a graphical approach to causal reasoning). The techniques in \citep{Blom2020} can also be used to construct graphical representations of causal relations and conditional independences of these models. In Appendix~\ref{app:rewriting equations}, we demonstrate that the basic enzyme reaction is perfectly adapting and we show how the causal ordering technique can be used to obtain graphical presentations and a Markov property for this model. This approach offers several advantages over the CCMs approach to model this system.

An application area where obtaining a causal understanding of complex systems is often non-trivial due to feedback and perfect adaptation is that of systems biology. 
A research question that has seen considerable interest in that field is which reaction networks can achieve perfect adaptation \citep{Araujo2018, Muzzey2009, Ferrell2016, Ma2009, Krishnan2019}.
The present work provides a method that facilitates the analysis of perfectly adapted dynamical systems by providing a principled and computationally efficient procedure to identify perfect adaptation either from model equations or from experimental data and background knowledge.

In Section~\ref{sec:application}, we apply our methodology in an attempt to arrive at a better understanding of the causal mechanisms present in protein signalling networks. 
For protein signalling networks, apparent `causal reversals' have been reported, that is, cases where causal discovery algorithms find the opposite causal relation of what is expected \citep{Triantafillou2017, Mooij2020, Mooij2013b, Ramsey2018, Boeken2020}. 
One explanation for these reversed edges in the output of causal discovery algorithm could be the unknown presence of measurement error \citep{Blom2018}. 
As we demonstrate in this work, unknown feedback loops that result in perfect adaptation can be another reason why one might find reversed causal relations when applying causal discovery algorithms to biological data. 

\section{Perfect adaptation}
\label{sec:perfect adaptation}

In this section, we introduce the notion of perfect adaptation by taking a close look at several examples of dynamical systems that can achieve perfect adaptation. We then consider graphical representations of these systems both before and after they have reached equilibrium. This will set the stage for our main theoretical results regarding the identification of perfect adaptation in models or data, which will be presented in Section~\ref{sec:perfect adaptation:identification}.
The goals of this section are: (i) building intuition about mechanisms that result in perfect adaptation, (ii) outlining the relevance of this phenomenon in application domains, and (iii) explaining how the causal ordering algorithm helps to understand perfect adaptation.

The ability of (part of) a system to converge to its original state when a constant and persistent external stimulus is added or changed is referred to as \emph{perfect adaptation}. If the adaptive behavior does not depend on the precise setting of the system's parameters then we say that the adaptation is \emph{robust}. For our purposes, the most interesting of the two is robust perfect adaptation. As this is also often simply referred to as `perfect adaptation' in the literature, we will do so here as well.

\subsection{Examples}
\label{sec:perfect adaptation:examples}

We consider three different dynamical models corresponding to a filling bathtub, a viral infection with an immune response, and a chemical reaction network. We use simulations to demonstrate that all of these systems are capable of achieving perfect adaptation. The details of the simulations presented in this section are provided in Appendix~\ref{app:perfect adaptation simulations}, and code to reproduce these is provided by the authors under a free and open source license (see Data availability statement for details).

\begin{figure}[ht]\centering
	\begin{subfigure}{0.25\textwidth}
		\centering
    \includegraphics[width=\textwidth]{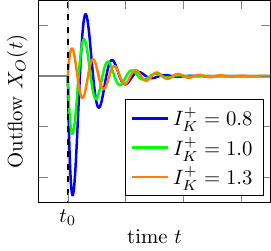}
    {\small $$I_K(t) = \begin{cases}
      I_K^- & t < t_0 \\
      I_K^+ & t \ge t_0 \\
    \end{cases}$$}
		\caption{Filling bathtub model.}
		\label{fig:adaptation:bathtub}
	\end{subfigure}%
  \qquad\qquad
	\begin{subfigure}{0.25\textwidth}
		\centering
		\includegraphics[width=\textwidth]{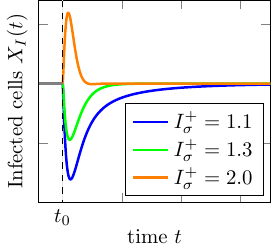}
    {\small $$I_\sigma(t) = \begin{cases}
      I_\sigma^- & t < t_0 \\
      I_\sigma^+ & t \ge t_0 \\
    \end{cases}$$}
		\caption{Viral infection model.}
		\label{fig:adaptation:viral infection}
	\end{subfigure}%
	\qquad\qquad
	\begin{subfigure}{0.25\textwidth}
		\centering
		\includegraphics[width=\textwidth]{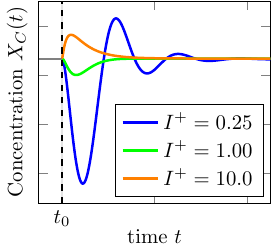}
    {\small $$I(t) = \begin{cases}
      I^- & t < t_0 \\
      I^+ & t \ge t_0 \\
    \end{cases}$$}
		\caption{Reaction network model.}
		\label{fig:adaptation:NFBLB}
	\end{subfigure}
  \caption{Simulations of the outflow rate $X_O(t)$ in the bathtub model (\subref{fig:adaptation:bathtub}), the amount of infected cells $X_I(t)$ in the viral infection model (\subref{fig:adaptation:viral infection}), and the concentration $X_C(t)$ in the biochemical reaction network with a negative feedback loop (\subref{fig:adaptation:NFBLB}). The input signal suddenly changes from an (constant) value for $t < t_0$ to a different (constant) value for $t \ge t_0$. The timing of this change is indicated by a vertical dashed line. The three systems started with input signals $I_K^-=1.2$, $I_\sigma^-=1.6$, and $I^-=1.5$, respectively. After a transient response, $X_O(t), X_I(t)$, and $X_C(t)$ all converge to their original equilibrium value (i.e.,\ they perfectly adapt to the input signal).}
	\label{fig:adaptation}
\end{figure}

\subsubsection{Filling bathtub}
\label{sec:perfect adapatation:examples:bathtub}

We consider the example of a filling bathtub of Iwasaki \& Simon \cite{Iwasaki1994} (see also \citep{Dash2005,Blom2020}). Let $I_K(t)$ be an input signal\footnote{In this work, `input signal' will always refer to an exogenous process, that is, a variable whose value may depend on time and is determined by some mechanism that is external to the system. In particular, it must not be causally influenced by the system.} that represents the size of the drain in the bathtub. The inflow rate $X_I(t)$, water level $X_D(t)$, water pressure $X_P(t)$, and outflow rate $X_O(t)$ are modelled by the following static and dynamic equations:\footnote{With the term `static equation' we refer to equations that do not contain any time derivatives. In this article, the term `dynamic equation' always corresponds to a first-order differential equation.}
\begin{align}
X_{I}(t) &= U_I,\label{eq:bathtub:dyn I}\\
\dot{X}_D(t) &= U_1(X_{I}(t)-X_{O}(t)),\label{eq:bathtub:dyn D}\\
\dot{X}_P(t) &= U_2(g\,U_3X_D(t)-X_P(t)),\label{eq:bathtub:dyn P}\\
\dot{X}_O(t) &= U_4(U_5 I_K(t) X_P(t)-X_{O}(t)),\label{eq:bathtub:dyn O}
\end{align}
where $g$ is the gravitational acceleration, and $U_I,U_1,U_2,U_3,U_4,U_5$ are independent exogenous random variables taking value in $\R_{>0}$. Let $X_D, X_P$, and $X_O$ denote the respective equilibrium solutions for the water level, water pressure, and outflow rate. The equilibrium equations associated with this model can easily be constructed by setting the time derivatives equal to zero and assuming the input signal $I_K(t)$ to have a constant value $I_K$:
\begin{align}
f_I:\qquad& X_{I} - U_I = 0, \label{eq:bathtub:eq I}\\
f_D:\qquad& U_1(X_{I}-X_{O}) = 0, \label{eq:bathtub:eq D}\\
f_P:\qquad& U_2(g\,U_3X_D-X_P) = 0, \label{eq:bathtub:eq P}\\
f_O:\qquad& U_4(U_5 I_K X_P-X_{O}) = 0.\label{eq:bathtub:eq O}
\end{align}
We call the labelling $f_D, f_P, f_O$ that we choose for the equilibrium equations that are constructed from the time derivatives the \emph{natural labelling} for this dynamical system.\footnote{In general, the natural labeling uses the same index $i$ to label as $f_i$ the equilibrium equation $0 = h_i(X(t))$ obtained from the dynamic equation $\dot{X}_{v_i}(t) = h_i(X(t))$ that models the dynamics of variable $X_{v_i}$ \citep{Mooij2013}.} A solution $(X_I, X_D, X_P, X_O)$ to the system of equilibrium equations satisfies $X_I=U_I$ and $X_O=X_I$ almost surely. From this we conclude that, at equilibrium, the outflow rate is independent of the size of the drain $I_K$, assuming that $U_I$ is independent of $I_K$. We recorded the changes in the system after we changed the input signal $I_K$ of the bathtub system in equilibrium. The results in Figure~\ref{fig:adaptation:bathtub} show that the outflow rate $X_O$ has a transient response to changes in the input signal $I_K$,\footnote{More precisely, when we speak of a `response' to an input signal, what we mean is that if we have two identical copies of a system, and from some point in time $t_0$ on, we change the input signal of one of the two copies, the endogenous response variables differ in distribution at some later time $t > t_0$.} but it eventually converges to its original value. We say that the outflow rate $X_O$ in the bathtub model \emph{perfectly adapts} to changes in $I_K$.

\subsubsection{Viral infection model}
\label{sec:perfect adaptation:examples:viral infection}

We consider the example of a simple dynamical model for a viral infection and immune response of De Boer \cite{Boer2012} (also discussed in \citep{BlomMooij_UAI_22}). The model describes target cells $X_T(t)$, infected cells $X_I(t)$, and an immune response $X_E(t)$. We will treat $I_{\sigma}(t)$ as an exogenous input signal that represents the production rate of target cells. The system is defined by the following dynamic equations:
\begin{alignat}{3}
\label{eq:simple T}
\dot{X}_T(t) &= I_{\sigma}(t) - d_T X_T(t) - \beta X_T(t) X_I(t), \\
\label{eq:simple I}
\dot{X}_I(t) &= (\beta X_T(t) - d_I - k X_E(t)) X_I(t),\\
\label{eq:simple E}
\dot{X}_E(t) &= (a X_I(t) - d_E) X_E(t).
\end{alignat}
Here, $\beta=\frac{bp}{c}$ where $b$ is the infection rate, $p$ is the number of virus particles produced per infected cell, and $c$ is the clearance rate of viral particles. Furthermore, $d_T$ is the death rate of target cells, $a$ is an activation rate, $d_E$ and $d_I$ are turnover rates and $k$ is a mass-action killing rate. We assume that $a, k$ are constants and that $d_T$, $d_I$, $d_E$, and $\beta$ are independent exogenous random variables. We use the natural labelling for the equilibrium equations that are constructed from the differential equations:\footnote{Following \citep{Boer2012}, we are only interested in strictly positive solutions of this dynamical system. Therefore, we use the equilibrium equation $f_I$ instead of $(f\beta X_T - d_I - k X_E) X_I = 0$ and $f_E$ instead of $(a X_I - d_E) X_E=0$.}
\begin{alignat}{3}
f_T:\qquad& I_{\sigma} - d_T X_T - \beta X_T X_I = 0, \\
\label{eq:eq I}
f_I:\qquad& \beta X_T - d_I - k X_E = 0,\\
\label{eq:eq E}
f_E:\qquad& a X_I - d_E = 0,
\end{alignat}
assuming a constant value $I_{\sigma}$ of the input signal. We initialized the model in an equilibrium state and simulated the response of the model after changing the input signal $I_{\sigma}$ to three different values. Figure~\ref{fig:adaptation:viral infection} shows that the amount of infected cells $X_I(t)$ has a transient response to a change in the input signal $I_\sigma$, but then returns to its original value. Hence, the amount of infected cells perfectly adapts to changes in the production rate of target cells.

\subsubsection{Reaction networks with a negative feedback loop}
\label{sec:perfect adaptation:examples:negative feedback loop}

The phenomenon of perfect adaptation is a common feature in biochemical reaction networks and there exist many reaction networks that can achieve (near) perfect adaptation \citep{Araujo2018, Ferrell2016}. For networks consisting of only three nodes, Ma \emph{et al.} \cite{Ma2009} found by an exhaustive search that there exist two major classes of reaction networks that produce (robust) adaptive behavior. The reaction diagrams for these networks are given in Figure~\ref{fig:robust network topologies}. Here we will only analyze the `Negative Feedback with a Buffer Node' (NFBN) network. An analysis of the other network, the `Incoherent Feed-forward Loop with a Proportioner Node' (IFFLP), is provided in Appendix~\ref{app:ifflp network}. The NFBN system can be described by the following first-order differential equations:
\begin{align}
\dot{X}_A(t) &= I(t) k_{IA} \frac{(1-X_A(t))}{K_{IA} + (1-X_A(t))} - F_A k_{F_A A} \frac{X_A(t)}{K_{F_A A} + X_A(t)}, \\
  \dot{X}_B(t) &= X_C(t) k_{CB} \frac{(1-X_B(t))}{K_{CB} + (1-X_B(t))} - F_B k_{F_B B} \frac{X_B(t)}{K_{F_B B} + X_B(t)}, \label{eq:buffer exact}\\
\dot{X}_C(t) &= X_A(t) k_{AC} \frac{(1-X_C(t))}{K_{AC} + (1-X_C(t))} - X_B(t) k_{BC} \frac{X_C(t)}{K_{BC} + X_C(t)},
\end{align}
where $X_A(t)$, $X_B(t)$, $X_C(t)$ are concentrations of three compounds $A$, $B$, and $C$, while $I(t)$ represents an external input into the system. Assume that $k_{IA}$, $k_{CB}$, and $k_{AC}$ are independent exogenous random variables, that we will denote as $U_A$, $U_B$, $U_C$ respectively, and that the other parameters are constants. Perfect adaptation is achieved under saturation conditions \citep{Ma2009}, $(1-X_B(t))\gg K_{CB}$ and $X_B(t)\gg K_{F_B B}$, in which case the following approximation can be made:
\begin{align}
\label{eq:buffer approx}
\dot{X}_B(t) &\approx X_C(t) k_{CB} - F_B k_{F_B B}.
\end{align}
Under the assumption that $I(t)$ has a constant value, the system converges to an equilibrium. We will denote the equilibrium equations that are associated with the time derivatives $\dot{X}_A(t)$ and $\dot{X}_C(t)$ using the natural labelling $f_A$ and $f_C$. The equilibrium equation $f_B$ is obtained by setting the approximation of the time derivative $\dot{X}_B(t)$ in \eqref{eq:buffer approx} equal to zero. We initialized this model in an equilibrium state and then simulated its response after changing the input signal $I$ to three different values. Figure~\ref{fig:adaptation:NFBLB} shows that $X_C(t)$ perfectly adapts to changes in the input signal $I$.

\begin{figure}[ht]
  \centering
	\begin{subfigure}[b]{0.35\textwidth}
		\centering
		\begin{tikzpicture}[scale=0.7,every node/.style={transform shape}] 
		\node[style={rectangle, very thick}] (I) {Input};
		\node[style={circle, draw=black, very thick}] (A) [left=0.5cm of I] {$A$};
		\node[style={circle, draw=black, very thick}] (C) [below=0.5cm of A] {$C$};
		\node[style={circle, draw=black, very thick}] (B) [left=0.75cm of A] {$B$};
		\node[style={rectangle, very thick}] (O) [right=0.5 cm of C] {Output};
		\node[style={rectangle, very thick}] (D) [left=0.75cm of B] {};
		\path
		(I) edge [style={->, very thick}] node [] {} (A)
		(A) edge [style={->, very thick}] node [] {} (C)
		(C) edge [style={->, very thick}] node [] {} (O)
		(B) edge [bend right, style={-o, very thick}] node [] {} (C)
		(C) edge [bend right, style={->, very thick}, color={orange}] node [] {} (B)
		(D) edge [style={-o, dashed, very thick}, color={orange}] node [] {} (B);
		\end{tikzpicture}
		\caption{Negative feedback with a buffer node.}
		\label{fig:NFBLB}
	\end{subfigure}\qquad
\begin{subfigure}[b]{0.45\textwidth}
		\centering
		\begin{tikzpicture}[scale=0.7,every node/.style={transform shape}] 
		\node[style={rectangle, very thick}] (I) {Input};
		\node[style={circle, draw=black, very thick}] (A) [left=0.5cm of I] {$A$};
		\node[style={circle, draw=black, very thick}] (C) [below=0.5cm of A] {$C$};
		\node[style={circle, draw=black, very thick}] (B) [left=0.75cm of A] {$B$};
		\node[style={rectangle, very thick}] (O) [right=0.5 cm of C] {Output};
		\node[style={rectangle, very thick}] (D) [left=0.75cm of B] {};
		\path
		(I) edge [style={->, very thick}] node [] {} (A)
		(A) edge [style={->, very thick}] node [] {} (C)
		(C) edge [style={->, very thick}] node [] {} (O)
		(B) edge [bend right, style={-o, very thick}] node [] {} (C)
		(A) edge [style={->, very thick}, color={orange}] node [] {} (B)
		(D) edge [style={-o, dashed, very thick}, color={blue}] node [] {} (B);
		\end{tikzpicture}
		\caption{Incoherent feed-forward loop with a proportioner node.}
		\label{fig:IFFLP}
	\end{subfigure}
	\caption{The two reaction networks that can achieve perfect adaptation \citep{Ma2009}. Figure~\ref{fig:NFBLB} shows Negative Feedback with a Buffer Node (NFBN), while Figure~\ref{fig:IFFLP} shows an Incoherent Feed-forward Loop with a Proportioner Node (IFFLP). Orange edges represent saturated reactions, blue edges represent linear reactions, and black edges are unconstrained reactions. Arrowheads represent positive influence and edges ending with a circle represent negative influence.}
	\label{fig:robust network topologies}
\end{figure}
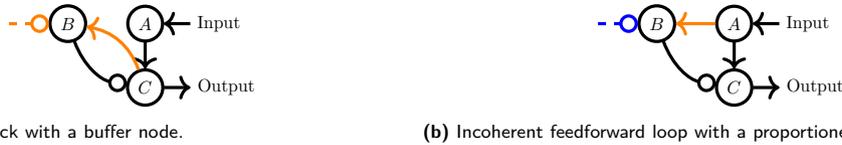

\subsection{Graphical representations}
\label{sec:perfect adaptation:graphical representations}

We will now provide the different graphical representations of the perfectly adapted dynamical systems that were introduced in the previous section. These representations are based on the graphs that are used in \citep{Iwasaki1994, Blom2020} to encode the equilibrium structure of equations, causal relations, and conditional independences. The main difference with previous work is that we also explicitly consider similar graphical representations for systems that have not yet reached equilibrium.\\

\noindent
\textit{Bipartite graph}: The \emph{equilibrium bipartite graphs} associated with the equilibrium equations of the filling bathtub, the viral infection, and the reaction network with feedback are given in Figures \ref{fig:equilibrium bipartite graph:bathtub}, \ref{fig:equilibrium bipartite graph:viral infection}, and \ref{fig:equilibrium bipartite graph:reaction network}, respectively. We have added also a node representing the input signal. The \emph{dynamic bipartite graphs} for the dynamics of these models are constructed from first-order differential equations in canonical form\footnote{A set of first-order differential equations is said to be in \emph{canonical form} if they are of the form $\dot{X}_i(t) = h_i(X_1(t),\dots,X_n(t))$ with $i=1,\dots,n$. That is, each differential equation has a single derivative on the left-hand side and a function of the variables (without any derivatives) on the right-hand side.} in the following way. Both the derivative $\dot{X}_i(t)$ and the corresponding variable $X_i(t)$ are associated with the same vertex $v_i$. The natural labelling is used for the differential equations, so that a vertex $g_i$ is associated with the differential equation for $\dot{X}_i(t)$. We then construct the dynamic bipartite graph $\BG_{\mathrm{dyn}}=\tuple{V,F,E}$ with variable vertices $v_i \in V$ and the corresponding dynamical equation vertices $g_i\in F$. Additional static equation vertices $f_i\in F$ are added as well in case the dynamical system consists of a combination of dynamic and static equations. The edge set $E$ has an edge $(v_i - f_j)$ whenever $X_i(t)$ appears in the static equation $f_j$. Additionally, there are edges $(v_i - g_j)$ whenever $X_i(t)$ or $\dot{X}_i(t)$ appears in the dynamic equation $g_j$ (which includes the cases $i=j$ due to the natural labelling used). The dynamic bipartite graphs for the bathtub model, the viral infection, and the reaction network with feedback are given in Figures \ref{fig:dynamic bipartite graph:bathtub}, \ref{fig:dynamic bipartite graph:viral infection}, and \ref{fig:dynamic bipartite graph:reaction network}, respectively.

Comparing the equilibrium bipartite graphs with the dynamic bipartite graphs we note that there is no edge $(v_D-f_D)$ in Figure~\ref{fig:equilibrium bipartite graph:bathtub} while $(v_D - g_D)$ is present in Figure~\ref{fig:dynamic bipartite graph:bathtub}. This is a direct consequence of the fact that the time derivative $\dot{X}_D(t)$ in equation~\eqref{eq:bathtub:dyn D} does not depend on $X_D(t)$ itself. Similarly, the edges $(v_I-f_I)$ and $(v_E-f_E)$ are not present in Figure~\ref{fig:equilibrium bipartite graph:viral infection} whilst the edges $(v_I-g_I)$ and $(v_E-g_E)$ are present in Figure~\ref{fig:dynamic bipartite graph:viral infection}. In this case, we see that even though the time derivatives $\dot{X}_I(t)$ and $\dot{X}_E(t)$ depend on $X_I(t)$ and $X_E(t)$ in differential equations \eqref{eq:simple I} and \eqref{eq:simple E}, these variables do not appear in the associated equilibrium equations \eqref{eq:eq I} and \eqref{eq:eq E}. Finally, there is no edge $(v_B - f_B)$ in Figure~\ref{fig:equilibrium bipartite graph:reaction network} while the edge $(v_B-g_B)$ is present in Figure~\ref{fig:dynamic bipartite graph:reaction network}. Here, the variable $X_B(t)$ does not appear in the equilibrium equation under saturation conditions \eqref{eq:buffer approx} that stems from the dynamic equation \eqref{eq:buffer exact} for $\dot{X}_B(t)$. The equilibrium bipartite graph can be compared to the dynamic bipartite graph to read off structural differences between the equations before and after equilibrium has been reached.\\

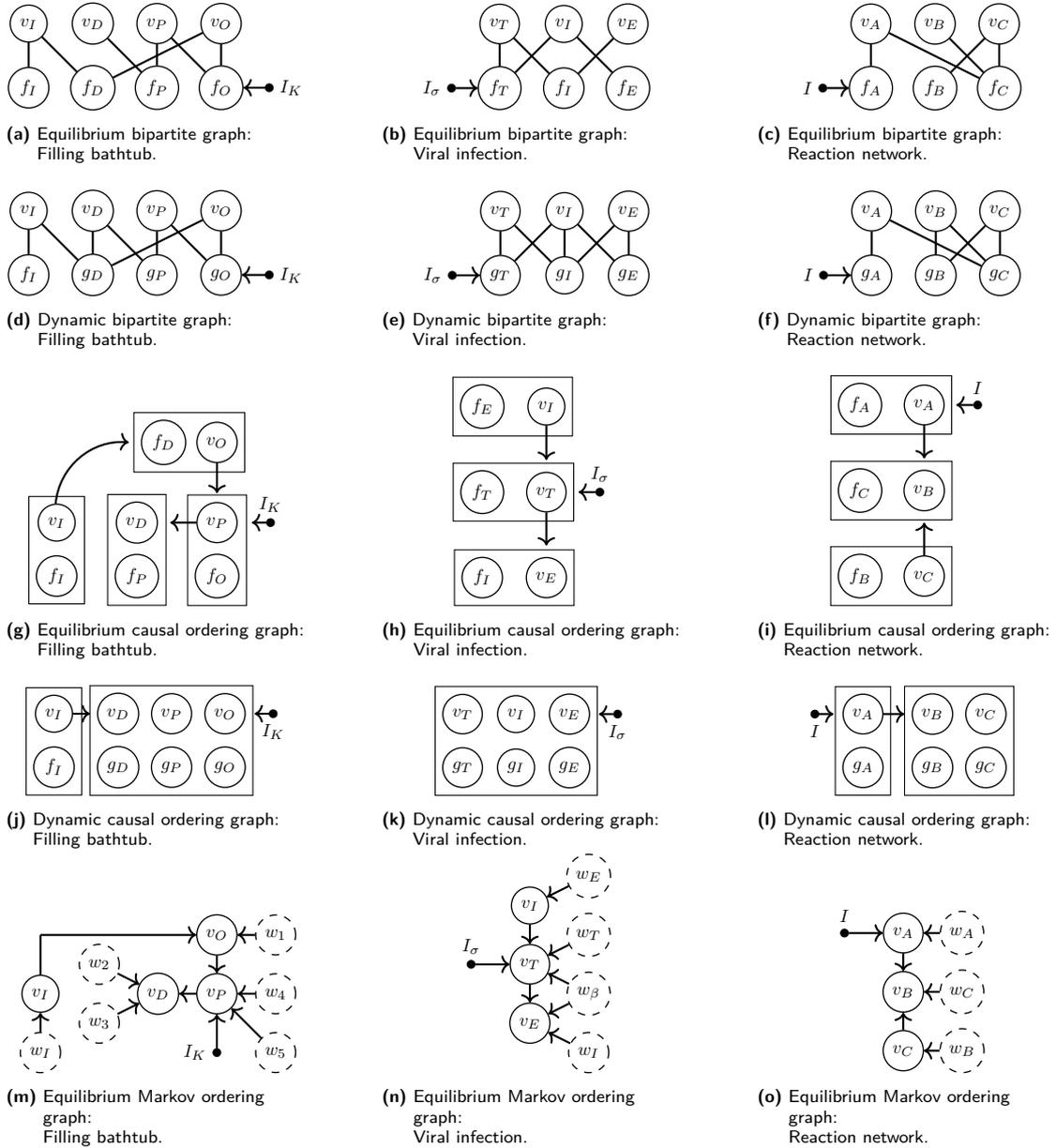
\begin{figure}[tp]\centering
	\captionsetup[subfloat]{margin=10pt,format=hang,singlelinecheck=false,justification=RaggedRight}
	\begin{subfigure}[b]{0.32\textwidth}
		\centering
		\vspace*{3mm}
		\begin{tikzpicture}[scale=0.75,every node/.style={transform shape}]
		\GraphInit[vstyle=Normal]
		\SetGraphUnit{1}
		\Vertex[L=$v_I$,x=0,y=0] {vI}
		\Vertex[L=$v_D$, x=1.2, y=0] {vD}
		\Vertex[L=$v_P$, x=2.4, y=0] {vP}
		\Vertex[L=$v_O$, x=3.6, y=0] {vO}
		\Vertex[L=$f_I$, x=0.0, y=-1.2] {fI}
		\Vertex[L=$f_D$, x=1.2, y=-1.2] {fD}
		\Vertex[L=$f_P$, x=2.4, y=-1.2] {fP}
		\Vertex[L=$f_O$, x=3.6, y=-1.2] {fO}
		
		\begin{scope}[VertexStyle/.append style = {minimum size = 4pt, 
			inner sep = 0pt,
			color=black}]
		\Vertex[x=4.5, y=-1.2, L=$I_K$, Lpos=0, LabelOut]{K}
		\end{scope}
		\draw[EdgeStyle, style={->}](K) to (fO);

		\draw[EdgeStyle, style={-}](vI) to (fI);
		\draw[EdgeStyle, style={-}](vP) to (fP);
		\draw[EdgeStyle, style={-}](vO) to (fO);
		
		\draw[EdgeStyle, style={-}](vI) to (fD);
		\draw[EdgeStyle, style={-}](vD) to (fP);
		\draw[EdgeStyle, style={-}](vP) to (fO);
		\draw[EdgeStyle, style={-}](vO) to (fD);
		\end{tikzpicture}
		\caption{Equilibrium bipartite graph:\\ Filling bathtub.}
		\label{fig:equilibrium bipartite graph:bathtub}
	\end{subfigure}%
	\begin{subfigure}[b]{0.32\textwidth}
		\centering
		\begin{tikzpicture}[scale=0.75,every node/.style={transform shape}]
		\GraphInit[vstyle=Normal]
		\SetGraphUnit{1}
		\Vertex[L=$v_T$,x=0,y=0] {vT}
		\Vertex[L=$v_I$, x=1.2, y=0] {vI}
		\Vertex[L=$v_E$, x=2.4, y=0] {vE}
		\Vertex[L=$f_T$, x=0.0, y=-1.2] {fT}
		\Vertex[L=$f_I$, x=1.2, y=-1.2] {fI}
		\Vertex[L=$f_E$, x=2.4, y=-1.2] {fE}

    \begin{scope}[VertexStyle/.append style = {minimum size = 4pt, 
			inner sep = 0pt,
			color=black}]
		\Vertex[x=-0.9, y=-1.2, L=$I_{\sigma}$, Lpos=180, LabelOut]{s}
		\end{scope}
		\draw[EdgeStyle, style={->}](s) to (fT);
		
		\draw[EdgeStyle, style={-}](vT) to (fT);
		
		\draw[EdgeStyle, style={-}](vT) to (fI);
		\draw[EdgeStyle, style={-}](vI) to (fT);
		\draw[EdgeStyle, style={-}](vI) to (fE);
		\draw[EdgeStyle, style={-}](vE) to (fI);
		\end{tikzpicture}
		\caption{Equilibrium bipartite graph:\\ Viral infection.}
		\label{fig:equilibrium bipartite graph:viral infection}
	\end{subfigure}%
	\begin{subfigure}[b]{0.32\textwidth}
		\centering
		\begin{tikzpicture}[scale=0.75,every node/.style={transform shape}]
		\GraphInit[vstyle=Normal]
		\SetGraphUnit{1}
		\Vertex[L=$v_A$,x=0,y=0] {vA}
		\Vertex[L=$v_B$, x=1.2, y=0] {vB}
		\Vertex[L=$v_C$, x=2.4, y=0] {vC}
		\Vertex[L=$f_A$, x=0.0, y=-1.2] {fA}
		\Vertex[L=$f_B$, x=1.2, y=-1.2] {fB}
		\Vertex[L=$f_C$, x=2.4, y=-1.2] {fC}

		\begin{scope}[VertexStyle/.append style = {minimum size = 4pt, 
			inner sep = 0pt,
			color=black}]
		\Vertex[x=-0.9, y=-1.2, L=$I$, Lpos=180, LabelOut]{I}
		\end{scope}
    \draw[EdgeStyle, style={->}](I) to (fA);
		
		\draw[EdgeStyle, style={-}](vA) to (fA);
		\draw[EdgeStyle, style={-}](vC) to (fC);
		
		\draw[EdgeStyle, style={-}](vA) to (fC);
		\draw[EdgeStyle, style={-}](vB) to (fC);
		\draw[EdgeStyle, style={-}](vC) to (fB);
		\end{tikzpicture}
		\caption{Equilibrium bipartite graph:\\ Reaction network.}
		\label{fig:equilibrium bipartite graph:reaction network}
	\end{subfigure}
	\begin{subfigure}[b]{0.32\textwidth}
		\vspace*{3mm}
		\centering
		\begin{tikzpicture}[scale=0.75,every node/.style={transform shape}]
		\GraphInit[vstyle=Normal]
		\SetGraphUnit{1}
		\Vertex[L=$v_I$,x=0,y=0] {vI}
		\Vertex[L=$v_D$, x=1.2, y=0] {vD}
		\Vertex[L=$v_P$, x=2.4, y=0] {vP}
		\Vertex[L=$v_O$, x=3.6, y=0] {vO}
		\Vertex[L=$f_I$, x=0.0, y=-1.2] {fI}
		\Vertex[L=$g_D$, x=1.2, y=-1.2] {fD}
		\Vertex[L=$g_P$, x=2.4, y=-1.2] {fP}
		\Vertex[L=$g_O$, x=3.6, y=-1.2] {fO}
		
    \begin{scope}[VertexStyle/.append style = {minimum size = 4pt, 
			inner sep = 0pt,
			color=black}]
		\Vertex[x=4.5, y=-1.2, L=$I_K$, Lpos=0, LabelOut]{K}
		\end{scope}
		\draw[EdgeStyle, style={->}](K) to (fO);
		
		\draw[EdgeStyle, style={-}](vI) to (fI);
		\draw[EdgeStyle, style={-}](vD) to (fD);
		\draw[EdgeStyle, style={-}](vP) to (fP);
		\draw[EdgeStyle, style={-}](vO) to (fO);
		
		\draw[EdgeStyle, style={-}](vI) to (fD);
		\draw[EdgeStyle, style={-}](vD) to (fP);
		\draw[EdgeStyle, style={-}](vP) to (fO);
		\draw[EdgeStyle, style={-}](vO) to (fD);
		\end{tikzpicture}
		\caption{Dynamic bipartite graph:\\ Filling bathtub.}
		\label{fig:dynamic bipartite graph:bathtub}
	\end{subfigure}%
	\begin{subfigure}[b]{0.32\textwidth}
		\centering
		\begin{tikzpicture}[scale=0.75,every node/.style={transform shape}]
		\GraphInit[vstyle=Normal]
		\SetGraphUnit{1}
		\Vertex[L=$v_T$,x=0,y=0] {vT}
		\Vertex[L=$v_I$, x=1.2, y=0] {vI}
		\Vertex[L=$v_E$, x=2.4, y=0] {vE}
		\Vertex[L=$g_T$, x=0.0, y=-1.2] {fT}
		\Vertex[L=$g_I$, x=1.2, y=-1.2] {fI}
		\Vertex[L=$g_E$, x=2.4, y=-1.2] {fE}

    \begin{scope}[VertexStyle/.append style = {minimum size = 4pt, 
			inner sep = 0pt,
			color=black}]
		\Vertex[x=-0.9, y=-1.2, L=$I_{\sigma}$, Lpos=180, LabelOut]{s}
		\end{scope}
		\draw[EdgeStyle, style={->}](s) to (fT);
		
		\draw[EdgeStyle, style={-}](vT) to (fT);
		\draw[EdgeStyle, style={-}](vI) to (fI);
		\draw[EdgeStyle, style={-}](vE) to (fE);
		
		\draw[EdgeStyle, style={-}](vT) to (fI);
		\draw[EdgeStyle, style={-}](vI) to (fT);
		\draw[EdgeStyle, style={-}](vI) to (fE);
		\draw[EdgeStyle, style={-}](vE) to (fI);
		\end{tikzpicture}
		\caption{Dynamic bipartite graph:\\ Viral infection.}
		\label{fig:dynamic bipartite graph:viral infection}
	\end{subfigure}%
	\begin{subfigure}[b]{0.32\textwidth}
		\centering
		\begin{tikzpicture}[scale=0.75,every node/.style={transform shape}]
		\GraphInit[vstyle=Normal]
		\SetGraphUnit{1}
		\Vertex[L=$v_A$,x=0,y=0] {vA}
		\Vertex[L=$v_B$, x=1.2, y=0] {vB}
		\Vertex[L=$v_C$, x=2.4, y=0] {vC}
		\Vertex[L=$g_A$, x=0.0, y=-1.2] {fA}
		\Vertex[L=$g_B$, x=1.2, y=-1.2] {fB}
		\Vertex[L=$g_C$, x=2.4, y=-1.2] {fC}
		
		\begin{scope}[VertexStyle/.append style = {minimum size = 4pt, 
			inner sep = 0pt,
			color=black}]
		\Vertex[x=-0.9, y=-1.2, L=$I$, Lpos=180, LabelOut]{I}
		\end{scope}
    \draw[EdgeStyle, style={->}](I) to (fA);

		\draw[EdgeStyle, style={-}](vA) to (fA);
		\draw[EdgeStyle, style={-}](vB) to (fB);
		\draw[EdgeStyle, style={-}](vC) to (fC);
		
		\draw[EdgeStyle, style={-}](vA) to (fC);
		\draw[EdgeStyle, style={-}](vB) to (fC);
		\draw[EdgeStyle, style={-}](vC) to (fB);
		\end{tikzpicture}
		\caption{Dynamic bipartite graph:\\ Reaction network.}
		\label{fig:dynamic bipartite graph:reaction network}
	\end{subfigure}
	\begin{subfigure}[b]{0.32\textwidth}
		\centering
		\vspace*{3mm}
		\begin{tikzpicture}[scale=0.75,every node/.style={transform shape}]
		\GraphInit[vstyle=Normal]
		\SetGraphUnit{1}
		
		\begin{scope}[VertexStyle/.append style = {minimum size = 4pt, 
			inner sep = 0pt,
			color=black}]
		\Vertex[x=4.0, y=0.0, L=$I_{K}$, Lpos=90, LabelOut]{K}
		\end{scope}
		
		\SetVertexMath
		\Vertex{v_I}
		\EA[unit=1.5](v_I){v_D}
		\EA[unit=1.5](v_D){v_P}
		\SO(v_I){f_I}
		\SO(v_D){f_P}
		\SO(v_P){f_O}
		\NO[unit=1.5](v_P){v_O}
		\WE(v_O){f_D}
		\node[draw=black, fit=(v_I) (f_I), inner sep=0.1cm ]{};
		\node[draw=black, fit=(v_D) (f_P), inner sep=0.1cm ]{};
		\node[draw=black, fit=(v_P) (f_O), inner sep=0.1cm ]{};
		\node[draw=black, fit=(v_O) (f_D), inner sep=0.1cm ]{};
		
		\draw[EdgeStyle, style={->}, bend left=45](v_I) to (1.35,1.5);
		\draw[EdgeStyle, style={->}](v_O) to (3,0.55);
		\draw[EdgeStyle, style={->}](v_P) to (2.125,0);
		\draw[EdgeStyle, style={->}, bend left=0](K) to (3.65,0.0);
		\end{tikzpicture}
		\caption{Equilibrium causal ordering graph:\\ Filling bathtub.}
		\label{fig:equilibrium causal ordering graph:bathtub}
	\end{subfigure}%
	\begin{subfigure}[b]{0.32\textwidth}
		\centering
		\vspace*{3mm}
		\begin{tikzpicture}[scale=0.75,every node/.style={transform shape}]
		\GraphInit[vstyle=Normal]
		\SetGraphUnit{1}
		\Vertex[L=$v_I$, x=0, y=0] {vI}
		\Vertex[L=$f_E$, x=-1.2, y=0.0] {fE}
		\Vertex[L=$v_T$, x=0, y=-1.6] {vT}
		\Vertex[L=$f_T$, x=-1.2, y=-1.6] {fT}
		\Vertex[L=$v_E$, x=0, y=-3.2] {vE}
		\Vertex[L=$f_I$, x=-1.2, y=-3.2] {fI}
		
		\begin{scope}[VertexStyle/.append style = {minimum size = 4pt, 
			inner sep = 0pt,
			color=black}]
		\Vertex[x=1.0, y=-1.6, L=$I_{\sigma}$, Lpos=90, LabelOut]{s}
		\end{scope}
		
		\node[draw=black, fit=(vT) (fT), inner sep=0.1cm ]{};
		\node[draw=black, fit=(vI) (fE), inner sep=0.1cm ]{};
		\node[draw=black, fit=(vE) (fI), inner sep=0.1cm ]{};
		
		\draw[EdgeStyle, style={->}](s) to (0.6,-1.6);
		\draw[EdgeStyle, style={->}](vI) to (0.0,-1.0);
		\draw[EdgeStyle, style={->}](vT) to (0.0,-2.6);
		\end{tikzpicture}
		\caption{Equilibrium causal ordering graph:\\ Viral infection.}
		\label{fig:equilibrium causal ordering graph:viral infection}
	\end{subfigure}%
	\begin{subfigure}[b]{0.32\textwidth}
		\centering
		\vspace*{3mm}
		\begin{tikzpicture}[scale=0.75,every node/.style={transform shape}]
		\GraphInit[vstyle=Normal]
		\SetGraphUnit{1}
		\Vertex[L=$v_A$, x=0, y=0] {vA}
		\Vertex[L=$f_A$, x=-1.2, y=0.0] {fA}
		\Vertex[L=$v_B$, x=0, y=-1.6] {vB}
		\Vertex[L=$f_C$, x=-1.2, y=-1.6] {fC}
		\Vertex[L=$v_C$, x=0, y=-3.2] {vC}
		\Vertex[L=$f_B$, x=-1.2, y=-3.2] {fB}
		
		\begin{scope}[VertexStyle/.append style = {minimum size = 4pt, 
			inner sep = 0pt,
			color=black}]
		\Vertex[x=1.0, y=0.0, L=$I$, Lpos=90, LabelOut]{I}
		\end{scope}
		
		\node[draw=black, fit=(vC) (fB), inner sep=0.1cm ]{};
		\node[draw=black, fit=(vB) (fC), inner sep=0.1cm ]{};
		\node[draw=black, fit=(vA) (fA), inner sep=0.1cm ]{};
		
		\draw[EdgeStyle, style={->}](I) to (0.6,0.0);
		\draw[EdgeStyle, style={->}](vA) to (0.0,-1.0);
		\draw[EdgeStyle, style={->}](vC) to (0.0,-2.2);
		\end{tikzpicture}
		\caption{Equilibrium causal ordering graph:\\ Reaction network.}
		\label{fig:equilibrium causal ordering graph:NFBLB}
	\end{subfigure}
	\begin{subfigure}[b]{0.32\textwidth}
		\vspace*{3mm}
		\centering
		\begin{tikzpicture}[scale=0.75,every node/.style={transform shape}]
		\GraphInit[vstyle=Normal]
		\Vertex[L=$v_I$,x=0,y=0] {vI}
		\Vertex[L=$v_D$, x=1.2, y=0] {vD}
		\Vertex[L=$v_P$, x=2.2, y=0] {vP}
		\Vertex[L=$v_O$, x=3.2, y=0] {vO}
		\Vertex[L=$f_I$, x=0.0, y=-1] {fI}
		\Vertex[L=$g_D$, x=1.2, y=-1] {fD}
		\Vertex[L=$g_P$, x=2.2, y=-1] {fP}
		\Vertex[L=$g_O$, x=3.2, y=-1] {fO}
		
		\node[draw=black, fit=(vI) (fI), inner sep=0.1cm ]{};
		\node[draw=black, fit=(vD) (vP) (vO) (fD) (fP) (fO), inner sep=0.1cm ]{};
		
		\begin{scope}[VertexStyle/.append style = {minimum size = 4pt, 
			inner sep = 0pt,
			color=black}]
		\Vertex[x=4.1, y=0.0, L=$I_K$, Lpos=270, LabelOut]{I}
		\end{scope}
		
		\draw[EdgeStyle, style={->}](vI) to (0.675,0.0);
		\draw[EdgeStyle, style={->}](I) to (3.75,0.0);
		\end{tikzpicture}
		\caption{Dynamic causal ordering graph:\\ Filling bathtub.}
		\label{fig:dynamic causal ordering graph:bathtub}
	\end{subfigure}%
	\begin{subfigure}[b]{0.32\textwidth}
		\centering
		\begin{tikzpicture}[scale=0.75,every node/.style={transform shape}]
		\GraphInit[vstyle=Normal]
		\SetGraphUnit{1}
		\Vertex[L=$v_T$,x=0,y=0] {vT}
		\Vertex[L=$v_I$, x=1.0, y=0] {vI}
		\Vertex[L=$v_E$, x=2.0, y=0] {vE}
		\Vertex[L=$g_T$, x=0.0, y=-1.0] {fT}
		\Vertex[L=$g_I$, x=1.0, y=-1.0] {fI}
		\Vertex[L=$g_E$, x=2.0, y=-1.0] {fE}
		
		\node[draw=black, fit=(vT) (vI) (vE) (fT) (fI) (fE), inner sep=0.1cm ]{};
		
		\begin{scope}[VertexStyle/.append style = {minimum size = 4pt, 
			inner sep = 0pt,
			color=black}]
		\Vertex[x=2.9, y=0.0, L=$I_{\sigma}$, Lpos=270, LabelOut]{I}
		\end{scope}
		
		\draw[EdgeStyle, style={->}](I) to (2.55,0.0);
		\end{tikzpicture}
		\caption{Dynamic causal ordering graph:\\ Viral infection.}
		\label{fig:dynamic causal ordering graph:viral infection}
	\end{subfigure}%
	\begin{subfigure}[b]{0.32\textwidth}
		\centering
		\begin{tikzpicture}[scale=0.75,every node/.style={transform shape}]
		\GraphInit[vstyle=Normal]
		\SetGraphUnit{1}
		\Vertex[L=$v_A$,x=0,y=0] {vA}
		\Vertex[L=$v_B$, x=1.3, y=0] {vB}
		\Vertex[L=$v_C$, x=2.3, y=0] {vC}
		\Vertex[L=$g_A$, x=0.0, y=-1.0] {fA}
		\Vertex[L=$g_B$, x=1.3, y=-1.0] {fB}
		\Vertex[L=$g_C$, x=2.3, y=-1.0] {fC}
		
		\node[draw=black, fit=(vA) (fA), inner sep=0.1cm ]{};
		\node[draw=black, fit=(vB) (vC) (fB) (fC), inner sep=0.1cm ]{};
		
		\begin{scope}[VertexStyle/.append style = {minimum size = 4pt, 
			inner sep = 0pt,
			color=black}]
		\Vertex[x=-0.9, y=0.0, L=$I$, Lpos=270, LabelOut]{I}
		\end{scope}
		
		\draw[EdgeStyle, style={->}](I) to (-0.55,0.0);
		\draw[EdgeStyle, style={->}](vA) to (0.775,0.0);
		\end{tikzpicture}
		\caption{Dynamic causal ordering graph:\\ Reaction network.}
		\label{fig:dynamic causal ordering graph:reaction network}
	\end{subfigure}
	\begin{subfigure}[b]{0.33\textwidth}
		\vspace*{3mm}
		\centering
		\begin{tikzpicture}[scale=0.75,every node/.style={transform shape}]
		\GraphInit[vstyle=Normal]
		\SetGraphUnit{1}
		\Vertex[L=$w_{I}$, style={dashed}, x=0.0, y=-1.1] {wI}
		\Vertex[L=$w_{2}$, style={dashed}, x=1.1, y=0.55] {w1}
		\Vertex[L=$w_{3}$, style={dashed}, x=1.1, y=-0.55] {w2}
		\Vertex[L=$w_{4}$, style={dashed}, x=4.4, y=0.0] {w3}
		\Vertex[L=$w_{5}$, style={dashed}, x=4.4, y=-1.1] {w4}
		\Vertex[L=$w_{1}$, style={dashed}, x=4.4, y=1.1] {w5}
		\begin{scope}[VertexStyle/.append style = {minimum size = 4pt, 
			inner sep = 0pt,
			color=black}]
		\Vertex[x=3.3, y=-1.1, L=$I_K$, Lpos=180, LabelOut]{I_K}
		\end{scope}		
		
		\SetVertexMath
		\Vertex{v_I}
		\EA[unit=2.2](v_I){v_D}
		\EA[unit=1.1](v_D){v_P}
		\NO[unit=1.1](v_P){v_O}
		
		\draw[EdgeStyle, style={-}](v_I) to (0,1.1);
		\draw[EdgeStyle, style={->}](0,1.1) to (v_O);
		\draw[EdgeStyle, style={->}](v_O) to (v_P);
		\draw[EdgeStyle, style={->}](v_P) to (v_D);
		\draw[EdgeStyle, style={->}](I_K) to (v_P);
		\draw[EdgeStyle, style={->}](wI) to (v_I);
		\draw[EdgeStyle, style={->}](w1) to (v_D);
		\draw[EdgeStyle, style={->}](w2) to (v_D);
		\draw[EdgeStyle, style={->}](w3) to (v_P);
		\draw[EdgeStyle, style={->}](w4) to (v_P);
		\draw[EdgeStyle, style={->}](w5) to (v_O);
		\end{tikzpicture}
		\caption{Equilibrium Markov ordering graph:\\ Filling bathtub.}
		\label{fig:markov ordering graph:bathtub}
	\end{subfigure}%
	\begin{subfigure}[b]{0.33\textwidth}
		\centering
		\begin{tikzpicture}[scale=0.75,every node/.style={transform shape}]
		\GraphInit[vstyle=Normal]
		\SetGraphUnit{1}
		\Vertex[L=$v_I$, x=1.1, y=1.1] {vI}
		\Vertex[L=$v_T$, x=1.1, y=0] {vT}
		\Vertex[L=$v_E$, x=1.1, y=-1.1] {vE}
		\Vertex[L=$w_{E}$, x=2.2, y=1.65, style={dashed}] {wE}
		\Vertex[L=$w_{T}$, x=2.2, y=0.55, style={dashed}] {wT}
		\Vertex[L=$w_{\beta}$, x=2.2, y=-0.55, style={dashed}] {wb}
		\Vertex[L=$w_{I}$, x=2.2, y=-1.65, style={dashed}] {wI}
		\begin{scope}[VertexStyle/.append style = {minimum size = 4pt, 
			inner sep = 0pt,
			color=black}]
		\Vertex[L=$I_{\sigma}$, x=0, y=0, Lpos=90, LabelOut]{Is}
		\end{scope}			
		\draw[EdgeStyle, style={->}](Is) to (vT);
		\draw[EdgeStyle, style={->}](vI) to (vT);
		\draw[EdgeStyle, style={->}](vT) to (vE);
		\draw[EdgeStyle, style={->}](wE) to (vI);
		\draw[EdgeStyle, style={->}](wT) to (vT);
		\draw[EdgeStyle, style={->}](wb) to (vT);
		\draw[EdgeStyle, style={->}](wb) to (vE);
		\draw[EdgeStyle, style={->}](wI) to (vE);
		\end{tikzpicture}
		\caption{Equilibrium Markov ordering graph:\\ Viral infection.}
		\label{fig:markov ordering graph:viral infection}
	\end{subfigure}%
	\begin{subfigure}[b]{0.33\textwidth}
		\centering
		\begin{tikzpicture}[scale=0.75,every node/.style={transform shape}]
		\GraphInit[vstyle=Normal]
		\SetGraphUnit{1}
		\Vertex[L=$v_A$, x=1.1, y=1.1] {vA}
		\Vertex[L=$v_B$, x=1.1, y=0] {vB}
		\Vertex[L=$v_C$, x=1.1, y=-1.1] {vC}
		\Vertex[L=$w_{A}$, x=2.2, y=1.1, style={dashed}] {wA}
		\Vertex[L=$w_{C}$, x=2.2, y=0.0, style={dashed}] {wC}
		\Vertex[L=$w_{B}$, x=2.2, y=-1.1, style={dashed}] {wB}
		\begin{scope}[VertexStyle/.append style = {minimum size = 4pt, 
			inner sep = 0pt,
			color=black}]
		\Vertex[L=$I$, x=0, y=1.1, Lpos=90, LabelOut]{I}
		\end{scope}	
		\draw[EdgeStyle, style={->}](I) to (vA);
		\draw[EdgeStyle, style={->}](vA) to (vB);
		\draw[EdgeStyle, style={->}](vC) to (vB);
		\draw[EdgeStyle, style={->}](wA) to (vA);
		\draw[EdgeStyle, style={->}](wC) to (vB);
		\draw[EdgeStyle, style={->}](wB) to (vC);
		\end{tikzpicture}
		\caption{Equilibrium Markov ordering graph:\\ Reaction network.}
		\label{fig:markov ordering graph:NFBLB}
	\end{subfigure}
  \caption{Graphical representations of the bathtub model (left column), the viral infection model (center column), and the reaction network with negative feedback (right column). The input vertices $I_K$, $I_{\sigma}$, and $I$ are represented by black dots while exogenous variables are indicated by dashed circles. For each model, the structure of the equilibrium equations can be read off from the equilibrium bipartite graphs in Figures \ref{fig:equilibrium bipartite graph:bathtub}, \ref{fig:equilibrium bipartite graph:viral infection}, and \ref{fig:equilibrium bipartite graph:reaction network}. The structure of the first-order differential equations can be represented with the dynamic bipartite graphs in Figures \ref{fig:dynamic bipartite graph:bathtub}, \ref{fig:dynamic bipartite graph:viral infection}, and \ref{fig:dynamic bipartite graph:reaction network}. The equilibrium causal ordering graphs corresponding to the equilibrium bipartite graphs are given in Figures \ref{fig:equilibrium causal ordering graph:bathtub}, \ref{fig:equilibrium causal ordering graph:viral infection}, and \ref{fig:equilibrium causal ordering graph:NFBLB}. Similarly, the dynamic causal ordering graphs corresponding to the dynamic bipartite graphs can be found in Figures \ref{fig:dynamic causal ordering graph:bathtub}, \ref{fig:dynamic causal ordering graph:viral infection}, and \ref{fig:dynamic causal ordering graph:reaction network}. The equilibrium Markov ordering graphs for the equilibrium distribution of the models are given in Figures \ref{fig:markov ordering graph:bathtub}, \ref{fig:markov ordering graph:viral infection}, and \ref{fig:markov ordering graph:NFBLB}.}
	\label{fig:graphical representations}
\end{figure}

\noindent
\textit{Causal ordering graph}: Application of the causal ordering algorithm to the equilibrium bipartite graphs of the filling bathtub, the viral infection, and the reaction network results in the \emph{equilibrium causal ordering graphs} in Figures \ref{fig:equilibrium causal ordering graph:bathtub}, \ref{fig:equilibrium causal ordering graph:viral infection}, and \ref{fig:equilibrium causal ordering graph:NFBLB}, respectively. Henceforth, we will assume that the dynamic bipartite graph has a perfect matching that extends the natural labelling of the dynamic equations, i.e.,\ such that all pairs $(v_i, g_i)$ are matched. Application of the causal ordering algorithm to the associated dynamic bipartite graph for the model of a filling bathtub, the viral infection model, and the reaction network results in the \emph{dynamic causal ordering graphs} in Figures \ref{fig:dynamic causal ordering graph:bathtub}, \ref{fig:dynamic causal ordering graph:viral infection}, and \ref{fig:dynamic causal ordering graph:reaction network}, respectively.\footnote{Our approach here differs from the dynamic causal ordering algorithm proposed in \citep{Iwasaki1994}, which includes separate vertices for derivatives and variables that are linked by `definitional' integration links.
} 

As shown in \citep{Blom2020}, the absence (presence) of a directed path from an equation vertex to a variable vertex in the equilibrium causal ordering graph indicates that a soft intervention targeting a parameter in that equation has no (a generic) effect on the value of that variable once the system has reached equilibrium again. Furthermore, 
the absence (presence) of a directed path from a cluster to a variable vertex in the equilibrium causal ordering graph indicates that a perfect intervention targeting the cluster has no (a generic) effect on the value of that variable once the system has reached equilibrium again.
Notice that the variables $v_i$ in the equilibrium causal ordering graph do not always end up in the same cluster with the equilibrium equation $f_i$ of the natural labelling. For example, we see in Figure~\ref{fig:equilibrium causal ordering graph:bathtub} that a soft intervention targeting the equilibrium equation $f_O$ (e.g.\ a change in the value of $U_5$) does \emph{not} affect the value of the outflow rate $X_O$ at equilibrium (since there is no directed path from $f_O$ to $v_O$), even though $f_O$ was obtained from the dynamic equation for the time derivative of the outflow rate $X_O(t)$. Similarly, targeting $f_I$ with a soft intervention in the viral infection model has no effect on $X_I$ at equilibrium and targeting $f_C$ in the reaction network model has no effect on the equilibrium distribution of $X_C$.\footnote{To preserve an unambiguous causal interpretation, equations and clusters that may be targeted by interventions should be clearly distinguished from the variables that could be affected by those interventions \citep{Blom2020}.} This suggests that the causal structure at equilibrium of perfectly adapted dynamical systems may differ from the transient causal structure. In the next section, we will use this idea to detect perfect adaptation from background knowledge and experimental data.\\

\noindent
\textit{Equilibrium Markov ordering graph}: As explained in Section~\ref{sec:background:causal ordering}, the Markov ordering graph is constructed from the causal ordering graph and includes exogenous variables. For the bathtub model, we let vertices $w_I,w_1,\ldots,w_5$ represent the independent exogenous random variables $U_I,U_1,\ldots,U_5$ that appear in the model. For the viral infection model we let $w_T, w_I$, $w_E$, $w_\beta$ represent independent exogenous random variables $d_T, d_I, d_E$, and $\beta$ in equations \eqref{eq:simple T}, \eqref{eq:simple I}, and \eqref{eq:simple E}. Finally, for the reaction network with negative feedback, we let $w_A$, $w_B$, and $w_C$ represent the independent exogenous random variables $k_{IA}$, $k_{CB}$ and $k_{AC}$, respectively. The equilibrium Markov ordering graphs for the filling bathtub model, the viral infection model, and the model of a reaction network with a negative feedback loop are given in Figures \ref{fig:markov ordering graph:bathtub}, \ref{fig:markov ordering graph:viral infection}, and \ref{fig:markov ordering graph:NFBLB} respectively.\footnote{The equilibrium Markov ordering graph for the bathtub model coincides with the result of \citep{Dash2005}, who simulated data from the bathtub model until the system reached equilibrium and then applied the PC algorithm to the equilibrium data. Although Dash \cite{Dash2005} interprets the learned graphical representation as the `causal graph', this graph in itself does not have a straightforward causal interpretation. See Appendix~\ref{app:causal interpretation mog} and the discussion in \citep{Blom2020} for more details.} These equilibrium Markov ordering graphs can be used to read off conditional independences in the equilibrium distribution that are implied by the equilibrium equations of the model. For example, since $v_I$ is $d$-separated from $v_D$ given $v_P$ in the equilibrium Markov ordering graph in Figure~\ref{fig:markov ordering graph:bathtub}, $X_I$ will be independent of $X_D$ given $X_P$ once the system has reached equilibrium. These independences can be tested for in equilibrium data by means of statistical conditional independence tests. These implied conditional independences can for instance be used in the process of model selection \citep{BlomMooij_UAI_22}.

\subsection{Existence and uniqueness of solutions}\label{sec:perfect adaptation:solvability properties}

The causal ordering algorithm is a graphical tool that can be useful when solving a system of equations. 
It decomposes the question of existence and uniqueness of a `global' solution into several `local' existence and uniqueness problems corresponding to a partitioning of the equations.
When a unique global solution exists for all possible joint values of the (independent) exogenous variables, this leads to both a causal semantics and to a Markov property \citep{Blom2020}.
We argue here that these ideas can also be extended to include differential equations.
We will illustrate this with the filling bathtub model.
We start with the (conceptually simpler) equilibrium model, which solely contains static equations, before discussing what to do when dynamic equations are present.

The equilibrium equations \eqref{eq:bathtub:eq I}--\eqref{eq:bathtub:eq O} can be solved in steps by following the 
topological ordering of the clusters in the equilibrium causal ordering graph in Figure~\ref{fig:equilibrium causal ordering graph:bathtub}. First, use $f_I$ to solve for $X_I$, resulting in $X_I = U_I$. Then, use $f_D$ to solve for $X_O$, which results in $X_O = X_{I}$. Subsequently, use $f_O$ to solve for $X_P$, yielding $X_P = \frac{X_{O}}{U_5 I_K}$. Finally, use $f_P$ to solve for $X_D$, resulting in $X_D = \frac{X_P}{g U_3}$. By substitution, we obtain a global solution of the form
$$(X_I,X_O,X_P,X_D) = \left(U_I,U_I,\frac{U_I}{U_5 I_K},\frac{U_I}{g U_3 U_5 I_K}\right).$$
Since we obtain a unique solution of each equation for the target variable in terms of the other variables appearing in the equation, this procedure shows that there exists a unique global solution of the system of equations for any value of the exogenous variables $U_I,U_1,U_2,U_3,U_4,U_5$ and any value of the input signal $I_K$.
Because of this, we obtain both a causal interpretation and a Markov property for the filling bathtub model at equilibrium as described in Section~\ref{sec:perfect adaptation:graphical representations}.

For the dynamic filling bathtub model, we can follow a similar procedure, but now the clusters may also contain differential equations. We can make use of the theory for the existence and uniqueness of solutions of ordinary differential equations (ODEs).
First note that the dynamic bipartite graph reflects the structure of the static and dynamic equations, once we rewrite the differential equations as integral equations. For example, for the time interval $[t_0,t]$:
\begin{align}
X_{I}(t) &= U_I,\\
  X_D(t) &= X_D(t_0) + \int_{t_0}^t U_1(X_{I}(\tau)-X_{O}(\tau)) \,d\tau,\\
  X_P(t) &= X_P(t_0) + \int_{t_0}^t U_2(g\,U_3X_D(\tau)-X_P(\tau)) \,d\tau,\\
  X_O(t) &= X_O(t_0) + \int_{t_0}^t U_4(U_5 I_K(\tau) X_P(\tau)-X_{O}(\tau)) \,d\tau.
\end{align}
Rewriting the differential equations as integral equations has two advantages:
(i) there is no need to introduce the derivatives as if they were (variation) independent processes; (ii) 
it makes the dependence on the initial conditions $X_D(t_0)$, $X_P(t_0)$ and $X_O(t_0)$ explicit.
The equations \eqref{eq:bathtub:dyn I}--\eqref{eq:bathtub:dyn O} describing the dynamical system can be solved in steps by following the 
topological ordering of the clusters in the dynamic causal ordering graph in Figure~\ref{fig:dynamic causal ordering graph:bathtub}.
First, solve $f_I$ for $X_I$, resulting in $X_I(t) = U_I$.
The cluster $\{g_D,g_P,g_O,v_D,v_P,v_O\}$ has to be dealt with as a single unit, which means we have to solve the subsystem of three differential equations $\{g_D,g_P,g_O\}$ (that is, equations \eqref{eq:bathtub:dyn D}--\eqref{eq:bathtub:dyn O}) for its solution with components $(X_D(t),X_P(t),X_O(t))$. 
By applying the Picard-Lindel\"of theorem \citep[see, e.g.,][]{CoddingtonLevinson1955}, one obtains that this subsystem has a unique solution on a time interval $[t_0,\infty)$ for any initial condition $(X_D(t_0),X_P(t_0),X_O(t_0))$, provided that $X_I(t)$ and $I_K(t)$ are continuous and that the input signal $I_K(t)$ is bounded.
Thus, the equations \eqref{eq:bathtub:dyn I}--\eqref{eq:bathtub:dyn O} have a unique global solution for any value of the exogenous variables $U_I,U_1,U_2,U_3,U_4,U_5$, any initial condition $(X_D(t_0),X_P(t_0),X_O(t_0))$, and any continuous and bounded input signal $I_K(t)$.
The approach of \citep{Blom2020} can in this way be extended to yield both a dynamic causal interpretation and a Markov property (by using the trick of \citep{Bongers2018} to interpret path-continuous stochastic processes as random variables).\footnote{A  more formal and rigorous treatment is left as future work.}

Important to note here is that this explicit solution procedure shows that at equilibrium, the value of the input signal $I_K$ may affect the value of $X_D$ and $X_P$, but cannot affect the values of $X_O$ and $X_I$, while there can be transient effects of the input signal $I_K(t)$ on $X_D(t), X_P(t)$ and $X_O(t)$, but not on $X_I(t)$. Furthermore, under appropriate local solvability conditions for each cluster, these observations can directly be read off from the (equilibrium and dynamic) causal ordering graphs.

\section{Identification of perfect adaptation}
\label{sec:perfect adaptation:identification}

In Section~\ref{sec:perfect adaptation:examples}, we identified perfect adaptation in three simple models through simulations. Here, we consider how to identify models that are capable of perfect adaptation without requiring simulations or explicit calculations. In Section~\ref{sec:perfect adaptation:identification:identification co} we will put the graphical representations of Section~\ref{sec:perfect adaptation:graphical representations} to use for identifying perfect adaptation in dynamical models.
We discuss possibilities for the identification of perfect adaptation from equilibrium data in Section~\ref{sec:perfect adaptation:data}.

\subsection{Identification of perfect adaptation via causal ordering}
\label{sec:perfect adaptation:identification:identification co}

The identification of perfect adaptation via causal ordering makes use of the causal semantics of the equilibrium causal ordering graph. 
The following lemma states that a change in the input signal has no effect on the value of a variable if there is no directed path from the input vertex to that variable in the equilibrium causal ordering graph.

\begin{lemma}
	\label{lemma:equilibrium response}
	Consider a model consisting of static equations, a set of first-order differential equations in canonical form, and an input signal. Assume that the equilibrium bipartite graph has a perfect matching and that the static equations and equilibrium equations derived from the first-order differential equations are uniquely solvable w.r.t.\ the equilibrium causal ordering graph for all relevant values of the input signal. If there is no directed path from the input vertex to a variable vertex in the equilibrium causal ordering graph then the value of the input signal does not influence the equilibrium distribution of that variable.
\end{lemma}
\begin{proof}
  The statement follows directly from Theorem~20 in \cite{Blom2020}.
\end{proof}

To establish perfect adaptation, we assume that the presence of a directed path in the dynamic causal ordering graph implies the presence of a transient causal effect.
\begin{assumption}
	\label{ass:transient response}
  Consider a model consisting of static equations, a set of first-order differential equations in canonical form, and an input signal.
  Assume that the dynamic bipartite graph has a perfect matching that extends the natural labelling.
	If there is a directed path from the input vertex to a variable vertex in the dynamic causal ordering graph, then there will be a response of that variable to changes in the input signal some time later.
\end{assumption}
Intuitively, this assumption may seem plausible, as the presence of the directed path in the dynamic causal ordering graph implies that the input signal enters into the construction of the solution of the variable, as sketched in Section~\ref{sec:perfect adaptation:solvability properties}. Unless a perfect cancellation occurs, one then expects a generic effect on the solution some time after the change in the input signal. Assumption~\ref{ass:transient response} can be seen as a consequence of a certain faithfulness assumption.\footnote{Indeed, it appears that Assumption~\ref{ass:transient response} follows from the faithfulness assumption that corresponds with a Markov property that was derived for Structural Dynamical Causal Models \citep{Bongers2018}.}
We conjecture that this assumption is generically satisfied for a large class of dynamical systems (for example, it might hold for almost all parameter values w.r.t.\ the Lebesgue measure on a suitable parameter space).\footnote{Proving this in sufficient generality seems not straightforward; indeed, even the well-known result that $d$-faithfulness is a generic property has only been shown so far for Bayesian networks with discrete variables and for linear-Gaussian structural equation models \citep{Meek1995}.}

By combining Lemma~\ref{lemma:equilibrium response} and Assumption~\ref{ass:transient response}, we immediately obtain the following result.
\begin{theorem}
	\label{thm:identification of perfect adaptation}
  Consider a model that satisfies the conditions of Lemma~\ref{lemma:equilibrium response} and assume that the associated dynamic bipartite graph has a perfect matching that extends the natural labelling. Under Assumption~\ref{ass:transient response}, the presence of a directed path from the input signal $I$ to a variable $X_v$ in the dynamic causal ordering graph and the absence of such a path in the equilibrium causal ordering graph implies that $X_v$ perfectly adapts to changes in the input signal $I$ if the system equilibrates.
\end{theorem}

Theorem \ref{thm:identification of perfect adaptation} can be directly applied to the equilibrium and dynamic causal ordering graphs in Figure~\ref{fig:graphical representations} to identify perfect adaptation. For example, we see that there is a directed path from the input signal $I_K$ to $v_O$ in the dynamic causal ordering graph of the filling bathtub in Figure~\ref{fig:dynamic causal ordering graph:bathtub}, while no such path exists in the equilibrium causal ordering graph in Figure~\ref{fig:equilibrium causal ordering graph:bathtub}. It follows from Theorem~\ref{thm:identification of perfect adaptation} that $X_O$ perfectly adapts to changes in the input signal $I_K$. This is in agreement with the simulation in Figure~\ref{fig:adaptation:bathtub}. Similarly, we can verify that the amount of infected cells $X_I$ in the viral infection model perfectly adapts to changes in the input signal $I_\sigma$ and that $X_C$ perfectly adapts to $I$ in the reaction network with negative feedback. Hence, perfect adaptation in the bathtub model, the viral infection model, and the reaction network with negative feedback can be identified by applying the graphical criteria in Theorem~\ref{thm:identification of perfect adaptation} to the respective causal ordering graphs. It is important to keep in mind, though, that what we can identify in this way is only the \emph{possibility} of perfectly adaptive behavior, relying on the implicit assumption that the system will actually equilibrate.

In Appendix~\ref{app:rewriting equations} we show that the sufficient conditions in Theorem \ref{thm:identification of perfect adaptation} for the identification of perfect adaptation are not necessary. More specifically, we construct graphical representations for a dynamical model of a basic enzymatic reaction that achieves perfect adaptation but does not satisfy the conditions in Theorem~\ref{thm:identification of perfect adaptation}. 
Interestingly, though, after rewriting the equations the perfectly adaptive behavior of these systems can be captured via Theorem~\ref{thm:identification of perfect adaptation}.
Further, in Appendix~\ref{app:ifflp network} we show that the biochemical reaction network in Figure~\ref{fig:IFFLP}, which Ma \emph{et al.} \cite{Ma2009} identified as being capable of achieving perfect adaptation, does not satisfy the conditions in Theorem~\ref{thm:identification of perfect adaptation} either. 
We show that a change of variables enables one to still capture the perfectly adaptive behavior of this system via Theorem~\ref{thm:identification of perfect adaptation}.

\subsection{Identification of perfect adaptation from data}
\label{sec:perfect adaptation:data}

So far we have only considered how perfect adaptation can be identified in mathematical models. In this section we focus on methods for identifying perfect adaptation from data that is generated by perfectly adapted dynamical systems under experimental conditions. The most straightforward approach to detect perfect adaptation is to collect time-series data while experimentally changing the input signal to the system. One can then simply observe whether the variables in the system revert to their original values. However, this type of experimentation is not always feasible. Another way to identify feedback loops that achieve perfect adaptation uses a combination of observational equilibrium data, background knowledge, and experimental data. Our second main result, Theorem~\ref{thm:detect perfect adaptation}, gives sufficient conditions under which we can identify a system that is capable of perfect adaptation from experimental equilibrium data.

\begin{restatable}{theorem}{detectpa}
\label{thm:detect perfect adaptation}
  Consider a set of first-order dynamical equations in canonical form for variables $V$, satisfying the conditions of Theorem~\ref{thm:identification of perfect adaptation}, with equilibrium equations $F$ under the natural labelling. Consider a soft intervention targeting an equation $f_i\in F$. Assume that the system is uniquely solvable w.r.t. the equilibrium causal ordering graph both before and after the intervention and that the intervention alters the equilibrium distribution of all descendants of $f_i$ in the equilibrium causal ordering graph. If either
\begin{enumerate}
\item the soft intervention does not change the equilibrium distribution of $X_i$, or \label{lemma:observe:c1}
\item the soft intervention alters the equilibrium distribution of a variable corresponding to a non-descendant of $v_i$ in the equilibrium Markov ordering graph, \label{lemma:observe:c2}
\end{enumerate}
(or both), then the system is capable of perfect adaptation.
\end{restatable}
\begin{proof}
	The proof is given in Appendix~\ref{app:proof of theorem}.
\end{proof}

Theorem~\ref{thm:detect perfect adaptation} applies in particular to experiments on the filling bathtub, viral infection, and chemical reaction systems (for the corresponding graphs, see Figure \ref{fig:graphical representations}). For example, a soft intervention targeting $f_O$ in the bathtub example has no effect on the outflow rate at equilibrium $X_O$, because there is no directed path from $f_O$ to $v_O$ in Figure~\ref{fig:equilibrium causal ordering graph:bathtub}, and an intervention targeting $f_C$ has no effect on the equilibrium concentration $X_C$ in the reaction network because there is no directed path from $f_C$ to $v_C$ in Figure \ref{fig:equilibrium causal ordering graph:NFBLB}. In both cases the first condition of Theorem~\ref{thm:detect perfect adaptation} is satisfied. For the viral infection model, we see that a soft intervention targeting $f_E$ has an effect on the amount of infected cells $X_I$ at equilibrium (since there is a directed path from $f_E$ to $v_I$ in Figure~\ref{fig:equilibrium causal ordering graph:viral infection}), while there is no directed path from $v_E$ to $v_I$ in Figure~\ref{fig:markov ordering graph:viral infection}. In this case the second condition of Theorem \ref{thm:detect perfect adaptation} is satisfied.

We can devise the following scheme to detect perfectly adapted dynamical systems from data and background knowledge. The procedure relies on several assumptions, including $d$-faithfulness of the equilibrium distribution to the equilibrium Markov ordering graph. We start by collecting observational equilibrium data and use a causal discovery algorithm (such as LCD or FCI) to learn a (partial) representation of the equilibrium Markov ordering graph, assuming the observational distribution at equilibrium to be $d$-faithful w.r.t.\ the equilibrium Markov ordering graph. We then consider a soft intervention that changes a known equation in the first-order differential equation model (i.e.\ it targets a known equilibrium equation). If this intervention does not change the distribution of the variable corresponding to this target using the natural labelling, or if it changes the distribution of identifiable non-descendants of the variable corresponding to the target according to the learned Markov equivalence class, we can apply Theorem~\ref{thm:detect perfect adaptation} to identify the dynamical system as being capable of perfect adaptation. This way, we could identify perfect adaptation in specific cases such as the filling bathtub, viral infection, and reaction network by exploiting a combination of background knowledge and experimental data.
Another example of a possible application of Theorem~\ref{thm:detect perfect adaptation} is given in Section~\ref{sec:application:MEK inhibition}.

\section{Perfect adaptation in protein signalling}
\label{sec:application}

In this section we apply the ideas developed in the previous sections to a biological system that has been intensely studied
during the past decades to emphasize the practical relevance of perfect adaptation. 
The so-called RAS-RAF-MEK-ERK signaling cascade is a text-book example of a \emph{protein signalling network}, which forms an important ingredient of the `machinery' of cells in living organisms.
The molecular pathways in such a network fulfill various important functions, for instance the transmission and processing of information.
Systems biologists make use of dynamical systems to model such networks both qualitatively and quantitatively. 
Because of the high complexity of protein signalling networks, which typically consist of many different interacting components, this has also been considered a promising application domain for causal discovery methods.

In an influential paper, Sachs \emph{et al.} \cite{Sachs2005} applied causal discovery to reconstruct a protein signaling network from experimental data. 
Over the years, the dataset of \citep{Sachs2005} has become an often used `benchmark' for assessing the accuracy of causal discovery algorithms, where the `consensus network' in \citep{Sachs2005} is usually considered as the perfect ground truth.
The apparent successes of causal discovery on this particular dataset may have led to the impression that causal discovery algorithms can in general successfully discover the causal semantics of complex protein signaling networks from real-world data. 
However, this success has hitherto not been repeated on other, similar datasets, to the best of our knowledge.
Indeed, modeling and understanding such systems and inferring their behavior and structure from data still poses many challenges, for instance because of feedback loops and the inherent dynamical nature of such systems \citep{Sachs2013}. 

In this section, we focus on understanding the properties of the equilibrium distribution of a simple model of the RAS-RAF-MEK-ERK signalling pathway, and specifically investigate the phenomenon of perfect adaptation. 
Like many other biological systems, protein signalling networks often show adaptive behavior which helps to ensure a certain robustness of their functionality against various disturbances and perturbations \citep{Ferrell2016}.
Using the technique of causal ordering to analyze the conditional independences and causal relations that are implied by the model at equilibrium, we elucidate the causal interpretation of the output of constraint-based causal discovery algorithms when they are applied to equilibrium protein expression data if the parameters are such that the system shows perfect adaptation.

We test some of the model's predictions on real-world data and compare with another model that has been proposed.
We do not claim that the perfectly adaptive model that we analyze here is a realistic model of the protein signalling pathway.
Although we will show in Section~\ref{sec:application:real-world protein signalling} that the model is able to explain certain observations in real-world data, this is not that surprising for a model with that many parameters.\footnote{As John von Neumann once put it: ``With four parameters I can fit an elephant, and with five I can make him wiggle his trunk''.} 
Instead, our goal is to demonstrate that in systems with perfect adaptation the standard interpretation of the output of causal discovery algorithms may not apply.\footnote{This was already pointed out in \citep{Dash2005} for the example of the filling bathtub, but our work shows how widespread this phenomenon may be, and thereby emphasizes its practical relevance for causal discovery.}
This could explain why the output of certain causal discovery algorithms applied to the data of \citep{Sachs2005} appears to be at odds with the biological consensus network presented in \citep{Sachs2005}, see for example \citep{Ramsey2018} and \citep{Mooij2020}.

This section is structured as follows. 
In Section~\ref{sec:application:dynamical model} we introduce the perfectly adaptive model for the signalling pathway. 
We proceed with the associated graphical representations in Section~\ref{sec:application:causal ordering}. 
Then, in Section~\ref{sec:application:MEK inhibition}, we study the model's predictions under a soft intervention and verify these in simulations.
In Section~\ref{sec:application:real-world protein signalling} we take a closer look at some real-world data, more specifically, the data from \citep{Sachs2005}, and compare the model's predictions with the data.
In Section~\ref{sec:application:causal discovery}, we explain how the phenomenon of perfect adaptation may lead to unexpected outcomes of causal discovery methods.
In the end, we will have to conclude that the causal structure of the RAS-RAF-MEK-ERK cascade seems far from understood, and that it seems unlikely that the data in \citep{Sachs2005} is sufficiently rich to be able to draw strong conclusions regarding the causal behavior of the signalling network.


\subsection{Dynamical model}
\label{sec:application:dynamical model}

We adapt the mathematical model of \citep{Shin2009} for the RAS-RAF-MEK-ERK signalling cascade, as in \citep{BlomMooij_UAI_22}.\footnote{For simplicity, we omitted the feedback mechanism through RAF Kinase Inhibitor Protein (RKIP). In the differential equation for activated MEK we therefore discarded the dependence on RKIP. The goal here is not to give the most realistic model but to elucidate the phenomenon of perfect adaptation and the causal interpretation of the equilibrium Markov ordering graph for perfectly adapted dynamical systems.}
 Let $V=\{v_s, v_r, v_m, v_e\}$ be an index set for endogenous variables that represent the equilibrium concentrations $X_{s}$, $X_{r}$, $X_{m}$, and $X_{e}$ of active (phosphorylated) RAS, RAF, MEK, and ERK proteins, respectively. We model their dynamics as:
\begin{align}
\label{eq:mapk s}
\dot{X}_{s}(t) &= \frac{I(t) k_{Is} \left(T_s-X_s(t)\right)}{\left(K_{Is} + (T_s-X_s(t)) \right) \left(1+\left(\frac{X_e(t)}{K_e}\right)^{\frac{3}{2}}\right) } -  F_{s} k_{F_s s} \frac{X_s(t)}{K_{F_s s} + X_s(t)} \\
\label{eq:mapk r}
\dot{X}_r(t) &= \frac{X_s(t) k_{sr} (T_r - X_r(t))}{K_{sr} + (T_r - X_r(t))} - F_r k_{F_r r} \frac{X_r(t)}{K_{F_r r} + X_r(t)} \\
\label{eq:mapk m}
\dot{X}_m(t) &= \frac{X_r(t) k_{rm} (T_m - X_m(t))}{K_{rm} + (T_m - X_m(t))} - F_m k_{F_m m} \frac{X_m(t)}{K_{F_m m} + X_m(t)} \\
\label{eq:mapk e}
\dot{X}_e(t) &= \frac{X_m(t) k_{me} (T_e - X_e(t))}{K_{me} + (T_e - X_e(t))} - F_e k_{F_e e} \frac{X_e(t)}{K_{F_e e} + X_e(t)},
\end{align}
where we assume that $I(t)$ is an external stimulus or perturbation. Roughly speaking, there is a signalling pathway that goes from $I(t)$ to $X_s(t)$ to $X_r(t)$ to $X_m(t)$ to $X_e(t)$ with negative feedback from $X_e(t)$ on $X_s(t)$. As we did for the reaction network with negative feedback in Section \ref{sec:perfect adaptation:examples:negative feedback loop}, we will consider the system under certain saturation conditions. Specifically, for $(T_e-X_e(t))\gg K_{me}$ and $X_e(t)\gg K_{F_e e}$ the following approximation holds:
\begin{align}
\label{eq:erk approximation}
\dot{X}_e(t) \approx X_m(t) k_{me} - F_e k_{F_e e}.
\end{align}
We let $f_s$, $f_r$, $f_m$, and $f_e$ represent the equilibrium equations corresponding to the dynamical equations in \eqref{eq:mapk s}, \eqref{eq:mapk r}, \eqref{eq:mapk m}, and \eqref{eq:erk approximation} respectively, where we assume the input signal to have a constant (possibly random) value $I$. 

We simulated the model under these saturation conditions (picking values for $K_{me}$ and $K_{F_ee}$ close to zero) until it reached equilibrium, and then we recorded the changes in the concentrations $X_s(t)$, $X_r(t)$, $X_m(t)$ and $X_e(t)$ after a change in the input signal $I$. The details of this simulation can be found in Appendix~\ref{app:perfect adaptation simulations}. The results in Figure \ref{fig:adaptation} show that active RAS, RAF, and MEK revert to their original values after an initial response, while the equilibrium value of active ERK changes. Apparently, the equilibrium concentrations $X_s$, $X_r$, and $X_m$ perfectly adapt to the input signal $I$, while the equilibrium concentration $X_e$ of active ERK depends on the input signal $I$.

\begin{figure}[ht]\centering
	\begin{subfigure}[b]{0.24\textwidth}
		\centering
    \includegraphics[width=\textwidth]{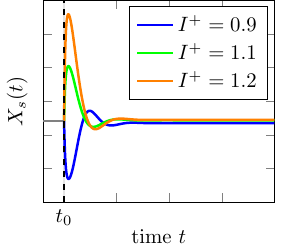}
    \caption{Concentration of active RAS ($X_s(t)$) against time ($t$).}
		\label{fig:adaptation:MAPK:RAS}
	\end{subfigure}\ 
	\begin{subfigure}[b]{0.24\textwidth}
		\centering
		\includegraphics[width=\textwidth]{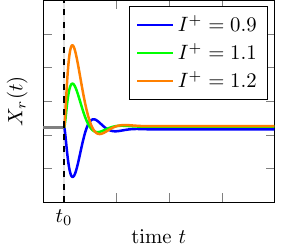}
    \caption{Concentration of active RAF ($X_r(t)$) against time ($t$).}
		\label{fig:adaptation:MAPK:RAF}
	\end{subfigure}\ 
	\begin{subfigure}[b]{0.24\textwidth}
		\centering
		\includegraphics[width=\textwidth]{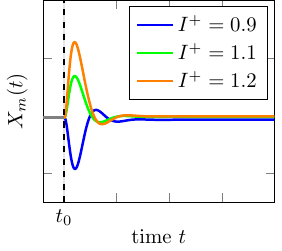}
		\caption{Concentration of active MEK ($X_m(t)$) against time ($t$).}
		\label{fig:adaptation:MAPK:MEK}
	\end{subfigure}\ 
	\begin{subfigure}[b]{0.24\textwidth}
		\centering
		\includegraphics[width=\textwidth]{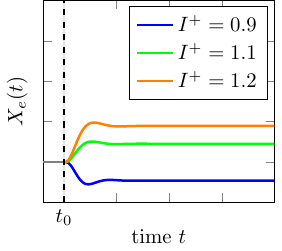}
		\caption{Concentration of active ERK ($X_e(t)$) against time ($t$).}
		\label{fig:adaptation:MAPK:ERK}
	\end{subfigure}
	\caption{Perfect adaptation in the model for the RAS-RAF-MEK-ERK signalling pathway. After an initial response to a change of the input signal from value $I^-=1.0$ to a value $I^+$ at time $t=t_0$, the equilibrium concentrations of active RAS, RAF, and MEK revert to their original values, while the concentration of active ERK changes.}
	\label{fig:adaptation:MAPK}
\end{figure}

\subsection{Graphical representations}
\label{sec:application:causal ordering}

We consider graphical representations of the protein signalling pathway. 
Using the natural labelling, we construct the dynamic bipartite graph in Figure~\ref{fig:protein pathway:dynamic bipartite graph} from the first-order differential equations, with the input signal $I$ included. The associated dynamic causal ordering graph is given in Figure~\ref{fig:protein pathway:dynamic causal ordering graph}.

Under saturation conditions, the equilibrium equations $f_s$, $f_r$, $f_m$, and $f_e$ obtained by setting equations \eqref{eq:mapk s}, \eqref{eq:mapk r}, \eqref{eq:mapk m}, and \eqref{eq:erk approximation} to zero have the bipartite structure in Figure~\ref{fig:protein pathway:bipartite graph}. Note that there is no edge $(f_e-v_e)$ in the equilibrium bipartite graph because $X_E(t)$ does not appear in the approximation \eqref{eq:erk approximation} of \eqref{eq:mapk e}. The associated equilibrium causal ordering graph is given in Figure~\ref{fig:protein pathway:causal ordering graph}, where the cluster $\{I\}$ is added with an edge towards the cluster $\{v_e, f_s\}$ because $I$ appears in equation~\eqref{eq:mapk s} and in no other equations. So far we have treated all symbols in equations \eqref{eq:mapk s}, \eqref{eq:mapk r}, \eqref{eq:mapk m}, and \eqref{eq:mapk e} as deterministic parameters. Let $w_s$, $w_r$, $w_m$, and $w_e$ represent independent exogenous random variables appearing in the equilibrium equations $f_s$, $f_r$, $f_m$, and $f_e$ respectively. After adding them to the causal ordering graph with edges to their respective clusters we construct the equilibrium Markov ordering graph for the equilibrium distribution in Figure~\ref{fig:protein pathway:markov ordering graph}.

\begin{figure}[ht]
%
%
%
	\begin{subfigure}[b]{0.33\textwidth}
		\begin{tikzpicture}[scale=0.75,every node/.style={transform shape}]
		\GraphInit[vstyle=Normal]
		\SetGraphUnit{1}
		\Vertex[L=$v_s$,x=0,y=0] {vs}
		\Vertex[L=$v_r$,x=1.1,y=0] {vr}
		\Vertex[L=$v_m$,x=2.2,y=0] {vm}
		\Vertex[L=$v_e$,x=3.3,y=0] {ve}
		\Vertex[L=$g_s$,x=0,y=-1.5] {fs}
		\Vertex[L=$g_r$,x=1.1,y=-1.5] {fr}
		\Vertex[L=$g_m$,x=2.2,y=-1.5] {fm}
		\Vertex[L=$g_e$,x=3.3,y=-1.5] {fe}
		\begin{scope}[VertexStyle/.append style = {minimum size = 4pt, 
			inner sep = 0pt,
			color=black}]
		\Vertex[x=-1.0, y=-1.5, L=$I$, Lpos=180, LabelOut]{I}
		\end{scope}
		\draw[EdgeStyle, style={->}](I) to (fs);
		\draw[EdgeStyle, style={-}](vs) to (fs);
		\draw[EdgeStyle, style={-}](vs) to (fr);
		\draw[EdgeStyle, style={-}](vr) to (fr);
		\draw[EdgeStyle, style={-}](vr) to (fm);
		\draw[EdgeStyle, style={-}](vm) to (fm);
		\draw[EdgeStyle, style={-}](vm) to (fe);
		\draw[EdgeStyle, style={-}](ve) to (fe);
		\draw[EdgeStyle, style={-}](ve) to (fs);
		\end{tikzpicture}
		\caption{Dynamic bipartite graph.}
		\label{fig:protein pathway:dynamic bipartite graph}
	\end{subfigure}%
	\begin{subfigure}[b]{0.33\textwidth}
		\begin{tikzpicture}[scale=0.75,every node/.style={transform shape}]
		\GraphInit[vstyle=Normal]
		\SetGraphUnit{1}
		\Vertex[L=$v_m$,x=3.0,y=0] {vm}
		\Vertex[L=$v_r$,x=2.0,y=0] {vr}
		\Vertex[L=$v_s$,x=1.0,y=0] {vs}
		\Vertex[L=$v_e$,x=0.0,y=0] {ve}
		\Vertex[L=$g_e$,x=3.0,y=-1.0] {fe}
		\Vertex[L=$g_m$,x=2.0,y=-1.0] {fm}
		\Vertex[L=$g_r$,x=1.0,y=-1.0] {fr}
		\Vertex[L=$g_s$,x=0.0,y=-1.0] {fs}
		
		\begin{scope}[VertexStyle/.append style = {minimum size = 4pt, 
			inner sep = 0pt,
			color=black}]
		\Vertex[x=-1.0, y=0.0, L=$I$, Lpos=180, LabelOut]{I}
		\end{scope}
		
		\node[draw=black, fit=(vm) (fe) (vr) (fm) (vs) (fr) (ve) (fs), inner sep=0.1cm]{};
		\draw[EdgeStyle, style={->}](I) to (-0.5,0.0);
		\end{tikzpicture}
		\caption{Dynamic causal ordering graph.}
		\label{fig:protein pathway:dynamic causal ordering graph}
	\end{subfigure}
  \\
	\begin{subfigure}[b]{0.33\textwidth}
		\vspace*{3mm}
		\begin{tikzpicture}[scale=0.75,every node/.style={transform shape}]
		\GraphInit[vstyle=Normal]
		\SetGraphUnit{1}
		\Vertex[L=$v_s$,x=0,y=0] {vs}
		\Vertex[L=$v_r$,x=1.1,y=0] {vr}
		\Vertex[L=$v_m$,x=2.2,y=0] {vm}
		\Vertex[L=$v_e$,x=3.3,y=0] {ve}
		\Vertex[L=$f_s$,x=0,y=-1.5] {fs}
		\Vertex[L=$f_r$,x=1.1,y=-1.5] {fr}
		\Vertex[L=$f_m$,x=2.2,y=-1.5] {fm}
		\Vertex[L=$f_e$,x=3.3,y=-1.5] {fe}
		\begin{scope}[VertexStyle/.append style = {minimum size = 4pt, 
			inner sep = 0pt,
			color=black}]
		\Vertex[x=-1.0, y=-1.5, L=$I$, Lpos=180, LabelOut]{I}
		\end{scope}
		\draw[EdgeStyle, style={->}](I) to (fs);
		\draw[EdgeStyle, style={-}](vs) to (fs);
		\draw[EdgeStyle, style={-}](vs) to (fr);
		\draw[EdgeStyle, style={-}](vr) to (fr);
		\draw[EdgeStyle, style={-}](vr) to (fm);
		\draw[EdgeStyle, style={-}](vm) to (fm);
		\draw[EdgeStyle, style={-}](vm) to (fe);
		\draw[EdgeStyle, style={-}](ve) to (fs);
		\end{tikzpicture}
		\caption{Equilibrium bipartite graph.}
		\label{fig:protein pathway:bipartite graph}
	\end{subfigure}%
	\begin{subfigure}[b]{0.33\textwidth}
		\begin{tikzpicture}[scale=0.75,every node/.style={transform shape}]
		\GraphInit[vstyle=Normal]
		\SetGraphUnit{1}
		\Vertex[L=$v_m$,x=3.9,y=0] {vm}
		\Vertex[L=$v_r$,x=2.6,y=0] {vr}
		\Vertex[L=$v_s$,x=1.3,y=0] {vs}
		\Vertex[L=$v_e$,x=0.0,y=0] {ve}
		\Vertex[L=$f_e$,x=3.9,y=-1.2] {fe}
		\Vertex[L=$f_m$,x=2.6,y=-1.2] {fm}
		\Vertex[L=$f_r$,x=1.3,y=-1.2] {fr}
		\Vertex[L=$f_s$,x=0.0,y=-1.2] {fs}
		
		\begin{scope}[VertexStyle/.append style = {minimum size = 4pt, 
			inner sep = 0pt,
			color=black}]
		\Vertex[x=-1.0, y=0.0, L=$I$, Lpos=180, LabelOut]{I}
		\end{scope}
		
		\node[draw=black, fit=(vm) (fe), inner sep=0.1cm]{};
		\node[draw=black, fit=(vr) (fm), inner sep=0.1cm]{};
		\node[draw=black, fit=(vs) (fr), inner sep=0.1cm]{};
		\node[draw=black, fit=(ve) (fs), inner sep=0.1cm]{};
		
		\draw[EdgeStyle, style={->}](I) to (-0.525,0.0);
		\draw[EdgeStyle, style={->}](vs) to (0.525,0.0);
		\draw[EdgeStyle, style={->}](vr) to (1.825,0.0);
		\draw[EdgeStyle, style={->}](vm) to (3.125,0.0);
		\end{tikzpicture}
		\caption{Equilibrium causal ordering graph.}
		\label{fig:protein pathway:causal ordering graph}
	\end{subfigure}%
	\begin{subfigure}[b]{0.33\textwidth}
		\begin{tikzpicture}[scale=0.75,every node/.style={transform shape}]
		\GraphInit[vstyle=Normal]
		\SetGraphUnit{1}
		\Vertex[L=$v_e$, x=0.0, y=0] {vE}
		\Vertex[L=$v_s$, x=1.1, y=0] {vS}
		\Vertex[L=$v_r$, x=2.2, y=0] {vR}
		\Vertex[L=$v_m$, x=3.3, y=0] {vM}
		\Vertex[L=$w_{s}$, x=0, y=-1.1, style={dashed}] {wS}
		\Vertex[L=$w_{r}$, x=1.1, y=-1.1, style={dashed}] {wR}
		\Vertex[L=$w_{m}$, x=2.2, y=-1.1, style={dashed}] {wM}
		\Vertex[L=$w_{e}$, x=3.3, y=-1.1, style={dashed}] {wE}
		
		\begin{scope}[VertexStyle/.append style = {minimum size = 4pt, 
			inner sep = 0pt,
			color=black}]
		\Vertex[x=-1.0, y=0.0, L=$I$, Lpos=180, LabelOut]{I}
		\end{scope}
		\draw[EdgeStyle, style={->}](I) to (vE);
		\draw[EdgeStyle, style={->}](wS) to (vE);
		\draw[EdgeStyle, style={->}](wR) to (vS);
		\draw[EdgeStyle, style={->}](wM) to (vR);
		\draw[EdgeStyle, style={->}](wE) to (vM);
		\draw[EdgeStyle, style={->}](vM) to (vR);
		\draw[EdgeStyle, style={->}](vR) to (vS);
		\draw[EdgeStyle, style={->}](vS) to (vE);
		\end{tikzpicture}
		\caption{Equilibrium Markov ordering graph.}
		\label{fig:protein pathway:markov ordering graph}
	\end{subfigure}
  \caption{Five graphs associated with the protein signalling pathway model under saturation conditions where indices $s,r,m,e$ correspond to concentrations of active RAS, RAF, MEK, and ERK respectively.}
	\label{fig:protein pathway}
\end{figure}
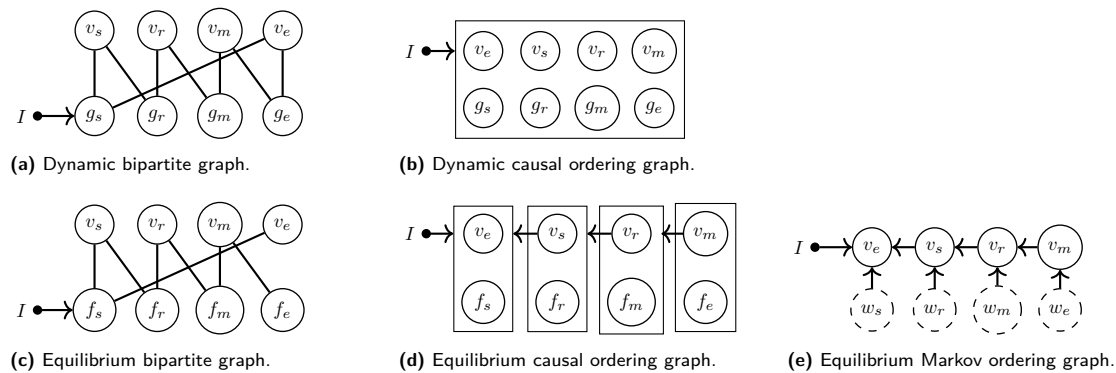

There is a directed path from the input vertex $I$ to $v_s$, $v_r$, and $v_m$ in the dynamic causal ordering graph (see Figure~\ref{fig:protein pathway:dynamic causal ordering graph}), while the equilibrium causal ordering graph has no directed paths from $I$ to either $v_s$, $v_r$, or $v_m$ (see Figure~\ref{fig:protein pathway:causal ordering graph}). Under Assumption~\ref{ass:transient response}, we can apply Theorem~\ref{thm:identification of perfect adaptation} to conclude that $X_s$, $X_r$, and $X_m$ will perfectly adapt to a persistent change in the input signal $I$. This is in line with what we observed in the simulations (see Figure~\ref{fig:adaptation:MAPK}).


The $d$-separations in the equilibrium Markov ordering graph (see Figure~\ref{fig:protein pathway:markov ordering graph}) imply conditional independences between the corresponding variables at equilibrium. For example, from the graph in Figure~\ref{fig:protein pathway:markov ordering graph} we read off the following (implied) conditional independences:
\begin{align*}
\quad I \indep v_r,  \quad I \indep v_m, \quad \quad v_e \indep v_m \given v_r\,.
\end{align*}
We verified that these conditional independences indeed appear in the simulated equilibrium distribution of the model (see Appendix~\ref{app:conditional independences} for details).

\subsection{Inhibiting the activity of MEK}
\label{sec:application:MEK inhibition}

A common biological experiment that is used to study protein signalling pathways is the use of an inhibitor that decreases the activity of a protein on the pathway. Such an inhibitor slows down the rate at which the active protein is able to activate another protein. Here, we consider inhibition of MEK activity. We can model this as a change of the parameters of the differential equations in which $X_m(t)$ appears. 
We can interpret this experiment as a soft intervention on differential equation $g_e$ in the dynamic model and on equation $f_e$ at equilibrium, decreasing the rate $k_{me}$ at which ERK is activated. Since there is a directed path from $f_e$ to $v_m,v_r,v_s$, and $v_e$ in the equilibrium causal ordering graph in Figure~\ref{fig:protein pathway:causal ordering graph}, we expect that a change in an input signal $I_e$ on $f_e$ (e.g.\ a change in the parameter $k_{me}$ in the case of the MEK inhibition) affects the equilibrium concentrations of active MEK, RAF, RAS, and ERK. 

We assessed the effect of decreasing the activity of MEK on the equilibrium concentrations of RAS, RAF, MEK, and ERK. To that end, we simulated the perfectly adapted model (with parameters as described in Appendix~\ref{app:perfect adaptation simulations}, in particular, $k_{me}=1.0$) until it reached equilibrium. We then decreased the parameter that controls the activity of MEK to $k_{me}=0.98$. The recorded responses of the concentrations of active RAS, RAF, MEK, and ERK are displayed in Figure~\ref{fig:experiments:simulations:inhibition}. From this we confirm our prediction that inhibition of MEK activity affects the equilibrium concentrations of RAS, RAF, MEK, and ERK. A qualitative aspect of this change that one cannot read off from the graph is that the MEK inhibition \emph{increases} (rather than decreases) the concentrations of active MEK, RAF, RAS according to the model for this choice of the parameters.

Note that RAS, RAF, and MEK are non-descendants of ERK in the equilibrium Markov ordering graph in Figure~\ref{fig:protein pathway:markov ordering graph}, so that under the assumptions in Theorem~\ref{thm:detect perfect adaptation} we could actually use this experiment to detect perfect adaptation in the protein pathway. 

\begin{figure}[ht]
	\centering
	\includegraphics{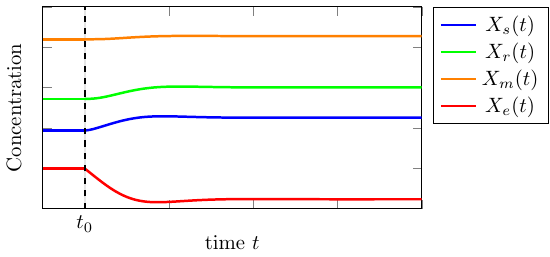}
	\caption{Simulation of the response of the concentrations of active RAS, RAF, MEK, and ERK after inhibition of the activity of MEK. The system starts out in equilibrium with $k_{me}=1.0$. The concentrations of RAS, RAF, MEK, and ERK are recorded after the parameter controlling MEK activity is decreased to $k_{me}=0.98$ from $t=t_0$ on.}
	\label{fig:experiments:simulations:inhibition}
\end{figure}

\subsection{Testing model predictions on real-world data}
\label{sec:application:real-world protein signalling}

In this subsection, we verify some of the predictions of the model we obtained in Sections~\ref{sec:application:causal ordering} and \ref{sec:application:MEK inhibition} on real-world data.
We will compare with predictions of the causal Bayesian network model proposed by \citep{Sachs2005}.

Figure~\ref{fig:sachs} shows scatter-plots for the (logarithms of) the expressions of active RAF, MEK and ERK in the multivariate
single-cell protein expression dataset that was used in \citep{Sachs2005}, for three (out of 14) different experimental conditions.
The baseline condition (in blue) is the one where the cells were treated with anti-CD3 and anti-CD28, activators of the RAS-RAF-MEK-ERK signalling cascade.
In another condition (in red), the cells were additionally exposed to U0126, a known inhibitor of MEK activity.
By inspecting the scatter plots, we get a quick visual check of some of the predictions of the model.
In particular, these plots clearly show that inhibition of MEK activity by administering U0126 results in an increase in the concentrations of active RAF and active MEK and a reduction in the concentration of active ERK. 
Furthermore, we clearly see a strong dependence between RAF and MEK (in both experimental conditions), but there is no discernible dependence between RAF and ERK or between MEK and ERK (in either experimental condition). 
In the light of the supposedly direct effect of MEK on ERK, it is surprising that the data shows no significant dependence between the two.\footnote{In addition, ERK levels are considerably lower than those of RAF and MEK. This might be something specific to this experimental setting (perhaps there was too much time between stimulation and measurement to see the ERK response), as in other experiments high ERK levels and strong correlations with MEK have been observed \citep{Filippi++2016}. In the setting of the experiment of \citep{Sachs2005}, the G06976 treatment condition shows that ERK levels can actually get as high as those of RAF and MEK.} 
This apparent `faithfulness violation' is problematic for constraint-based causal discovery methods, as they will typically not identify the causal relation between MEK and ERK.

\begin{figure}[ht]\centering
	\begin{subfigure}[b]{0.32\textwidth}
		\centering
		\includegraphics[width=\textwidth]{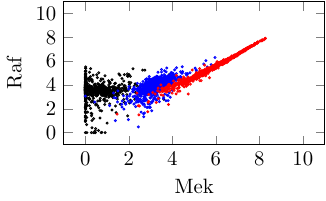}
		\caption{Log-expressions of active MEK and RAF.}
		\label{fig:sachs:rafmek}
	\end{subfigure}\hfill
	\begin{subfigure}[b]{0.33\textwidth}
		\centering
		\includegraphics[width=\textwidth]{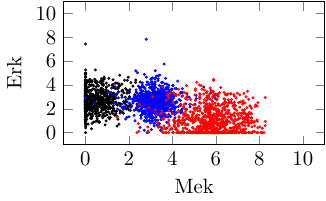}
		\caption{Log-expressions of active MEK and ERK.}
		\label{fig:sachs:erkmek}
	\end{subfigure}\hfill
	\begin{subfigure}[b]{0.32\textwidth}
		\centering
    \includegraphics[width=\textwidth]{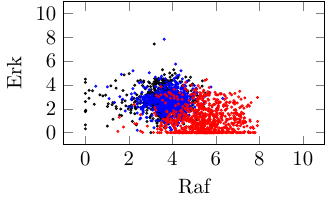}
		\caption{Log-expressions of active RAF and ERK.}
		\label{fig:sachs:raferk}
	\end{subfigure}%
  \caption{Scatter plots of the logarithms of active RAF, MEK, and ERK concentrations for the data in \citep{Sachs2005}.
  The blue circles correspond to cells treated only with anti-CD3 and anti-CD28, which activate the signaling cascade.
  The red circles correspond to cells treated with anti-CD3, anti-CD28 and in addition, the MEK-activity inhibitor U0126. 
  The inhibition of MEK results in an increase of MEK and RAF, whereas ERK is reduced. The black circles correspond to cells treated with $\beta$2cAMP (but not anti-CD3 and anti-CD28), which seems to affect MEK, but leaves RAF and ERK invariant.}
	\label{fig:sachs}
\end{figure}

\begin{table}[b]\centering
  \caption{For various observations in the data of \citep{Sachs2005}, we indicate whether they are predicted by 
  the causal Bayesian network model $\mathrm{RAF} \to \mathrm{MEK} \to \mathrm{ERK}$ proposed by \citep{Sachs2005} and
  by our perfectly adaptive equilibrium model (Section~\ref{sec:application:causal ordering}).\label{tab:sachs:predictions}}
\begin{tabular}{lll}
\toprule
  Observation & Causal Bayesian network model & Perfectly adaptive model \\
\midrule
  RAF and MEK are dependent in both conditions                 & $+$ & $+$ \\
  MEK and ERK are independent in both conditions               & $-$ & $-$ \\
  Inhibition of MEK activity affects active RAF & $-$ & $+$ \\
  Inhibition of MEK activity affects active MEK & $-$ & $+$ \\
  Inhibition of MEK activity affects active ERK & $+$ & $+$ \\
		\bottomrule
\end{tabular}
\end{table}

According to \citep{Sachs2005}, the biological consensus (at the time) was that there is a signalling pathway from RAF to MEK to ERK.\footnote{Nowadays, the biological consensus appears to be that RAF activates MEK, MEK activates ERK, and that it is very likely that there is negative feedback from ERK to RAF, although the molecular pathway of this feedback remains unknown \citep{Fritsche2011}.}
They propose to model this as a causal Bayesian network $\mathrm{RAF} \to \mathrm{MEK} \to \mathrm{ERK}$, and this is also the structure identified by their causal discovery algorithm.
That model predicts that inhibiting MEK activity can only affect ERK.
In Table~\ref{tab:sachs:predictions}, we summarize some of the observations made using the scatter plots in Figure~\ref{fig:sachs}, and whether or not they are in line with the predictions of the models (our perfectly adaptive model on the one hand, and the causal Bayesian network model on the other hand).
We conclude that the predictions of the perfectly adaptive model are more in agreement with the data than those of the causal Bayesian network model.

Still, the perfectly adaptive model does not explain all aspects of the data. 
For example, the effects of the $\beta$2cAMP stimulation (black circles in Figure~\ref{fig:sachs}) are hard to explain with either model. 
$\beta$2cAMP is an AMP analogue that is supposed to activate the RAS-RAF-MEK-ERK cascade by promoting active RAS.
It seems counterintuitive that this would lead to a strong reduction of active MEK in comparison to the activation of the cascade by means of anti-CD3 and anti-CD28, while leading to the same levels of active RAF and ERK.
For completeness, in Table~\ref{tab:sachs:effects} we have indicated for all 12 perturbations in \citep{Sachs2005} the effects on the levels of active RAF, MEK and ERK. 
It appears questionable whether a simple causal model (such as the perfectly adaptive model) could account for all the observed perturbation effects.

\begin{table}[h!b]
  \caption{Qualitative effects of reagents on the measured abundances of active RAF, MEK and ERK, as read off from the data in \cite{Sachs2005}. Legend: $--$ strong decrease, $-$ decrease, ($-$) slight decrease, $0$ no change, ($+$) slight increase, $+$ increase, $++$ strong increase. Most researchers only use conditions 1--9 for causal discovery. The data appears to contain an error: RAF and MEK values are identical for the first 848 samples in conditions 3 and 7; therefore, we disregarded those values.}
  \label{tab:sachs:effects}
	\centering
	\begin{tabular}{l l c c c}
		\toprule
    Condition & Reagents                                        & RAF   & MEK   & ERK \\
		\midrule
    1         & anti-CD3 + anti-CD28                            & \multicolumn{3}{c}{\dots\dots baseline \dots\dots} \\
    3         & anti-CD3 + anti-CD28 + AKT inhibitor            & NA    & NA    & $0$   \\
    4         & anti-CD3 + anti-CD28 + G06976                   & $++$  & $++$  & $++$  \\
    5         & anti-CD3 + anti-CD28 + Psitectorigenin          & $0$   & $0$   & $--$  \\
    6         & anti-CD3 + anti-CD28 + U0126                    & $++$  & $++$  & ($+$) \\
    7         & anti-CD3 + anti-CD28 + LY294002                 & NA    & NA    & $0$   \\
    8         & PMA                                             & $-$   & $0$   & $0$   \\
    9         & $\beta$2cAMP                                    & ($-$) & $--$  & ($+$) \\
    \midrule
    2         & anti-CD3 + anti-CD28 + ICAM-2                   & \multicolumn{3}{c}{\dots\dots baseline \dots\dots} \\
    10        & anti-CD3 + anti-CD28 + ICAM-2 + AKT inhibitor   & $0$   & $0$   & $--$  \\
    11        & anti-CD3 + anti-CD28 + ICAM-2 + G06976          & $++$  & $++$  & $++$  \\
    12        & anti-CD3 + anti-CD28 + ICAM-2 + Psitectorigenin & $0$   & $0$   & $--$  \\
    13        & anti-CD3 + anti-CD28 + ICAM-2 + U0126           & $0$   & $(+)$ & $--$  \\
    14        & anti-CD3 + anti-CD28 + ICAM-2 + LY294002        & $0$   & $0$   & $--$  \\
		\bottomrule
	\end{tabular}
\end{table}

\subsection{Caveats for causal discovery}\label{sec:application:causal discovery}

Experiments in which the protein signalling network is perturbed in various ways are of crucial importance to obtaining a causal understanding of the system.
While very sophisticated causal discovery algorithms are available, we will here illustrate the key concepts by means of applying one of the simplest causal discovery algorithms based on conditional independences to equilibrium data from the model.
We simulate the system in two different conditions, a baseline and a condition where the activity of MEK has been inhibited.

Consider observational equilibrium data from the protein signalling pathway model and also experimental equilibrium data from a setting where MEK activity is inhibited.
We introduce a \emph{context} variable $C$ that in this case simply indicates for each sample whether or not a MEK inhibition was applied.
Because the decision whether to apply the MEK inhibitor to a sample occurs \emph{before} the measurement, the equilibrium concentrations measured afterwards cannot causally affect this decision. 
Hence $C$ is an exogenous variable, i.e., it is not caused by other observed variables.
This motivates the use of the LCD algorithm, which requires the existence of such a variable.
For the MEK inhibition, the context variable represents a (soft) intervention on the equation $f_e$ in the equilibrium causal ordering graph in Figure~\ref{fig:protein pathway:causal ordering graph}. 
The equilibrium Markov ordering graph that includes the context variable $C$ (but where the independent exogenous random variables have been marginalized out) is given in Figure~\ref{fig:protein pathway:lcd:model}. 
To construct this graph, the context variable $C$ is first added to the equilibrium causal ordering graph in Figure~\ref{fig:protein pathway:causal ordering graph} as a singleton cluster with an edge towards the cluster $\{v_m, f_e\}$. 
The equilibrium Markov ordering graph is then constructed from the resulting directed cluster graph in the usual way. 
From Figure~\ref{fig:protein pathway:lcd:model}, we can read off (conditional) independences to find the LCD triples that are implied by the equilibrium equations of the model. 
We find that our model implies the following LCD triples: $(C, v_m, v_r)$, $(C, v_m, v_s)$, $(C,v_m,v_e)$, $(C,v_r,v_s)$, $(C,v_r,v_e)$, and $(C,v_s, v_e)$.
We verified this by explicit simulation (details in Appendix~\ref{app:conditional independences}).

In particular, one of the LCD triples we obtained is $(C, v_m, v_r)$, whereas the triple $(C, v_r, v_m)$ does not qualify as such.
According to the standard interpretation of causal discovery algorithms, this would imply that MEK causes RAF rather than the other way around (see also Section~\ref{sec:background:local causal discovery}),
a surprising finding in the light of the text-book treatments of the RAS-RAF-MEK-ERK signalling pathway, and the `consensus network' according to \citep{Sachs2005}.
The equilibrium Markov ordering graph corresponding to the causal Bayesian network model of \citep{Sachs2005} is shown in Figure~\ref{fig:protein pathway:lcd:sachs}.
As one can see from Figure~\ref{fig:protein pathway:lcd:sachs}, the causal Bayesian network model implies no LCD triples at all.

Apparently, the causal relationship RAF$\to$MEK appears to be reversed in the LCD triples found in the equilibrium data of the perfectly adaptive model.
Similar observations of apparent `causal reversals' in protein interaction networks have been observed more often, see also \citep{Triantafillou2017, Mooij2020, Mooij2013b, Ramsey2018, Boeken2020}.
We conclude that the mechanism of perfect adaptation provides one possible theoretical explanation of what might seem at first sight to be an incorrect reversal of a causal edge.

\begin{table}[b]
  \centering
  \caption{Results of conditional independence tests on the (log-transformed) protein expression data of \citep{Sachs2005}. Specifically, we report the p-values of Kendall's test for partial correlation.}
	\label{tab:sachs:conditional independences}
    \begin{tabular}{ll}
		\toprule
      (Conditional) independence tested & p-value \\
		\midrule
      $v_r \indep v_m \given C$ & $<4.6 \times {10}^{-324}$ \\
      $v_r \indep v_e \given C$ & $0.79$ \\
      $v_m \indep v_e \given C$ & $0.35$ \\
      $v_r \indep C$            & $4.0\times 10^{-172}$ \\
      $v_m \indep C$            & $3.7\times 10^{-245}$ \\
      $v_e \indep C$            & $3.5\times 10^{-138}$ \\
      $v_r \indep C \given v_m$ & $2.7\times 10^{-8}$ \\
      $v_m \indep C \given v_r$ & $8.2\times 10^{-168}$ \\
      $v_r \indep C \given v_e$ & $4.3\times 10^{-210}$ \\
      $v_m \indep C \given v_e$ & $<4.6 \times {10}^{-324}$ \\
      $v_e \indep C \given v_m$ & $1.5\times 10^{-134}$ \\
      $v_e \indep C \given v_r$ & $2.6\times 10^{-156}$ \\
    \bottomrule
    \end{tabular}
\end{table}

We investigated which of these LCD patterns can be found in the real-world data.
In Table~\ref{tab:sachs:conditional independences} we list the $p$-values of conditional independence tests applied to the data of \citep{Sachs2005}.\footnote{We made use of Kendall's test for partial correlation because the implementation in the R package \texttt{ppcor} can deal with ties, while that for Spearman's test cannot.}
Using a reasonable critical level for the test (say $\alpha=0.01$), we do not find any LCD pattern.
Yet, the conditional dependence of RAF on the context variable when conditioning on MEK is much weaker than the conditional dependence of MEK on the context variable when conditioning on RAF. 
Strictly speaking, no conclusions should be drawn from this, but it does seem to suggest that also here the data is more in line with the perfectly adaptive model than with the causal Bayesian network model; indeed, if the adaptation is imperfect, for example because the saturation conditions only hold approximately, one would expect to see a weak conditional dependence $v_r \not\indep C \given v_m$.

\begin{figure}[ht]
	\centering
	\begin{subfigure}[b]{0.4\textwidth}
    \begin{tikzpicture}
      \centering
      \GraphInit[vstyle=Normal]
      \SetGraphUnit{1}
      \Vertex[L=$v_e$, x=0.0, y=0] {vE}
      \Vertex[L=$v_s$, x=1.1, y=0] {vS}
      \Vertex[L=$v_r$, x=2.2, y=0] {vR}
      \Vertex[L=$v_m$, x=3.3, y=0] {vM}
      
      \begin{scope}[VertexStyle/.append style = {minimum size = 4pt, 
        inner sep = 0pt,
        color=black}]
      \Vertex[x=0, y=1.0, L=$I$, Lpos=180, LabelOut]{I}
      \Vertex[L=$C$, x=3.3, y=1.0, Lpos=180, LabelOut]{C}
      \end{scope}
      \draw[EdgeStyle, style={->}](I) to (vE);
      \draw[EdgeStyle, style={->}](vM) to (vR);
      \draw[EdgeStyle, style={->}](vR) to (vS);
      \draw[EdgeStyle, style={->}](vS) to (vE);
      \draw[EdgeStyle, style={->}](C) to (vM);
    \end{tikzpicture}
    \caption{Perfectly adaptive model (with feedback)}
		\label{fig:protein pathway:lcd:model}
	\end{subfigure}\qquad
	\begin{subfigure}[b]{0.3\textwidth}
    \begin{tikzpicture}
      \centering
      \GraphInit[vstyle=Normal]
      \SetGraphUnit{1}
      \Vertex[L=$v_r$, x=0.0, y=0] {vR}
      \Vertex[L=$v_m$, x=1.1, y=0] {vM}
      \Vertex[L=$v_e$, x=2.2, y=0] {vE}
      
      \begin{scope}[VertexStyle/.append style = {minimum size = 4pt, 
        inner sep = 0pt,
        color=black}]
      \Vertex[x=0, y=1.0, L=$I$, Lpos=180, LabelOut]{I}
      \Vertex[L=$C$, x=2.2, y=1.0, Lpos=180, LabelOut]{C}
      \end{scope}
      \draw[EdgeStyle, style={->}](I) to (vR);
      \draw[EdgeStyle, style={->}](vR) to (vM);
      \draw[EdgeStyle, style={->}](vM) to (vE);
      \draw[EdgeStyle, style={->}](C) to (vE);
    \end{tikzpicture}
    \caption{Causal Bayesian network model}
		\label{fig:protein pathway:lcd:sachs}
	\end{subfigure}
	\caption{Equilibrium Markov ordering graphs of the protein signalling pathway with the context variable $C$ included, corresponding to two  hypothetical models. This context variable indicates whether a cell was treated with a MEK inhibitor or not.}
	\label{fig:protein pathway:lcd}
\end{figure}
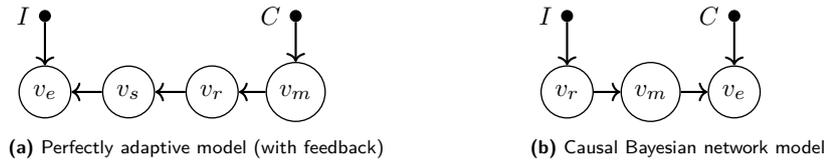

\section{Discussion and conclusion}
\label{sec:discussion}

Perfect adaptation is the phenomenon that a dynamical system initially responds to a change of input signal but reverts back to its original value as the system converges to equilibrium. We used the technique of causal ordering to obtain sufficient graphical conditions to identify perfect adaptation in a dynamical system described by a combination of equations and first-order differential equations. To represent the structure of the (non-equilibrium) dynamical system, we introduced the notions of the \emph{dynamic bipartite graph} and the corresponding \emph{dynamic causal ordering graph} obtained by the causal ordering algorithm. Moreover, we showed how perfect adaptation can be detected in equilibrium observational and experimental data for soft interventions with known targets.
We illustrated our ideas on a variety of dynamical models and corresponding equilibrium equations.
We believe that the methods presented in this work provide a useful tool for the characterization of a large class of dynamical systems that are able to achieve perfect adaptation and for the automated analysis of the behavior of certain perfectly adapted dynamical systems.

In all examples that we discussed, the technique of causal ordering revealed the structure of a \emph{given} set of equations.
In some cases, more structure can be revealed by first rewriting the equations before applying the causal ordering algorithm.
In Appendix~\ref{app:rewriting equations} we analyze a dynamical system describing a basic enzyme reaction, for which rewriting of the equilibrium equations reveals more structure (and hence yields a stronger Markov property).
In Appendix~\ref{app:ifflp network}, we analyze the `Incoherent Feed-forward Loop with a Proportioner Node' (IFFLP).
We observe that a nonlinear transformation of the variables (and rewriting the equations correspondingly) reveals more structure, and hence yields more conditional independences and less causal relations amongst the transformed variables.
The further development of methods to analyze perfectly adapted dynamical systems that do not satisfy the conditions of Theorem~\ref{thm:identification of perfect adaptation} remains a challenge for future work. 
Also, the question of how one can more generally discover causal structure using nonlinear transformations of the variables is an interesting topic for future research that relates to what is nowadays known as `causal representation learning' in machine learning \citep{Chalupka++2017}.

We also investigated the consequences of the phenomenon of perfect adaptation for causal discovery.
We demonstrated that for perfectly adapted dynamical systems the output of existing constraint-based causal discovery algorithms applied to equilibrium data may appear counterintuitive and at odds with our understanding of the mechanisms that drive the system. 
As we have illustrated in this work, careful application of the causal ordering algorithm enables a better theoretical understanding of these phenomena.

We applied our approach to a model for a well-known protein signalling pathway, and tested the model's predictions both in simulations and on real-world protein expression data. 
The challenges for causal discovery that are encountered in non-linear dynamical systems with feedback loops, possibly leading to context-specific perfectly adaptive behavior, are substantial.
If the behavior of the model that we analyzed in Section~\ref{sec:application} is representative of that of actual systems occurring \emph{in vitro} and \emph{in vivo}, then it seems unlikely that existing causal discovery methods based on causal Bayesian networks will lead to reasonable results.
These observations further motivate the development of causal discovery algorithms based on bipartite graphical representations that would be more widely applicable than the existing ones based on causal Bayesian networks or simple structural causal models. 

\section*{Acknowledgments}

We thank the reviewers and the editor for their constructive comments, which helped us improve our work.

\medskip
\noindent \textbf{Funding information:}
This work was supported by the ERC under the European Union's Horizon 2020 research and innovation programme (grant agreement 639466).

\medskip
\noindent 
\textbf{Author contributions:}
The idea to use causal ordering to study perfectly adapted dynamical systems
was due to TB, as well as the development of the theoretical results. The
notions of the dynamic bipartite graph and the dynamic causal
ordering graph were proposed by JMM. Simulations were conducted by TB.
Analysis of the protein signaling data and preparation of the manuscript
was done by both authors.
The SDC approach was applied for the sequence of authors.
All authors have accepted responsibility for the entire content of this manuscript and approved its submission.

\medskip
\noindent 
\textbf{Conflict of interest:}
JMM is a member of the Editorial Board of Journal of Causal Inference and
was not involved in the review process of this article.

%

\medskip
\noindent 
\textbf{Data availability:}
The simulated datasets can be reproduced with the \texttt{R} code provided at \url{https://bitbucket.org/jorism/jci2023paper.git} as free and open source software.
The protein expression dataset of \cite{Sachs2005} analyzed in Section~\ref{sec:application:real-world protein signalling}
is publicly available as Supplementary Material to \cite{Sachs2005} at
\url{https://www.science.org/doi/suppl/10.1126/science.1105809/suppl_file/sachs.som.datasets.zip}

\bibliography{library.bib}

\appendix

\section{Proof of Theorem \ref{thm:detect perfect adaptation}}
\label{app:proof of theorem}

\detectpa*

\begin{proof}
	If condition~\ref{lemma:observe:c1} holds there is no directed path from $f_i$ to $v_i$ in the equilibrium causal ordering graph, by the assumption that the soft intervention on $f_i$ changes the equilibrium distribution of all its descendants. By definition of the dynamic bipartite graph there is a directed path from $g_i$ to $v_i$ in the dynamic causal ordering graph, because $g_i$ and $v_i$ end up in the same cluster (note that this follows by using the natural labelling as perfect matching and the result that the causal ordering graph does not depend on the chosen perfect matching \citep{Blom2020}). It follows from Theorem~\ref{thm:identification of perfect adaptation} that $X_i$ perfectly adapts to an input signal $I_{f_i}$ on $f_i$ (i.e.\ a soft intervention targeting $\dot{X}_i(t)$ and thus the equilibrium equation $f_i$). 
	
	Suppose that \ref{lemma:observe:c1} does not hold while \ref{lemma:observe:c2} does hold. By Theorem 4 in \citep{Blom2020} (which roughly states that the presence of a causal effect at equilibrium implies the presence of a corresponding directed path in the equilibrium causal ordering graph) we have that $f_i$ is an ancestor of $v_i$ and some $v_h$ in the equilibrium causal ordering graph, while $v_i$ is not an ancestor of $v_h$ in the equilibrium Markov ordering graph. For a perfect matching $M$ of the equilibrium bipartite graph let $v_j = M(f_i)$. Then $v_j$ is in the same cluster as $f_i$ in the equilibrium causal ordering graph by construction. Note that $j=i$ would give a contradiction, as then $v_i$ would be an ancestor of $v_h$ in the equilibrium Markov ordering graph. Suppose that the vertex $f_j$, which is associated with $v_j$ through the natural labelling, is matched to a non-ancestor of $v_j$ in the equilibrium causal ordering graph. Because of the edge $(g_j - v_j)$ in the dynamic bipartite graph, it follows from Theorem~\ref{thm:identification of perfect adaptation} that $X_j$ perfectly adapts to an input signal $I_{f_j}$ on $f_j$. Therefore the system is able to achieve perfect adaptation. Now suppose that $f_j$ is matched to an ancestor $v_k$ of $v_j$ in the equilibrium causal ordering graph, and consider the vertex $f_k$. The previous argument can be repeated to show perfect adaptation for $X_k$ is present when $f_k$ is matched to a non-ancestor of $v_k$ in the equilibrium causal ordering graph. Otherwise, $f_k$ must be matched to an ancestor of $v_k$. Note that the ancestors of $v_k$ are a subset of the ancestors of $v_j$, which in turn are a subset of the ancestors of $v_i$. In a finite system of equations, $v_i$ has a finite set of ancestors and therefore we eventually find, by repeating our argument, a vertex $f_m$ that cannot be matched to an ancestor of $v_m$ because $v_m$ has no ancestors that are not matched to one of the vertices $f_i, f_j, f_k, \ldots$ that were considered up to that point. Because $f_m$ is matched to a non-ancestor we then find that $X_m$ perfectly adapts to an input signal on $I_{f_m}$ as before.
\end{proof}

\section{Simulation settings}
\label{app:perfect adaptation simulations}

For the simulations in Figures \ref{fig:adaptation}, \ref{fig:adaptation:MAPK} and \ref{fig:experiments:simulations:inhibition} of the model of a filling bathtub, the viral infection model, the reaction network with a feedback loop, and the protein pathway we used the settings listed below. Since we only simulated a single response, we used constant values for the exogenous random variables as well.

\begin{description}
	\item[Filling bathtub] First we recorded the behavior of the system for the parameters $I_K=1.2$, $U_I=5.0$, $U_1=1.1$, $U_2=1.0$, $U_3=1.2$, $U_4=1.0$, $U_5=0.8$, $g=1.0$ until it reached equilibrium. We then changed the input parameter $I_K$ to $0.8$, $1.0$, and $1.3$ and recorded the response until the system reverted to equilibrium.
	\item[Viral infection] For the parameter settings $I_{\sigma}=1.6$, $d_T=0.9$, $\beta=0.9$, $d_I=0.3$, $k=1.5$, $a=0.1$, $d_E=0.25$, we simulated the model until it reached equilibrium. We changed the input parameter $I_{\sigma}$ to $1.1$, $1.3$, and $2.0$ and recorded the response until equilibrium was reached.
	\item[Reaction network]	We simulated the model until it reached equilibrium with parameters $I=1.5$, $k_{IA}=1.4$, $K_{IA}=0.8$, $F_A=1.1$, $k_{F_A A}=0.9$, $K_{F_A A}=1.2$, $k_{CB}=0.6$, $K_{CB}=0.0001$, $F_B=0.7$, $k_{F_B B}=0.7$, $K_{F_B B}=0.0001$, $k_{AC}=2.1$, $K_{AC}=1.5$, $k_{BC}=0.7$, $K_{BC}=0.6$. The settings were chosen in such a way that the saturation conditions $(1-X_B(t))\gg K_{CB}$ and $X_B(t)\gg K_{F_B B}$ were satisfied. We then changed the input signal to $0.25$, $1.0$, and $10.0$ and recorded the response.
  \item[Protein pathway] The parameter settings of the simulation were $I=1.0$, $k_{Is}=1.0$, $T_s=1.0$, $K_{Is}=1.0$, $K_e=1.2$, $F_s=1.0$, $k_{F_s s}=1.0$, $K_{F_s s}=0.9$, $k_{sr}=1.0$, $K_{sr}=1.0$, $T_r=1.0$, $F_r=0.3$, $k_{F_r r}=1.0$, $K_{F_r r}=0.8$, $k_{rm}=1.0$, $K_{rm}=0.9$, $T_m=1.0$, $F_m=0.2$, $k_{F_m m}=1.0$, $K_{F_m m}=1.2$, $k_{me}=1.0$, $K_{me}=0.0001$, $T_e=1.0$, $F_e=0.7$, $k_{F_e e}=1.2$, $K_{F_e e}=0.0001$. This ensured that the saturation conditions $(T_e-X_e(t))\gg K_{me}$ and $X_e(t)\gg K_{F_e e}$ were satisfied. For Figure~\ref{fig:adaptation:MAPK}, we changed the input signal $I$ to $0.9, 1.1, 1.2$, respectively, after the system had reached equilibrium, and continued the simulation. For Figure~\ref{fig:experiments:simulations:inhibition} we reduced the parameter value of $k_{me}$ to $0.98$ at some point in time after the system had equilibrated.
\end{description}

The same qualitative behavior as reported here can be observed for a range of parameter values.
The behavior of the protein pathway is rather complex; in particular, $X_e(t)$ may show a switch-like behavior (either taking on a value close to 0, or close to $T_e$), and therefore some parameter tuning is required if one wants to ensure that the assumed saturation conditions ($(T_e-X_e(t))\gg K_{me}$ and $X_e(t)\gg K_{F_e e}$) hold.

\section{Conditional independences and causal discovery}
\label{app:conditional independences}

The equilibrium Markov ordering graph in Figure~\ref{fig:protein pathway:markov ordering graph} was derived from the equilibrium equations of the protein pathway model under saturation conditions. From this we can read off the following $d$-separations:

\begin{align*}
&I\dsep{} v_s, \quad I\dsep{} v_s\given v_r, \quad I\dsep{}v_s\given v_m, \quad I\dsep{} v_s \given \{v_r, v_m\}, \\
&I\dsep{} v_r, \quad I\dsep{} v_r\given v_s, \quad I\dsep{}v_r\given v_m, \quad I\dsep{} v_r \given \{v_s, v_m\}, \\
&I\dsep{} v_m, \quad I\dsep{} v_m\given v_s, \quad I\dsep{}v_m\given v_r, \quad I\dsep{} v_m \given \{v_s, v_r\}, \\
&v_e \dsep{} v_r \given v_s, \quad v_e \dsep{} v_r \given \{v_s,v_m\}, \quad v_e \dsep{} v_r \given \{v_s, I\}, \quad v_e \dsep{} v_r \given \{v_s,v_m,I\},\\
&v_e \dsep{} v_m \given v_s, \quad v_e \dsep{} v_m \given v_r, \quad v_e \dsep{} v_m \given \{v_s, v_r\},\\
&v_e \dsep{} v_m \given \{v_s, I\}, \quad v_e \dsep{} v_m \given \{v_r, I\}, \quad v_e \dsep{} v_m \given \{v_s, v_r, I\}, \\
&v_s \dsep{} v_m \given v_r, \quad v_s \dsep{} v_m \given \{v_e, v_r\}, \quad v_s \dsep{} v_m \given \{v_r, I\}, \quad v_s \dsep{} v_m \given \{v_e, v_r, I\}.
\end{align*}

It is easy to check that the equilibrium equations and endogenous variables in this model are uniquely solvable w.r.t.\ the causal ordering graph. Therefore, the $d$-separations above imply conditional independences between the variables in the model, assuming the exogenous variables to be independent. 

To test whether the predicted conditional independences hold when the system is at equilibrium, we ran the simulation $n=500$ times until it reached equilibrium and recorded the equilibrium concentrations $X_s$, $X_r$, $X_m$, and $X_e$. 
Parameters were as described in Appendix~\ref{app:perfect adaptation simulations}, except that $k_{Is}, k_{sr}, k_{rm}, k_{me}$ were all drawn randomly from a uniform distribution on $[1.0,1.1]$, and the input signal $I$ was drawn randomly from a uniform distribution on $[0.9,1.1]$. 
We tested all (conditional) independences with a maximum of one conditioning variable using Spearman's rank correlation test with a p-value threshold of $0.01$.\footnote{Because the LCD algorithm only uses conditional independence tests with a maximum of one variable in the conditioning test, we do not consider conditional independence tests with larger conditioning sets in this work. We did experiment with larger conditioning sets but we were not able to retrieve all predicted conditional dependences with our parameter settings and only $n=500$ samples.}
Table~\ref{tab:protein pathway:conditional independences} shows that the conditional independences with a maximum conditioning set of size one that are implied by the equilibrium Markov ordering graph are indeed present in the simulated data.

\begin{table}[ht]
  \caption{The conditional independences in the simulation of the protein pathway (described in Section~\ref{sec:application:causal ordering} and Appendix~\ref{app:conditional independences}) were assessed using Spearman's rank correlations. With a p-value threshold of $0.01$, $d$-separations with a separating set of size $0$ or $1$ coincide with conditional independence test results.}
	\label{tab:protein pathway:conditional independences}
	\centering
	\begin{tabular}{l l l l}
		\toprule
		Independence test & Correlation & $p$-value & $d$-separation \\
		\midrule
		$I\indep X_s$ & $-0.010$   & $0.82$ & yes \\
		$I\indep X_r$ & $-0.003$   & $0.94$ & yes \\
		$I\indep X_m$ & $+0.012$   & $0.80$ & yes\\
    $I\indep X_e$ & $+0.408$   & $<2.2 \times {10}^{-16}$ & no\\
		$X_s\indep X_r$ & $+0.976$ & $<2.2 \times {10}^{-16}$ & no\\
		$X_s\indep X_m$ & $+0.946$ & $<2.2 \times {10}^{-16}$ & no\\
		$X_s\indep X_e$ & $-0.891$ & $<2.2 \times {10}^{-16}$ & no\\
		$X_r\indep X_m$ & $+0.969$ & $<2.2 \times {10}^{-16}$ & no\\
		$X_r\indep X_e$ & $-0.868$ & $<2.2 \times {10}^{-16}$ & no\\
		$X_m\indep X_e$ & $-0.836$ & $<2.2 \times {10}^{-16}$ & no\\
		$I\indep X_s\given X_r$ & $-0.033$ & $0.46$ & yes\\
		$I\indep X_s\given X_m$ & $-0.066$ & $0.14$ & yes\\ 
		$I\indep X_r\given X_s$ & $+0.032$ & $0.48$ & yes\\ 
		$I\indep X_r\given X_m$ & $-0.059$ & $0.19$ & yes\\ 
		$I\indep X_m\given X_s$ & $+0.066$ & $0.14$ & yes\\ 
		$I\indep X_m\given X_r$ & $+0.060$ & $0.18$ & yes\\
		$X_e\indep X_r\given X_s$ & $+0.020$ & $0.66$ & yes\\
		$X_e\indep X_m\given X_s$ & $+0.045$ & $0.32$ & yes\\
		$X_e\indep X_m\given X_r$ & $+0.041$ & $0.36$ & yes\\
		$X_s\indep X_m\given X_r$ & $-0.007$ & $0.87$ & yes\\
    $I\indep X_e \given X_s$ & $+0.876$ & $7.9 \times {10}^{-160}$ & no \\
    $I\indep X_e \given X_r$ & $+0.814$ & $2.0 \times {10}^{-119}$ & no \\
    $I\indep X_e \given X_m$ & $+0.760$ & $3.5 \times {10}^{-95}$ & no \\
    $I\indep X_s \given X_e$ & $+0.850$ & $3.2 \times {10}^{-140}$ & no \\
    $I\indep X_r \given X_e$ & $+0.772$ & $8.6 \times {10}^{-100}$ & no \\
    $I\indep X_m \given X_e$ & $+0.703$ & $1.5 \times {10}^{-75}$ & no \\
    $X_e \indep X_s \given X_r$ & $-0.405$ & $3.6 \times {10}^{-21}$ & no \\
    $X_e \indep X_s \given X_m$ & $-0.562$ & $7.8 \times {10}^{-43}$ & no \\
    $X_e \indep X_s \given I$ & $-0.971$ & $2.0 \times {10}^{-310}$ & no \\
    $X_e \indep X_r \given X_m$ & $-0.425$ & $2.4 \times {10}^{-23}$ & no \\
    $X_e \indep X_r \given I$ & $-0.949$ & $1.0 \times {10}^{-250}$ & no \\
    $X_e \indep X_m \given I$ & $-0.921$ & $3.9 \times {10}^{-205}$ & no \\
    $X_s \indep X_r \given I$ & $+0.976$ & $<4.6 \times {10}^{-324}$ & no \\
    $X_s \indep X_r \given X_e$ & $+0.900$ & $1.5 \times {10}^{-181}$ & no \\
    $X_s \indep X_r \given X_m$ & $+0.745$ & $2.0 \times {10}^{-89}$ & no \\
    $X_s \indep X_m \given I$ & $+0.946$ & $1.9 \times {10}^{-245}$ & no \\
    $X_s \indep X_m \given X_e$ & $+0.807$ & $1.2 \times {10}^{-115}$ & no \\
    $X_r \indep X_m \given I$ & $+0.969$ & $4.6 \times {10}^{-305}$ & no \\
    $X_r \indep X_m \given X_e$ & $+0.894$ & $2.3 \times {10}^{-175}$ & no \\
    $X_r \indep X_m \given X_s$ & $+0.652$ & $1.1 \times {10}^{-61}$ & no\\
		\bottomrule
	\end{tabular}
\end{table}

To verify the predicted LCD triples of Section~\ref{sec:application:causal discovery},
we ran the simulation $n=500$ times with $k_{me} = C$ as the context variable.
Parameters were as described in Appendix~\ref{app:perfect adaptation simulations}, except that $I, k_{Is}, k_{F_s s}, k_{sr}, k_{rm}, k_{me}$ were all drawn randomly from a uniform distribution on $[1.0,1.1]$, and the parameter $k_{F_e e}$ from a uniform distribution on $[1.2,1.3]$.
Note that some parameter tuning is necessary if one wants to ensure that the equilibrium state satisfies the saturation conditions $(T_e-X_e(t))\gg K_{me}$ and $X_e(t)\gg K_{F_e e}$.
On the other hand, the parameters need to have sufficient random variation, otherwise LCD may fail because of (almost) deterministic relations between variables.
We ran the simulations until the system reaches equilibrium and recorded the equilibrium values of the variables. 
We then applied the LCD algorithm to search for LCD triples in this equilibrium data with context variable $C = k_{me}$.
For the conditional independence tests we used Spearman's rank correlation with a p-value threshold of $0.01$. 
We found the expected LCD triples $(C, v_m, v_r)$, $(C, v_m, v_s)$, $(C,v_m,v_e)$, $(C,v_r,v_s)$, $(C,v_r,v_e)$, $(C,v_s, v_e)$ and no others.

\section{Rewriting equations may reveal additional structure}
\label{app:rewriting equations}

Theorem~\ref{thm:identification of perfect adaptation} specifies sufficient but not necessary conditions for the presence of perfect adaptation. The equilibrium distribution of some systems is not faithful to the equilibrium Markov ordering graph associated with the equilibrium equations of the model. Here, we will discuss a dynamical model for a basic enzymatic reaction and we will demonstrate that this model is capable of perfect adaptation. However, it does not satisfy the conditions in Theorem~\ref{thm:identification of perfect adaptation}, and the presence of directed paths in the equilibrium causal ordering graph does not imply the presence of a causal effect at equilibrium. We will also show that this may be addressed by appropriately rewriting the equations.

The basic enzyme reaction models a substrate $S$ that reacts with an enzyme $E$ to form a complex $C$, which is converted into a product $P$ and the enzyme $E$. The dynamical equations for the concentrations $X_S(t),X_E(t),X_C(t),$ and $X_P(t)$ are given by:
\begin{align}
\dot{X}_S(t) &= k_0 - k_1 X_S(t) X_E(t) + k_{-1}X_C(t),\label{eq:ber s} \\
\dot{X}_C(t) &= k_1 X_S(t)X_E(t) - (k_{-1}+k_2) X_C(t),\label{eq:ber c}\\
\dot{X}_E(t) &= -k_1 X_S(t)X_E(t) + (k_{-1}+k_2) X_C(t),\label{eq:ber e}\\
\dot{X}_P(t) &= k_2 X_C(t) - k_3 X_P(t)\label{eq:ber p},
\end{align}
where $k_{-1},k_0,k_1,k_2,k_3$ are parameters taking value in $\mathbb{R}_{>0}$ \citep{Blom2019, Murray2002}.
We treat the initial conditions $X_S(0) = S_0$, $X_C(0) = C_0$, $X_E(0) = E_0$, and $X_P(0) = P_0$ as independent exogenous random variables, and consider the parameter $k_1$ as input signal.

Figure~\ref{fig:enzyme reaction:dynamic bipartite graph} shows the dynamic bipartite graph corresponding to the dynamic equations, and Figure~\ref{fig:enzyme reaction:dynamic causal ordering graph} the corresponding dynamic causal ordering graph.
Since there is a directed path from $k_1$ to $v_P$ in Figure~\ref{fig:enzyme reaction:dynamic causal ordering graph},
we would expect that a change in $k_1$ would generically lead to a response of $X_P(t)$. We verified this by simulating this model with $k_{-1}=1.0$, $k_0=1.0$, $k_1=k_1^-=10.0$, $k_2=0.8$, $k_3=2.5$ and with initial conditions $S_0=0.45$, $E_0=0.5$, $C_0=1.25$, and $P_0=0.40$ until the system reached equilibrium. We then recorded the response of $X_P(t)$ after changing the input signal $k_1$ into different values $k_1^+$. Figure~\ref{fig:enzyme reaction:adaptation} shows that $X_P(t)$ indeed responds to changes in the input signal $k_1$, but eventually returns to its original equilibrium value.

\begin{figure}[ht]
  \centering
	\begin{subfigure}[t]{0.25\textwidth}
		\centering
		\begin{tikzpicture}[scale=0.75,every node/.style={transform shape}]
		\GraphInit[vstyle=Normal]
		\SetGraphUnit{1}
		\Vertex[L=$v_S$,x=0,y=0] {vS}
		\Vertex[L=$v_E$,x=1.1,y=0] {vE}
		\Vertex[L=$v_C$,x=2.2,y=0] {vC}
		\Vertex[L=$v_P$,x=3.3,y=0] {vP}
		\Vertex[L=$g_S$,x=0,y=-1.3] {fS}
		\Vertex[L=$g_E$,x=1.1,y=-1.3] {fE}
		\Vertex[L=$g_C$,x=2.2,y=-1.3] {fC}
		\Vertex[L=$g_P$,x=3.3,y=-1.3] {fP}
		
		\begin{scope}[VertexStyle/.append style = {minimum size = 4pt, 
			inner sep = 0pt,
			color=black}]
		\Vertex[x=1.1, y=-2.6, L=$k_1$, Lpos=270, LabelOut]{k1}
		\end{scope}
		
		\draw[EdgeStyle, style={->}](k1) to (fS);
		\draw[EdgeStyle, style={->}](k1) to (fC);
		\draw[EdgeStyle, style={->}](k1) to (fE);

		\draw[EdgeStyle, style={-}](vS) to (fS);
		\draw[EdgeStyle, style={-}](vE) to (fS);
		\draw[EdgeStyle, style={-}](vC) to (fS);

		\draw[EdgeStyle, style={-}](vS) to (fC);
		\draw[EdgeStyle, style={-}](vE) to (fC);
		\draw[EdgeStyle, style={-}](vC) to (fC);

		\draw[EdgeStyle, style={-}](vE) to (fE);
		\draw[EdgeStyle, style={-}](vS) to (fE);
		\draw[EdgeStyle, style={-}](vC) to (fE);
		
		\draw[EdgeStyle, style={-}](vP) to (fP);
		\draw[EdgeStyle, style={-}](vC) to (fP);
		\end{tikzpicture}
		\caption{Dynamic bipartite graph.}
		\label{fig:enzyme reaction:dynamic bipartite graph}
	\end{subfigure}%
  \qquad
	\begin{subfigure}[t]{0.26\textwidth}
		\centering
      \begin{tikzpicture}[scale=0.75,every node/.style={transform shape}]
      \GraphInit[vstyle=Normal]
      \SetGraphUnit{1}
      \Vertex[L=$v_S$,x=1.2,y=0] {vS}
      \Vertex[L=$v_E$,x=2.4,y=0] {vE}
      \Vertex[L=$v_C$,x=3.6,y=0] {vC}
      \Vertex[L=$v_P$,x=5.1,y=0] {vP}
      \Vertex[L=$g_S$,x=1.2,y=-1.2] {gS}
      \Vertex[L=$g_E$,x=2.4,y=-1.2] {gE}
      \Vertex[L=$g_C$,x=3.6,y=-1.2] {gC}
      \Vertex[L=$g_P$,x=5.1,y=-1.2] {gP}
      \begin{scope}[VertexStyle/.append style = {minimum size = 4pt, 
        inner sep = 0pt,
        color=black}]
      \Vertex[x=2.5, y=-2.5, L=$k_1$, Lpos=270, LabelOut]{k1}
      \end{scope}
      \node[draw=black, fit=(vS) (vC) (vE) (gS) (gC) (gE), inner sep=0.1cm]{};
      \node[draw=black, fit=(vP) (gP), inner sep=0.1cm]{};
      \draw[EdgeStyle, style={->}](vC) to (4.55,0.0);
      \draw[EdgeStyle, style={->}](k1) to (2.5,-1.85);
      \end{tikzpicture}
		\caption{Dynamic causal ordering graph.}
		\label{fig:enzyme reaction:dynamic causal ordering graph}
	\end{subfigure}%
  \qquad
	\begin{subfigure}[t]{0.35\textwidth}
		\centering
		\includegraphics{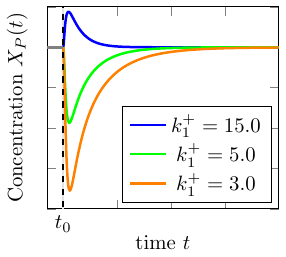}
		\caption{Response of $X_P$ to persistent changes of $k_1$.
    $k_1(t) = k_1^-$ for $t < t_0$, $k_1(t) = k_1^+$ for $t \ge t_0$.}
		\label{fig:enzyme reaction:adaptation}
	\end{subfigure}
  \caption{The dynamic bipartite graph (\subref{fig:enzyme reaction:dynamic bipartite graph}) and dynamic causal ordering graph (\subref{fig:enzyme reaction:dynamic causal ordering graph}) for the basic enzyme reaction modelled by equations \eqref{eq:ber s}, \eqref{eq:ber c}, \eqref{eq:ber e}, and \eqref{eq:ber p}. Panel~(\subref{fig:enzyme reaction:adaptation}) shows simulations that suggest that the concentration $X_P$ perfectly adapts after an initial transient response to a persistent change in the parameter $k_1$.}
	\label{fig:enzyme reaction}
\end{figure}

The equilibrium equations of the model are given by:
\begin{align}
f_S:\qquad & k_0 - k_1 X_S X_E + k_{-1}X_C = 0, \label{eq:ber:eq S}\\
f_C:\qquad & k_1 X_S X_E - (k_{-1}+k_2) X_C = 0,\label{eq:ber:eq C}\\
f_E:\qquad & -k_1 X_S X_E + (k_{-1}+k_2) X_C = 0,\label{eq:ber:eq E}\\
f_P:\qquad & k_2 X_C - k_3 X_P = 0,\label{eq:ber:eq P}\\
f_{CE}:\qquad & X_C + X_E - (C_0 + E_0) = 0,\label{eq:ber:eq CE}
\end{align}
where the last equation is derived from the constant of motion $X_C(t)+X_E(t)$ (see \citep{Blom2019} for more details).
The equilibrium bipartite graph (Figure~\ref{fig:enzyme reaction:equilibrium bipartite graph}) does not have a perfect matching (as it contains more equations than endogenous variables), 
but we can apply the extended causal ordering algorithm \citep{Blom2020} to construct the equilibrium causal ordering graph in Figure~\ref{fig:enzyme reaction:equilibrium causal ordering graph}. There is a directed path from $k_1$ to $v_P$ in the equilibrium Markov ordering graph. Therefore, even though the basic enzyme reaction does achieve perfect adaptation, we see that it does not satisfy the conditions of Theorem~\ref{thm:identification of perfect adaptation}. 
The basic enzyme reaction is an example of a system for which directed paths in the equilibrium causal ordering graph do not imply generic causal relations between variables.

\begin{figure}
  \centering
	\begin{subfigure}[b]{0.35\textwidth}
    \centering
		\begin{tikzpicture}[scale=0.75,every node/.style={transform shape}]
		\GraphInit[vstyle=Normal]
		\SetGraphUnit{1}
		\Vertex[L=$v_S$,x=0,y=0] {vS}
		\Vertex[L=$v_E$,x=1.1,y=0] {vE}
		\Vertex[L=$v_C$,x=2.2,y=0] {vC}
		\Vertex[L=$v_P$,x=3.3,y=0] {vP}
		\Vertex[L=$f_S$,x=0,y=-1.3] {fS}
		\Vertex[L=$f_E$,x=1.1,y=-1.3] {fE}
		\Vertex[L=$f_C$,x=2.2,y=-1.3] {fC}
		\Vertex[L=$f_P$,x=3.3,y=-1.3] {fP}
    \Vertex[L=$f_{CE}$,x=-1.1,y=-1.3] {fCE}
		\begin{scope}[VertexStyle/.append style = {minimum size = 4pt, 
			inner sep = 0pt,
			color=black}]
		\Vertex[x=1.1, y=-2.6, L=$k_1$, Lpos=270, LabelOut]{k1}
		\end{scope}
		
		\draw[EdgeStyle, style={->}](k1) to (fS);
		\draw[EdgeStyle, style={->}](k1) to (fC);
		\draw[EdgeStyle, style={->}](k1) to (fE);

		\draw[EdgeStyle, style={-}](vS) to (fS);
		\draw[EdgeStyle, style={-}](vS) to (fC);
		\draw[EdgeStyle, style={-}](vS) to (fE);
		\draw[EdgeStyle, style={-}](vE) to (fS);
		\draw[EdgeStyle, style={-}](vE) to (fC);
		\draw[EdgeStyle, style={-}](vE) to (fE);
		\draw[EdgeStyle, style={-}](vE) to (fCE);
		\draw[EdgeStyle, style={-}](vC) to (fS);
		\draw[EdgeStyle, style={-}](vC) to (fC);
		\draw[EdgeStyle, style={-}](vC) to (fE);
		\draw[EdgeStyle, style={-}](vC) to (fP);
		\draw[EdgeStyle, style={-}](vC) to (fCE);
		\draw[EdgeStyle, style={-}](vP) to (fP);
		\end{tikzpicture}
    \caption{Equilibrium bipartite graph.}
		\label{fig:enzyme reaction:equilibrium bipartite graph}
	\end{subfigure}%
	\qquad\qquad
	\begin{subfigure}[b]{0.4\textwidth}
    \centering
      \begin{tikzpicture}[scale=0.75,every node/.style={transform shape}]
      \GraphInit[vstyle=Normal]
      \SetGraphUnit{1}
      \Vertex[L=$v_S$,x=1.2,y=0] {vS}
      \Vertex[L=$v_E$,x=2.4,y=0] {vE}
      \Vertex[L=$v_C$,x=3.6,y=0] {vC}
      \Vertex[L=$v_P$,x=5.1,y=0] {vP}
      \Vertex[L=$f_{CE}$,x=0,y=-1.2] {fEC}
      \Vertex[L=$f_S$,x=1.2,y=-1.2] {fS}
      \Vertex[L=$f_E$,x=2.4,y=-1.2] {fE}
      \Vertex[L=$f_C$,x=3.6,y=-1.2] {fC}
      \Vertex[L=$f_P$,x=5.1,y=-1.2] {fP}
      \begin{scope}[VertexStyle/.append style = {minimum size = 4pt, 
        inner sep = 0pt,
        color=black}]
      \Vertex[x=2.5, y=-2.5, L=$k_1$, Lpos=270, LabelOut]{k1}
      \end{scope}
      \node[draw=black, fit=(vS) (vC) (vE) (fS) (fC) (fE) (fEC), inner sep=0.1cm]{};
      \node[draw=black, fit=(vP) (fP), inner sep=0.1cm]{};
      \draw[EdgeStyle, style={->}](vC) to (4.55,0.0);
      \draw[EdgeStyle, style={->}](k1) to (2.5,-1.85);
      \end{tikzpicture}
		\caption{Equilibrium causal ordering graph.}
		\label{fig:enzyme reaction:equilibrium causal ordering graph}
  \end{subfigure}\\[\baselineskip]
	\begin{subfigure}[b]{0.35\textwidth}
    \centering
		\begin{tikzpicture}[scale=0.75,every node/.style={transform shape}]
		\GraphInit[vstyle=Normal]
		\SetGraphUnit{1}
		\Vertex[L=$v_S$,x=0,y=0] {vS}
		\Vertex[L=$v_E$,x=1.1,y=0] {vE}
		\Vertex[L=$v_C$,x=2.2,y=0] {vC}
		\Vertex[L=$v_P$,x=3.3,y=0] {vP}
		\Vertex[L=$f_S$,x=0,y=-1.3] {fS}
		\Vertex[L=$f_{CE}$,x=1.1,y=-1.3] {fE}
		\Vertex[L=$f'_C$,x=2.2,y=-1.3] {fC}
		\Vertex[L=$f_P$,x=3.3,y=-1.3] {fP}
		\begin{scope}[VertexStyle/.append style = {minimum size = 4pt, 
			inner sep = 0pt,
			color=black}]
		\Vertex[x=-1.1, y=-1.3, L=$k_1$, Lpos=180, LabelOut]{k1}
		\end{scope}
		
		\draw[EdgeStyle, style={->}](k1) to (fS);

		\draw[EdgeStyle, style={-}](vS) to (fS);
		\draw[EdgeStyle, style={-}](vE) to (fE);
		\draw[EdgeStyle, style={-}](vC) to (fC);
		\draw[EdgeStyle, style={-}](vP) to (fP);
		
		\draw[EdgeStyle, style={-}](vE) to (fS);
		\draw[EdgeStyle, style={-}](vC) to (fS);
		
		\draw[EdgeStyle, style={-}](vC) to (fE);
		
		\draw[EdgeStyle, style={-}](vC) to (fP);
		\end{tikzpicture}
		\caption{Alternative equilibrium bipartite graph.}
		\label{fig:enzyme reaction:substitution:equilibrium bipartite graph}
	\end{subfigure}%
	\qquad\qquad
	\begin{subfigure}[b]{0.40\textwidth}
    \centering
		\begin{tikzpicture}[scale=0.75,every node/.style={transform shape}]
		\GraphInit[vstyle=Normal]
		\SetGraphUnit{1}
		\Vertex[L=$v_P$,x=4.5,y=0] {vP}
		\Vertex[L=$v_C$,x=3.0,y=0] {vC}
		\Vertex[L=$v_E$,x=1.5,y=0] {vE}
		\Vertex[L=$v_S$,x=0.0,y=0] {vS}
		\Vertex[L=$f_P$,x=4.5,y=-1.2] {fP}
		\Vertex[L=$f'_C$,x=3.0,y=-1.2] {fC}
		\Vertex[L=$f_{CE}$,x=1.5,y=-1.2] {fE}
		\Vertex[L=$f_S$,x=0.0,y=-1.2] {fS}

    \begin{scope}[VertexStyle/.append style = {minimum size = 4pt, 
			inner sep = 0pt,
			color=black}]
		\Vertex[x=-1.1, y=-1.3, L=$k_1$, Lpos=180, LabelOut]{k1}
		\end{scope}
		
    \draw[EdgeStyle, style={->}](k1) to (-0.6,-1.3);
		
		\node[draw=black, fit=(vS) (fS), inner sep=0.1cm]{};
		\node[draw=black, fit=(vC) (fC), inner sep=0.1cm]{};
		\node[draw=black, fit=(vE) (fE), inner sep=0.1cm]{};
		\node[draw=black, fit=(vP) (fP), inner sep=0.1cm]{};
		
		\draw[EdgeStyle, style={->}](vC) to (3.95,0.0);
		\draw[EdgeStyle, style={->}](vC) to (2.15,0.0);
		\draw[EdgeStyle, style={->}](vE) to (0.515,0.0);
		
		\draw[EdgeStyle, style={-}](vC) to (3.0,0.8);
		\draw[EdgeStyle, style={-}](3.0,0.8) to (0.0,0.8);
		\draw[EdgeStyle, style={->}](0.0,0.8) to (0.0,0.525);
		\end{tikzpicture}
		\caption{Alternative equilibrium causal ordering graph.}
		\label{fig:enzyme reaction:substitution:equilibrium causal ordering graph}
	\end{subfigure}
  \caption{Different graphical representations of the basic enzyme reaction at equilibrium. 
    Panels (\subref{fig:enzyme reaction:equilibrium bipartite graph}) and (\subref{fig:enzyme reaction:equilibrium causal ordering graph}) show the equilibrium bipartite graph and equilibrium causal ordering graph, respectively, for the equilibrium equations \eqref{eq:ber:eq S}, \eqref{eq:ber:eq C}, \eqref{eq:ber:eq E}, \eqref{eq:ber:eq P}, and \eqref{eq:ber:eq CE}. 
    Panels (\subref{fig:enzyme reaction:substitution:equilibrium bipartite graph}) and (\subref{fig:enzyme reaction:substitution:equilibrium causal ordering graph}) show the equilibrium bipartite graph and the equilibrium causal ordering graph, respectively, for the rewritten equilibrium equations \eqref{eq:ber:eq S}, \eqref{eq:ber:eq P}, \eqref{eq:ber:eq CE} and \eqref{eq:ber:eq C'}. 
    Rewriting the equilibrium equations yields a sparser equilibrium bipartite graph, and therefore more structure is revealed in the equilibrium causal ordering graph.}
	\label{fig:enzyme reaction:causal ordering}
\end{figure}

By rewriting the equilibrium equations we can achieve stronger conclusions for this particular case. For instance, we can consider the equation $f'_C$, obtained from summing equations $f_S$ and $f_C$:
\begin{align}
f'_C:\qquad & k_0 - k_2 X_C = 0,\label{eq:ber:eq C'}
\end{align}
in combination with $f_S$, $f_P$, and $f_{CE}$. 
The equilibrium equations $f_C$ and $f_E$ can be dropped because they are linear combinations of the other equations. 
The graphs corresponding with these rewritten equations are shown in Figure~\ref{fig:enzyme reaction:causal ordering}\subref{fig:enzyme reaction:substitution:equilibrium bipartite graph}--\subref{fig:enzyme reaction:substitution:equilibrium causal ordering graph}.
The equilibrium bipartite graph for the rewritten equations in Figure~\ref{fig:enzyme reaction:substitution:equilibrium bipartite graph} turns out to be sparser than the one for the original equilibrium equations in Figure~\ref{fig:enzyme reaction:equilibrium bipartite graph}.
The equilibrium causal ordering graph in Figure~\ref{fig:enzyme reaction:substitution:equilibrium causal ordering graph} for the rewritten equilibrium equations consequently reveals more structure than the one in Figure~\ref{fig:enzyme reaction:equilibrium causal ordering graph} for the original equilibrium equations.

The two equilibrium causal ordering graphs do not model the same set of perfect interventions. For example, the (non)effects of an intervention targeting the cluster $\{v_S, f_S\}$ in the equilibrium causal ordering graph in Figure~\ref{fig:enzyme reaction:substitution:equilibrium causal ordering graph} (where $f_S$ is replaced by an equation $v_S = \xi_S$ setting $v_S$ equal to a constant $\xi_S\in\mathbb{R}_{>0}$) cannot be read off from the equilibrium causal ordering graph in Figure~\ref{fig:enzyme reaction:equilibrium causal ordering graph}.
Furthermore, there is no directed path from $k_1$ to $v_P$ in Figure~\ref{fig:enzyme reaction:substitution:equilibrium causal ordering graph}, and hence we can now make use of Theorem~\ref{thm:identification of perfect adaptation} to conclude that the concentration $X_P(t)$ perfectly adapts to persistent changes of the parameter $k_1$.

\section{Transforming variables may reveal structure}
\label{app:ifflp network}

The IFFLP network in Ma \emph{et al.} \cite{Ma2009} that we briefly discussed in Section~\ref{sec:perfect adaptation:examples:negative feedback loop} could be a graphical representation of the following differential equations:
\begin{align}
\label{eq:ifflp:A}
\dot{X}_A(t) &= I(t) k_{IA} \frac{(1-X_A(t))}{K_{IA} + (1-X_A(t))} - F_A k_{F_A A} \frac{X_A(t)}{K_{F_A A} + X_A(t)}, \\
\label{eq:ifflp:B}
\dot{X}_B(t) &= X_A(t) k_{AB} \frac{(1-X_B(t))}{K_{AB} + (1-X_B(t))} - F_B k_{F_B B} \frac{X_B(t)}{K_{F_B B} + X_B(t)}, \\
\label{eq:ifflp:C}
\dot{X}_C(t) &= X_A(t) k_{AC} \frac{(1-X_C(t))}{K_{AC} + (1-X_C(t))} - X_B(t) k_{BC} \frac{X_C(t)}{K_{BC} + X_C(t)},
\end{align}
where $I(t)$ represents an external input into the system. This network is capable of perfect adaptation if the first term of $\dot{X}_B(t)$ is in the saturated region $(1-X_B(t))\gg K_{AB}$ and the second term is in the linear region $X_B(t)\ll K_{F_B B}$, which allows us to make the following approximation:
\begin{align}
\label{eq:ifflp:B'}
\frac{dX_B(t)}{dt} &\approx X_A(t) k_{AB} - \frac{F_B k_{F_B B}}{K_{F_B B}} X_B(t).
\end{align}
Therefore, an equilibrium solution $X_B$ for $B$ is proportional to the equilibrium solution $X_A$ for $A$. Since both terms in equation \eqref{eq:ifflp:C} are proportional to $X_A$ we find that the equilibrium solution $X_C$ for $C$ is a function of only the parameters $k_{AC}$, $K_{AC}$, $k_{BC}$, and $K_{BC}$ (note that $X_A$ factors out of the equilibrium equation corresponding to \eqref{eq:ifflp:B'}), and hence it does not depend on the input parameter $I(t)$. Since a change in the input signal $I(t)$ changes $\dot{X}_A(t)$ there is a transient effect on $X_A(t)$. Similarly there must also be a transient effect on both $X_B(t)$ and $X_C(t)$. It follows that the system achieves perfect adaptation.

The equilibrium equations associated with equations \eqref{eq:ifflp:A}, the approximation \eqref{eq:ifflp:B'} to \eqref{eq:ifflp:B}, and \eqref{eq:ifflp:C} are given by:
\begin{align}
f_A:\qquad& I k_{IA} \frac{(1-X_A)}{K_{IA} + (1-X_A)} - F_A k_{F_A A} \frac{X_A}{K_{F_A A} + X_A} = 0, \label{eq:A}\\
f_B:\qquad& X_A k_{AB} - \frac{F_B k_{F_B B}}{K_{F_B B}} X_B = 0, \label{eq:B} \\
f_C:\qquad& X_A k_{AC} \frac{(1-X_C)}{K_{AC} + (1-X_C)} - X_B k_{BC} \frac{X_C}{K_{BC} + X_C} = 0. \label{eq:C}
\end{align}
The associated equilibrium causal ordering graph in Figure \ref{fig:IFFLP:equilibrium causal ordering graphs:original}
shows that there is a directed path from the input signal $I$ to the cluster $\{v_A, v_B, v_C\}$. Therefore, the conditions of Theorem \ref{thm:identification of perfect adaptation} are not satisfied for the system.

Interestingly, though, if we first make a change of variables from $(X_A,X_B,X_C) \mapsto (X_A,X_R,X_C)$ with $X_R := X_B / X_A$
(assuming that $X_A \in \R_{>0}$), we can make use of the causal ordering approach. 
Indeed, we can rewrite the equilibrium equations as follows in terms of the new variables $X_A,X_R,X_C$:
\begin{align}
f_A:\qquad& I k_{IA} \frac{(1-X_A)}{K_{IA} + (1-X_A)} - F_A k_{F_A A} \frac{X_A}{K_{F_A A} + X_A} = 0, \label{eq:A2}\\
f_R:\qquad& k_{AB} - \frac{F_B k_{F_B B}}{K_{F_B B}} X_R = 0, \label{eq:B2} \\
f_{C'}:\qquad& k_{AC} \frac{(1-X_C)}{K_{AC} + (1-X_C)} - X_R k_{BC} \frac{X_C}{K_{BC} + X_C} = 0. \label{eq:C2}
\end{align}
This yields a sparser equilibrium bipartite graph. As shown in Figure~\ref{fig:IFFLP:equilibrium causal ordering graphs:reparameterized}, we can now read off from the corresponding equilibrium causal ordering graph that the input signal $I$ does not affect the equilibrium values of $X_R$ and $X_C$.
Hence, in this parameterization of the model, Theorem \ref{thm:identification of perfect adaptation} can be used to identify the perfect adaptive behavior of the system.

\begin{figure}[ht] \centering
	\begin{subfigure}[b]{0.25\textwidth}
    \begin{tikzpicture}[scale=0.75,every node/.style={transform shape}]
    \GraphInit[vstyle=Normal]
    \SetGraphUnit{1}
    \Vertex[L=$v_A$,x=0,y=0] {vT}
    \Vertex[L=$v_B$,x=1.5, y=0] {vI}
    \Vertex[L=$v_C$,x=3.0, y=0] {vE}
    \Vertex[L=$f_A$,x=0.0, y=-1.5] {fT}
    \Vertex[L=$f_B$,x=1.5, y=-1.5] {fI}
    \Vertex[L=$f_C$,x=3.0, y=-1.5] {fE}
    
    \node[draw=black, fit=(vT) (vI) (vE) (fT) (fI) (fE), inner sep=0.1cm ]{};
    
    \begin{scope}[VertexStyle/.append style = {minimum size = 4pt, 
      inner sep = 0pt,
      color=black}]
    \Vertex[x=-1.2, y=0.0, L=$I$, Lpos=270, LabelOut]{I}
    \end{scope}
    
    \draw[EdgeStyle, style={->}](I) to (-0.6,0.0);
    \end{tikzpicture}
	  \caption{Original parameterization.}
		\label{fig:IFFLP:equilibrium causal ordering graphs:original}
	\end{subfigure}%
  \qquad\qquad
	\begin{subfigure}[b]{0.25\textwidth}
    \begin{tikzpicture}[scale=0.75,every node/.style={transform shape}]
    \GraphInit[vstyle=Normal]
    \SetGraphUnit{1}
    \Vertex[L=$v_A$,x=0,y=0] {vT}
    \Vertex[L=$v_R$,x=1.5, y=0] {vI}
    \Vertex[L=$v_C$,x=3.0, y=0] {vE}
    \Vertex[L=$f_A$,x=0.0, y=-1.5] {fT}
    \Vertex[L=$f_R$,x=1.5, y=-1.5] {fI}
    \Vertex[L=$f_{C'}$,x=3.0, y=-1.5] {fE}
    
    \node[draw=black, fit=(vT) (fT), inner sep=0.1cm] (cT) {};
    \node[draw=black, fit=(vI) (fI), inner sep=0.1cm] (cI) {};
    \node[draw=black, fit=(vE) (fE), inner sep=0.1cm] (cE) {};
    
    \begin{scope}[VertexStyle/.append style = {minimum size = 4pt, 
      inner sep = 0pt,
      color=black}]
    \Vertex[x=-1.2, y=0.0, L=$I$, Lpos=270, LabelOut]{I}
    \end{scope}
    
    \draw[EdgeStyle, style={->}](I) to (-0.6,0.0);
    \draw[EdgeStyle, style={->}](vI) to (2.4,0.0);
    \end{tikzpicture}
	\caption{Alternative parameterization.}
		\label{fig:IFFLP:equilibrium causal ordering graphs:reparameterized}
  \end{subfigure}
    \caption{Equilibrium causal ordering graphs for the IFFLP network. (\subref{fig:IFFLP:equilibrium causal ordering graphs:original}) modelled by equations \eqref{eq:A}, \eqref{eq:B}, and \eqref{eq:C} and variables $X_A, X_B, X_C$; (\subref{fig:IFFLP:equilibrium causal ordering graphs:reparameterized}) modelled by equations \eqref{eq:A2}, \eqref{eq:B2}, and \eqref{eq:C2} and variables $X_A, X_R, X_C$.}
	\label{fig:IFFLP:equilibrium causal ordering graphs}
\end{figure}

\section{Markov ordering graphs have no inherent causal interpretation}
\label{app:causal interpretation mog}

The causal interpretation of the equilibrium Markov ordering graph for the bathtub model is discussed at length in \citep{Blom2020}. The  conclusion is that the Markov ordering graph \emph{alone} does not contain enough information to read off the effects of interventions in an unambiguous way.\footnote{To arrive at this conclusion, we first need to explicitly state what we mean when we talk about `causal relations'. We follow the common interpretation in contemporary literature in terms of interventions. In the context of a specific model, this means that an intervention on the cause brings about a change in the effect.} As a result, the Markov ordering graph does not have a straightforward causal interpretation in terms of interventions, contrary to what is sometimes claimed \citep{Dash2005, Iwasaki1994}. For the sake of completeness we will summarize here the discussion of the causal interpretation of the equilibrium Markov ordering graph of the bathtub model. 

\begin{example}
For the bathtub model (Section~\ref{sec:perfect adapatation:examples:bathtub}),
consider an intervention targeting the dynamical equation $g_D$ that also changes the associated equilibrium equation $f_D$. For example, consider putting the bathtub outside in the rain. This will not change the inflow through the faucet, but will add to the total amount of in-flowing water into the tub, and can be modeled by modifying equation \eqref{eq:bathtub:dyn D} into:
$$\dot{X}_D(t) = U_1(X_I(t) - X_O(t) + U_{R}(t)),$$
where $U_R(t)$ is a new exogenous variable that quantifies the amount of in-flowing water due to the rain (in the original model, $U_R(t) = 0$).
Through explicit calculations, it can be shown that this soft intervention has an effect on $X_O$, $X_P$, and $X_D$ at equilibrium.\footnote{Assuming that the system converges to equilibrium after the intervention, the equilibrium causal ordering graph in Figure~\ref{fig:equilibrium causal ordering graph:bathtub} tells us that soft interventions on $f_D$ \emph{generically} change the equilibrium distributions of $X_O$, $X_P$, and $X_D$. We need explicit calculations to verify that the generic effects correspond to a change in distribution.} If we were to read off the effects of a soft intervention targeting $f_D$ by verifying whether there is a directed path from $v_D$ in the equilibrium Markov ordering graph, we would erroneously conclude that this intervention \emph{only} has an effect on $X_D$. In a similar fashion it can be shown that the graph does not represent the effects of perfect interventions targeting $\{f_D, v_D\}$ or $\{f_D, v_O\}$ either. From the equilibrium causal ordering graph one can read off that the absence and presence of (generic) effects of the perfect intervention $\{f_P, v_D\}$ correspond with the absence and presence of directed paths in the equilibrium Markov ordering graph starting from $v_D$. Because the equilibrium Markov ordering graph alone does not specify to which experimental procedure a perfect intervention on one of its vertices corresponds to, we conclude that the equilibrium Markov ordering graph on its own does not possess a causal interpretation.
\end{example}


A correct interpretation of a directed edge $(v_i\to v_j)$ in the equilibrium Markov ordering graph would be that an intervention targeting equations in the cluster of $v_i$ \emph{in the causal ordering graph} generically will have an effect on the equilibrium distribution of $v_j$. However, the Markov ordering graph itself does not specify the variables and equations that share a cluster with the variables in the causal ordering graph. In many dynamical systems, though, the equilibrium equations $f_i$ derived from differential equations for variables $X_i(t)$ end up in the same cluster as the associated variable $v_i$. Under such an assumption, one could give an unambiguous causal interpretation to the Markov ordering graph in terms of perfect interventions.\footnote{For first-order dynamical systems in which each variable is self-regulating, each variable $v_i$ shares a cluster with the equilibrium equation $f_i$ associated with the time derivative $\dot{X}_i(t)$ \citep{BlomMooij_UAI_22}. Note that the examples of perfectly adaptive systems that we have considered in this work do contain variables that are not self-regulating.}

However, there is another obstacle to interpreting a directed edge $(v_i\to v_j)$ in the Markov ordering graph as a \emph{direct} causal effect. This can go wrong in case causal cycles are present. Indeed, the Markov ordering graph for a simple SCM is the acyclification of its graph \citep{Bongers2020}. Therefore, if $v_j$ is part of a feedback loop together with $v_k$, and the structural equation of $v_k$ depends on $v_i$, the Markov ordering graph will contain the directed edge $v_i \to v_j$ even if there is no direct effect of $v_i$ on $v_j$ (that is, if the structural equation for $v_j$ does not depend on $v_i$).

Therefore, if one cannot rule out the possibility of feedback, one should avoid reading the Markov ordering graph as if the directed edges represent direct causal effects (and as if directed paths represent causal effects).

\nocite{Dash2008}

\end{document}